\documentclass[a4paper,fleqn]{cas-sc}

\usepackage[numbers]{natbib}
\usepackage{placeins}   

\newtheorem{theorem}{Theorem}
\newtheorem{lemma}[theorem]{Lemma}
\newtheorem{proposition}[theorem]{Proposition}
\newtheorem{corollary}[theorem]{Corollary}

\newdefinition{definition}{Definition}
\newdefinition{rmk}{Remark}
\newdefinition{summary}{Summary}
\newdefinition{assumption}{Assumption}
\newproof{proof}{proof}

\def\tsc#1{\csdef{#1}{\textsc{\lowercase{#1}}\xspace}}
\tsc{WGM}
\tsc{QE}
\tsc{EP}
\tsc{PMS}
\tsc{BEC}
\tsc{DE}

\begin{document}
\let\WriteBookmarks\relax
\def\floatpagepagefraction{1}
\def\textpagefraction{.001}

\setlength{\textfloatsep}{10pt plus 2pt minus 2pt}  
\setlength{\floatsep}{8pt plus 2pt minus 2pt}       
\setlength{\intextsep}{8pt plus 2pt minus 2pt}      
\setlength{\abovecaptionskip}{5pt}                  
\setlength{\belowcaptionskip}{0pt}                  
\makeatletter
\setlength{\@fptop}{0pt}                  
\setlength{\@fpsep}{8pt plus 2pt}         
\setlength{\@fpbot}{0pt plus 1fil}        
\makeatother
\shorttitle{Generalized Convexity and Smoothness}
\shortauthors{Binchuan Qi}

\title [mode = title]{Generalized Convexity and Smoothness via Conjugate Duality, with Applications to Deep Network Training}                      

\author[1,2]{Binchuan Qi}[orcid=0000-0001-5832-1884]
                        
\cormark[1]
\fnmark[1]
\ead{qibinchuan@163.com}

\credit{Methodology, Software, Data curation, Writing - Original draft preparation}

\affiliation[1]{organization={Zhejiang Yuying College of Vocational Technology},
                addressline={Sihao Street}, 
                city={Hangzhou},
                postcode={310018}, 
                state={Zhejiang},
                country={China}}

\nonumnote{In this work, we propose an optimization theory framework based on conjugate duality. This framework extends convex optimization theory and reveals the gradient-based optimization mechanism for a class of non-convex, non-Lipschitz-smooth optimization problems exemplified by deep neural network training.}

\begin{abstract}
Stochastic gradient descent (SGD) and its variants train deep neural networks (DNNs) effectively in practice, yet classical optimization theory does not fully explain this success: conventional analyses rest on differentiability, convexity, and smoothness, assumptions that DNN objectives routinely violate. We close this gap by generalizing convexity and smoothness, through Legendre functions and convex conjugation, into a unified optimization framework for DNN training. We introduce \(\mathcal{H}(\psi)\)-convexity and \(\mathcal{H}(\Psi)\)-smoothness, a single formalism that encompasses convex and non-convex as well as smooth and non-smooth objectives and reveals a natural duality between generalized convexity and smoothness. Building on these notions, we derive generalized gradient descent (GD) and generalized SGD via convex conjugation. We prove that generalized GD attains an optimal learning rate of exactly \(1\), and we establish rigorous gradient-energy-based convergence rates for both optimizers. We further recast DNN training as a composite optimization problem and show that its convergence hinges on jointly reducing the gradient energy and controlling the extreme eigenvalues of the structure matrix; a model capacity risk quantifies the residual error that no optimizer can remove. Experiments on six datasets (MNIST, Fashion-MNIST, CIFAR-10, CIFAR-100, TREC, and SST-2) and on convolutional, residual, and Transformer architectures, with five loss functions, are consistent with the predicted training dynamics.
\end{abstract}

\begin{keywords}%
non-convex optimization \sep non-Lipschitz smoothness \sep convex conjugate duality \sep gradient energy \sep trainability \sep deep neural networks
\end{keywords}

\maketitle

\section{Introduction}
\label{sec:introduction}

Training deep neural networks (DNNs) yields highly non-convex, non-smooth empirical-risk objectives. Nevertheless, stochastic gradient descent (SGD) and its variants consistently converge to high-quality solutions with low empirical risk, even though the convexity and regularity assumptions required for classical guarantees are absent~\citep{Zhang2021UnderstandingDL}. This success contrasts sharply with classical non-convex optimization theory, which typically guarantees only convergence to stationary points~\citep{Jentzen2018StrongEA}; moreover, even simple non-convex problems are NP-hard in general~\citep{nesterov2008advance}. This theory--practice gap motivates optimization-theoretic frameworks that can characterize DNN training.

Current research mainly follows two directions:
\begin{enumerate}
\item The first exploits special structural properties of DNNs. The Neural Tangent Kernel (NTK) framework~\citep{Jacot2018NeuralTK} analyzes lazy training with minimal parameter change, but cannot capture feature-learning dynamics in finite-width networks. Fenchel--Young-loss approaches~\citep{QiGL25} relate gradient energies to distribution-fitting errors, but apply only to the Fenchel--Young family of losses. As detailed in Appendix~\ref{sec:model:related}, both have limitations in scope and applicability.

\item The second extends classical strong convexity and Lipschitz smoothness. For convexity, higher-order strong convexity has been explored in various domains~\citep{Lin2003SomeEP,Mohsen2019StronglyCF,Alabdali2019CharacterizationsOU,noor2020higher,Noor2021PropertiesOH}. For smoothness, \((L_0,L_1)\)-smoothness~\citep{Zhang2019WhyGC} has been analyzed in convex and non-convex settings~\citep{Zhang2020ImprovedAO,Qian2021UnderstandingGC,Zhao2021OnTC,crawshaw2022robustness} and combined with variance reduction~\citep{Reisizadeh2023VariancereducedCF}. Further variants include \(\alpha\)-symmetric generalized smoothness~\citep{Chen2023GeneralizedSmoothNO} and \(\ell\)-smoothness~\citep{Li2023ConvexAN}. Most recently, \citet{qi2025extendedconvexitysmoothnessapplications} extended strong convexity and Lipschitz smoothness using \(L_2\)-norm power functions, proving DNN trainability under convexity and differentiability of the loss. However, restricting to \(L_2\)-norm power functions imposes restrictive assumptions and complicates the analysis of gradient descent (GD) and SGD.
\end{enumerate}

Despite this progress, both directions still rely on assumptions---such as infinite width, convexity, or differentiability---that practical DNN objectives routinely violate, and they often incur substantial analytical complexity.

Our starting point is the well-known duality between strong convexity and Lipschitz smoothness (Lemma~\ref{lem:convex_smooth_dual})~\citep{kakade2009duality,zhou2018fenchel,qi2025extendedconvexitysmoothnessapplications}. Both classical properties ultimately control function variation through the quadratic term \(\|\cdot\|_2^2\).  This suggests a natural generalization: if the quadratic energy is replaced by a broader family of functions built from Legendre functions and convex conjugates, the same duality may extend far beyond the convex, smooth setting. 
Following this line of reasoning, we develop generalized notions of convexity and smoothness, together with the corresponding generalized GD and SGD, and use them to formulate DNN training as a composite optimization problem.  The formal definitions, the precise duality, and the resulting convergence guarantees are presented below; the main consequences for optimization are summarized in the Contributions. 

\paragraph{Contributions}
The main contributions of this work are summarized as follows:

\begin{itemize}
 \item We generalize classical strong convexity and Lipschitz smoothness to \(\mathcal{H}(\psi)\)-convexity and \(\mathcal{H}(\Psi)\)-smoothness using Legendre functions, and investigate their mathematical properties. This framework unifies convex and non-convex, smooth and non-smooth functions and reveals a duality between smoothness and convexity.
 \item We define generalized GD and SGD via convex conjugation and prove their convergence with respect to the gradient energy under the generalized conditions. Notably, the optimal learning rate of generalized GD is exactly \(1\), and both generalized GD and SGD reduce to their classical counterparts under quadratic energy functions.

 \item We construct a composite optimization framework that encompasses DNN training. We prove that such problems can be solved by jointly reducing the gradient energy and controlling the extreme eigenvalues of the model's structure matrix, which together bound the optimization objective from above and below.

 \item We apply the framework to DNN optimization, demonstrating theoretically and empirically that SGD effectively reduces the gradient energy, while architectural techniques such as skip connections, random initialization, and over-parameterization control the extreme eigenvalues of the structure matrix. Experiments on six datasets and several architectures support the framework.
\end{itemize}

\paragraph{Organization}

The rest of this paper is organized as follows. Section~\ref{sec:pre} introduces preliminaries. Section~\ref{sec:def} defines $\mathcal{H}(\psi)$-convexity and $\mathcal{H}(\Psi)$-smoothness and analyzes their basic properties. Section~\ref{sec:methods} presents generalized GD and generalized SGD and analyzes their convergence rates under generalized convexity and smoothness. Section~\ref{sec:com_opt} develops a theoretical approach to a special class of non-convex, non-smooth problems represented by DNN training. Section~\ref{sec:experiments} reports experimental settings and results, and Section~\ref{sec:conclusion} concludes and discusses future directions. Appendix~\ref{sec:lemmas} collects auxiliary lemmas from convex analysis and information theory, Appendix~\ref{sec:model:related} reviews related literature, and Appendix~\ref{appendix:proof} provides complete proofs of all theoretical results. 

\section{Preliminaries}
\label{sec:pre}

This section establishes the foundational concepts, notation, and basic assumptions used throughout the paper. 

\subsection{Notation}
\label{subsec:notation}

\begin{itemize}
\item For a differentiable function \(f_\theta(x)\) with parameters \(\theta\in\mathbb{R}^m\) and \(d\)-dimensional outputs, we use the denominator layout convention. The Jacobian of \(f_\theta(x)\) w.r.t.\ \(\theta\) is
\begin{equation}
\mathbf{J}_\theta f_\theta(x) := \frac{\partial f_\theta(x)}{\partial \theta} \in \mathbb{R}^{m\times d}, \quad
\left(\frac{\partial f_\theta(x)}{\partial \theta}\right)_{ij} = \frac{\partial f_j}{\partial \theta_i},
\end{equation}
where \(i=1,\dots,m\) indexes parameters and \(j=1,\dots,d\) indexes outputs. For scalar outputs (\(d=1\)), the Jacobian reduces to a row vector; its transpose, the column gradient, is denoted \(\nabla_\theta f_\theta(x)\in\mathbb{R}^m\).

\item The Legendre--Fenchel (convex) conjugate of \(\Psi\) is
\begin{equation}
\Psi^*(\nu) := \sup_{\mu\in\mathrm{dom}(\Psi)} \bigl\{ \langle \mu,\nu \rangle - \Psi(\mu) \bigr\},
\end{equation}
with \(\langle\cdot,\cdot\rangle\) the standard inner product \citep{Todd2003ConvexAA}. If \(\Psi\) is strictly convex and differentiable, \(\nabla\Psi(\mu)\) denotes its gradient, and we define the \textit{conjugate dual} \(\mu_\Psi^* := \nabla\Psi(\mu)\). In particular, \(\mu_\Psi^*=\mu\) for \(\Psi(\cdot)=\frac12\|\cdot\|_2^2\).

\item The Fenchel--Young loss induced by a convex function \(G\) is
\begin{equation}\label{def:fy_loss}
d_{G}(\mu,\nu) := G(\mu) + G^*(\nu) - \langle \mu,\nu \rangle,
\end{equation}
where \(\mu\in\mathrm{dom}(G)\) and \(\nu\in\mathrm{dom}(G^*)\) \citep{Blondel2019LearningWF}. 
\item For any matrix \(A\), \(\lambda_{\max}(A)\) and \(\lambda_{\min}(A)\) denote its maximum and minimum eigenvalues, and \(\sigma_{\max}(A)\) and \(\sigma_{\min}(A)\) its maximum and minimum singular values. 

\item The random pair \(Z=(X,Y)\) follows distribution \(q_{\mathcal{X}\mathcal{Y}}\) (shortened to \(q\)) over \(\mathcal{Z}=\mathcal{X}\times\mathcal{Y}\), where \(\mathcal{X}\) is the input space and \(\mathcal{Y}\) the finite label set; \(|\mathcal{Z}|\) denotes cardinality. We assume that \(\mathcal{X}\) is finite, as is the case for data represented in finite precision; together with the finiteness of \(\mathcal{Y}\), this implies \(|\mathcal{Z}|<\infty\).

\item The Frobenius norm of \(A\in\mathbb{R}^{m\times n}\) is
\(\|A\|_F = \sqrt{\sum_{i=1}^m\sum_{j=1}^n A_{ij}^2} = \sqrt{\operatorname{tr}(A^\top A)}\),
where \(\operatorname{tr}(\cdot)\) denotes the matrix trace.
\item Without loss of generality, we only consider minimization of real-valued functions, as maximizing a function is equivalent to minimizing its negative. All functions studied herein are continuous and proper, with convex domains. At non-differentiable points, the gradient is set to zero by convention (see Section~\ref{sec:def} for the precise convention).
\end{itemize}

\section{Generalized Convexity and Smoothness}
\label{sec:def}

In this section, we introduce the concept of an energy function and use it to generalize the classical notions of convexity and Lipschitz smoothness. This generalization gives rise to the concepts of \(\mathcal{H}(\psi)\)-convexity and \(\mathcal{H}(\Psi)\)-smoothness. We further analyze their properties and implications in optimization theory.
 
\subsection{Energy Functions}
We begin with the formal definition of an energy function.
\begin{definition}[Energy Function]
\label{def:energy_fun}
A function \(\Psi\colon \mathbb{R}^m \to \mathbb{R}\) is called an energy function if it is a Legendre function with a unique global minimum attained at the origin, and it admits the radial representation:
\begin{equation}
 \Psi(\mu) = \Psi_\circ(\|\mu\|),
\end{equation}
where \(\|\cdot\|\) denotes an arbitrary norm on \(\mathbb{R}^m\).
If the norm is restricted to the \(L_2\)-norm, we refer to it as an \(L_2\)-norm energy function. 
\end{definition}
We first present the properties of energy functions. 
\begin{proposition}
 \label{prop:energy_fun}
 \begin{enumerate}
 \item \label{prop:energy_fun1} If \(\Psi(\mu)\) is an energy function, then \(\Psi^*(\nu)\) is also an energy function, and 
 \begin{equation}
 \begin{aligned}
 & \Psi^*(\nu) = \Psi_\circ^*(\|\nu\|_*), \quad \Psi(0)+\Psi^*(0)=0, \quad \mu^\top \mu_\Psi^* \geq 0,
 \end{aligned}
 \end{equation}
where \(\Psi_\circ^*\) is the one-dimensional convex conjugate of \(\Psi_\circ\). 
 \item \label{prop:energy_fun2}If \(\Psi(\mu) = \|a_{\Psi}\mu\|_2^{r_\Psi}/r_{\Psi}+c_\Psi\) is an energy function with parameters \(a_\Psi\in\mathbb{R}_{>0}\), \(r_\Psi>1\), and \(c_\Psi\in\mathbb{R}\), then its convex conjugate \(\Psi^*(\nu)=\|a_{\Psi^*}\nu\|_2^{r_{\Psi^*}}/r_{\Psi^*}+c_{\Psi^*}\) is also an energy function, where
\begin{equation}
 \begin{aligned}
 1/r_{\Psi} + 1/r_{\Psi^*} &= 1, a_{\Psi}a_{\Psi^*} = 1, c_{\Psi^*}+c_{\Psi} = 0.
 \end{aligned}
\end{equation}
\item \label{prop:energy_fun3} Let \(\Psi(\mu)\) be an energy function. For any \(\mu \in \mathbb{R}^m\), \(F(a) = \Psi(a\mu)\) is monotonically non-decreasing on \([0,\infty)\); 
if \(\mu \neq 0\), then \(F\) is strictly increasing and convex.
 \end{enumerate}
\end{proposition}
The proof of this proposition is provided in Appendix~\ref{appendix:proof_min_energy}. 
 
\subsection{Generalized Convexity and Smoothness}

To address non-differentiable points in DNNs, we adopt a generalized gradient convention. For a function \(G: \mathbb{R}^m \to \bar{\mathbb{R}}\) with finitely many non-differentiable points, we define \(\nabla G(\mu)\) for all \(\mu\in\mathbb{R}^m\) as:
\begin{enumerate}
\item The classical gradient if \(G\) is differentiable at \(\mu\);
\item \(\nabla G(\mu) := \mathbf{0}\) otherwise.
\end{enumerate}
This convention renders first-order gradient operations well defined and facilitates the theoretical analysis of non-smooth DNNs (e.g., ReLU-based networks). It also aligns with practical deep learning implementations and ensures a zero gradient at global optima.
 
Using energy functions and the extended gradient, we introduce generalized notions of convexity and smoothness. Throughout, uppercase energy functions (\(\Psi\), \(\Phi\)) characterize smoothness and lowercase ones (\(\psi\), \(\phi\)) characterize convexity. 
\begin{definition}[Generalized Convexity and Smoothness]
\label{def:g_convex_smooth}
Define the first-order Taylor remainder \(S_G(\mu,\nu) = G(\mu) - G(\nu) - \langle \nabla G(\nu), \mu - \nu \rangle\). Note that \(S_G(\mu,\nu)\) is the residual of the first-order expansion of \(G\) around \(\nu\) evaluated at \(\mu\). 
We define the following notions of generalized convexity and smoothness. 
\begin{itemize}
 \item \textbf{\(\mathcal{H}(\Psi)\)-smooth function}. A real-valued function \(G(\cdot)\) is said to be \(\mathcal{H}(\Psi)\)-smooth if
\begin{equation}
 S_G(\mu, \nu) \le \Psi(\mu - \nu), \quad \forall \mu, \nu \in \mathrm{dom}(G),
 \end{equation}
 where \(\Psi\) is an energy function.

 \item \textbf{\(\mathcal{H}(\psi)\)-convex function}. A function \(G(\cdot)\) is said to be \(\mathcal{H}(\psi)\)-convex if 
\begin{equation}
 S_G(\mu, \nu) \ge \psi(\mu - \nu), \quad \forall \mu, \nu \in \mathrm{dom}(G),
\end{equation}
 where \(\psi\) is an energy function. 

\end{itemize}
Since \(S_G(\mu,\mu)=0\), these inequalities imply \(\Psi(0)\ge 0\) and \(\psi(0)\le 0\).
\end{definition}

Generalized smoothness and convexity further satisfy the following properties:
\begin{proposition}
\label{prop:gen_prop}
\begin{enumerate}
 \item \label{prop:gen_prop_duality} Let \(G\) be closed and convex, and let \(G\) and its convex conjugate \(G^*\) be differentiable on their domains. Then the following statements are equivalent: (i)~\(G\) is \(\mathcal{H}(\Psi)\)-smooth; (ii)~\(G^*\) is \(\mathcal{H}(\Psi^*)\)-convex.
 \item \label{prop:gen_prop_scale}For constants \(a>0\), \(b>0\), and \(c\in\mathbb{R}\), define \(F(x) := a G(b x) + c\) on the domain \(\{x: bx\in\mathrm{dom} G\}\). If \(G\) is \(\mathcal{H}(\Psi)\)-smooth, then \(F\) is \(\mathcal{H}(a\Psi(b\,\cdot))\)-smooth. If \(G\) is \(\mathcal{H}(\psi)\)-convex, then \(F\) is \(\mathcal{H}(a\psi(b\,\cdot))\)-convex.
 \item \label{prop:gen_prop_dataset} Let the random variable \(Z\) take values in \(\mathcal{Z}\). If \(G(\theta, z)\) has the same generalized convexity and smoothness properties with respect to \(\theta\) for all \(z \in \mathcal{Z}\), then \(G(\theta, s)\) also shares the same properties with respect to \(\theta\).
 \item \label{prop:gen_prop_transform_keep}
 Let \(s\) be a fixed constant vector, and let \(G\) be closed, strictly convex and differentiable. If \(G(\cdot)\) is \(\mathcal{H}(\psi)\)-convex, then \(d_G(\cdot, s)\) has the same generalized convexity and smoothness properties as \(G(\cdot)\), and \(d_G(s, \cdot)\) has the same generalized convexity and smoothness properties as \(G^*(\cdot)\).
\end{enumerate}
\end{proposition}
The proof of this proposition is provided in Appendix~\ref{appendix:proof_gen_prop}.

Item~\ref{prop:gen_prop_duality} reflects the conjugacy between generalized convexity and generalized smoothness. Item~\ref{prop:gen_prop_scale} reveals the scaling and vertical-shift invariance of generalized convexity and smoothness. Item~\ref{prop:gen_prop_dataset} shows that the generalized convexity and smoothness of the dataset are identical to those at the sample level. Item~\ref{prop:gen_prop_transform_keep} reveals the relationship between the generalized convexity and smoothness of the Fenchel--Young loss and those of its generating function.
It is well established that a wide range of commonly used loss functions in statistics and machine learning, including CrossEntropy loss, mean squared error (MSE) loss, and perceptron loss, can be expressed as Fenchel--Young losses \(d_F(y, f_\theta(x))\) through an appropriate choice of \(F\)~\citep{Blondel2019LearningWF}. 
Proposition~\ref{prop:gen_prop_transform_keep} demonstrates that the generalized smoothness and convexity of the Fenchel--Young loss can be analyzed through its generating function.

In practical DNN training, the corresponding objective function \(G(\mu)\) can often be modeled as \(\mathcal{H}(\Psi)\)-smooth on the parameter region relevant to training. Indeed, for objective functions satisfying suitable regularity and local growth conditions, the remainder inequality
\[
S_G(\mu,\nu)\le \Psi(\mu-\nu),
\qquad
\Psi(x)=\frac{\|a_\Psi x\|_2^{r_\Psi}}{r_\Psi}+c_\Psi,
\]
can be enforced on a prescribed bounded parameter region by choosing the scale \(a_\Psi\) and the exponent \(r_\Psi\) sufficiently large (and, if necessary, adjusting the constant \(c_\Psi\)), so that \(G(\mu)\) is \(\mathcal{H}(\Psi)\)-smooth there. This is not claimed to hold globally and unconditionally for arbitrary functions, but rather for the class of training objectives that satisfy the above regularity and growth conditions.
 
According to the definitions of \(\mathcal{H}(\psi)\)-convexity and \(\mathcal{H}(\Psi)\)-smoothness, \(G\) need not be convex or smooth. When \(\psi(0)=0\) the function is convex, whereas a strictly negative \(\psi(0)\) yields non-convex behavior; the constant \(\psi(0)\) therefore unifies convex and non-convex objectives. Likewise, \(\Psi(0)\) unifies smooth and non-smooth objectives. Henceforth we analyze both regimes uniformly through \(\mathcal{H}(\psi)\)-convexity and \(\mathcal{H}(\Psi)\)-smoothness, without distinguishing convex from non-convex or smooth from non-smooth.

\subsection{Connecting Gradients and Optimal Solutions}
The following theorem bounds the suboptimality \(G(\mu)-\min_\nu G(\nu)\) in terms of the gradient energy, connecting gradient dynamics to global optimality without requiring classical convexity or smoothness. 
\begin{theorem}
\label{thm:h_bound}
Let \(\mu_*\) be a global minimizer of \(G\). Then, for all \(\mu\), the following inequalities hold: 
\begin{enumerate}
\item\label{thm:h_bound:1} If \(G(\mu)\) is \(\mathcal{H}(\Psi)\)-smooth, then:
\begin{equation}
\begin{aligned}
 \Psi^*(\nabla G(\mu)) &\leq G(\mu) - G(\mu_*) \leq \Psi(\mu - \mu_*).
 \end{aligned}
\end{equation}

\item\label{thm:h_bound:2} If \(G(\mu)\) is \(\mathcal{H}(\psi)\)-convex, then:
\begin{equation}
\begin{aligned}
 \psi(\mu - \mu_*) &\leq G(\mu) - G(\mu_*)\leq \psi^*(\nabla G(\mu)).
 \end{aligned}
\end{equation}

\item\label{thm:h_bound:3} If \(G(\mu)\) is both \(\mathcal{H}(\Psi)\)-smooth and \(\mathcal{H}(\psi)\)-convex, then:
\begin{equation}
 \begin{aligned}
 \Psi^*(\nabla G(\mu)) &\leq G(\mu) -  G(\mu_*) \leq \psi^*(\nabla G(\mu)).
 \end{aligned}
\end{equation}

\end{enumerate}
\end{theorem}
The proof of this theorem is provided in Appendix~\ref{appendix:proof_h_bound}.

This theorem shows that \(\mathcal{H}(\psi)\)-convexity and \(\mathcal{H}(\Psi)\)-smoothness symmetrically relate global optimality to gradient behavior, without requiring classical convexity or smoothness. This relationship provides a rigorous framework for analyzing the optimization dynamics of gradient-based algorithms in complex settings.

Building upon Theorem~\ref{thm:h_bound}, we derive the following corollary:
\begin{corollary}
\label{cor:h_bound}
If \(G(\mu)\) is both \(\mathcal{H}(\psi)\)-convex and \(\mathcal{H}(\Psi)\)-smooth, with \(\Psi(\mu) = \frac{k_\Psi}{2}\|\mu\|_2^2\) and \(\psi(\mu) = \frac{k_\psi}{2}\|\mu\|_2^2\), where \(k_\Psi \ge k_\psi > 0\), the following holds: 
\begin{equation}
 \frac{\|\nabla G(\mu)\|_2^2}{2 k_\Psi} \leq G(\mu) - G(\mu_*) \leq \frac{\|\nabla G(\mu)\|_2^2}{2 k_\psi},
\end{equation}
where \(\mu_*\) is a global minimizer of \(G\). 
\end{corollary}
This recovers the classical bounds under Lipschitz smoothness and strong convexity (Lemma~\ref{lem:convex_smooth_dual}), verifying the generality of Theorem~\ref{thm:h_bound}.

\subsection{Determining Generalized Convexity and Smoothness}
 
The definitions of \(\mathcal{H}(\psi)\)-convexity and \(\mathcal{H}(\Psi)\)-smoothness are general and can be applied to a broad class of neural network objectives. 
However, this generality introduces a challenge: a single objective function may admit multiple valid choices of \((\psi, \Psi)\). To illustrate, consider an \(\mathcal{H}(\Psi)\)-smooth function \(G(\theta)\); one can construct \(\Psi'(\cdot) = \Psi(\cdot)+c\) with \(c>0\), and \(G(\theta)\) is then also \(\mathcal{H}(\Psi')\)-smooth. 
An analogous issue exists for \(\mathcal{H}(\psi)\)-convexity.

While different constructions of \(\Psi\) still satisfy the gradient-energy inequalities of Theorem~\ref{thm:h_bound}, they may render the resulting bounds on the global optimum too loose to offer practical guidance for optimization. 
This is analogous to classical GD/SGD convergence analysis under the Lipschitz smoothness assumption: if the employed Lipschitz constant is excessively large (which is indeed the case for DNNs, where the constant often does not exist or is extremely large~\citep{VirmauxS18,liu2022learning}), the resulting optimal learning rate approaches zero. Such a prediction fails to reflect the actual convergence behavior observed with realistic learning rates. Therefore, the construction of the core parameters for \(\mathcal{H}(\psi)\)-convexity and \(\mathcal{H}(\Psi)\)-smoothness directly impacts the practical utility of the theory. 

By Theorem~\ref{thm:h_bound}, the gradient energy controls the suboptimality gap \(G(\mu)-G(\mu_*)\) between the current iterate and a global optimum. When the objective is both \(\mathcal{H}(\Psi)\)-smooth and \(\mathcal{H}(\psi)\)-convex, the upper and lower bounds attain their minima at gradient energies \(\Psi^*(0)\) and \(\psi^*(0)\), respectively. By Proposition~\ref{prop:energy_fun}, \(\Psi(0)+\Psi^*(0)=0\); hence these minima are \(-\Psi(0)\) and \(-\psi(0)\), and their gap is \(\Psi(0)-\psi(0)\), with \(\Psi(0)\ge0\) and \(\psi(0)\le0\). Minimizing this gap therefore requires both \(\Psi(0)\) and \(\psi(0)\) to be as close to zero as possible.

Under \(\mathcal{H}(\Psi)\)-smoothness, Theorem~\ref{thm:general_gd_convergence} shows that the gradient energy of generalized GD converges to \(\Psi^*(0)=-\Psi(0)\), which is thus the optimization limit of gradient-descent-type algorithms on a smooth objective. Since the generalized convexity and smoothness characterize the objective once the energy functions are fixed, they govern the limit achievable by directly applying gradient-descent-type algorithms. Theorems~\ref{thm:general_gd_convergence} and~\ref{thm:general_sgd_convergence} both imply that, for the relevant energy function \(\Phi\), a smaller \(\Phi(0)\) favors convergence to a smaller gradient energy. For DNN training, Theorem~\ref{thm:com_general_bound} further indicates that keeping both \(\Psi(0)\) and \(\psi(0)\) close to zero benefits model training.

In summary, these observations motivate the following construction principles:
\begin{proposition}\label{prop:generalized_principle}
For \((\psi,\Psi)\) to be useful in optimization analysis, the following two construction principles should be satisfied: 
\begin{enumerate}
 \item \(\psi(0)\) and \(\Psi(0)\) should be as close to 0 as possible;
 \item The pair \((\Psi, \psi)\) should provide tight upper and lower bounds for the remainder \(S_G(\theta,\eta)\).
\end{enumerate}
\end{proposition}

Precisely computing the globally optimal \(\Psi\) for complex DNNs is challenging; indeed, we conjecture that this problem is NP-hard, consistent with the established NP-hardness of computing global Lipschitz constants for DNNs~\citep{VirmauxS18,liu2022learning}. Importantly, however, this computational challenge does not negate the theoretical or practical value of our framework.  
Theoretically, the \(\mathcal{H}(\psi)\)-convexity and \(\mathcal{H}(\Psi)\)-smoothness framework generalizes classical strictly convex and Lipschitz smooth frameworks. It provides relationships between gradients and global optima for non-convex and non-differentiable problems, alleviating the difficulties classical frameworks face when analyzing DNNs. In Section~\ref{sec:com_opt}, we utilize this to explain the non-convex optimization mechanism of SGD for DNNs. 
Practically, the convergence properties under these generalized conditions guide architectural design towards better optimization landscapes. By adhering to the construction principles (Proposition~\ref{prop:generalized_principle}), one can design modules for DNNs with provably favorable spectral properties, enhancing trainability without explicit computation of the global constants.

For several important special cases, we derive the following criteria for generalized convexity and smoothness. 
\begin{proposition}
\label{prop:convex_smooth_case}
\begin{enumerate}
\item \label{prop:convex_smooth_case_hessian} Suppose \(G\) is a continuously differentiable function defined on a closed convex set, twice differentiable on the interior of its domain, and with a positive definite Hessian matrix. Let \(\lambda_{\min}\) and \(\lambda_{\max}\) denote the smallest and largest eigenvalues of the Hessian matrix \(\nabla^2 G(\mu)\) for all \(\mu\) in the domain of \(G\). Then \(G\) is both \(\mathcal{H}(\lambda_{\max}\|\cdot\|_2^2/2)\)-smooth and \(\mathcal{H}(\lambda_{\min}\|\cdot\|_2^2/2)\)-convex.
\item ~\label{prop:convex_smooth_case_mse}The MSE loss \(\frac{1}{2}\|f_\theta(x)-y\|_2^2\) with respect to \(f_\theta(x)\) is both \(\mathcal{H}(\|\cdot\|_2^2/2)\)-convex and \(\mathcal{H}(\|\cdot\|_2^2/2)\)-smooth. 
\item ~\label{prop:convex_smooth_case_ce}The Softmax CrossEntropy loss \(\mathrm{CrossEntropy}(q,\mathrm{Softmax}(f_\theta(x)))\) with respect to \(f_\theta(x)\) is both \(\mathcal{H}(\min_i p_i(x) \|\cdot\|_2^2)\)-convex and \(\mathcal{H}(2 \ln 2 \|\cdot\|_2^2)\)-smooth, where \(p(x)=\mathrm{Softmax}(f_\theta(x))\) and \(q\) is the target distribution. 
\end{enumerate}
\end{proposition}
The proof of this proposition is provided in Appendix~\ref{appendix:proof_convex_smooth_case}.

\section{Generalized GD and SGD}
\label{sec:methods}
In this section, we introduce a more general form of gradient-based optimization, which we refer to as generalized GD and generalized SGD, and analyze their convergence rates under generalized convexity and generalized smoothness.  

\subsection{Generalized GD}
\subsubsection{Definitions}

\begin{definition}[Generalized GD]
\label{def:general_gd}
Consider the deterministic optimization problem
\begin{equation}
 \min_\theta G(\theta),
\end{equation}
where \(G\) is an \(\mathcal{H}(\Psi)\)-smooth objective. We define its \textit{gradient energy} as
\begin{equation}
 \Psi^*(\nabla G(\theta)),
\end{equation}
with \(\Psi\) the energy function fixed by the smoothness of \(G\).

The generalized GD method with constant step size \(\alpha\) then follows the update rule
\begin{equation}
 \theta_{k+1} := \theta_k - \alpha(\nabla G(\theta_k))_{\Psi^*}^*,
\end{equation}
where \((\nabla G(\theta_k))_{\Psi^*}^* := \nabla\Psi^*(\nabla G(\theta_k))\).
\end{definition}

The gradient energy measures the distance of the gradient from the origin, and from a conceptual viewpoint it characterizes the stability of the objective function with respect to the model parameters. Generalized GD no longer proceeds strictly along the gradient direction; instead, it moves along the conjugate direction \(\nabla\Psi^*(\nabla G(\theta_k))\). We now present a proposition showing that the update direction coincides with the gradient direction if and only if \(\Psi(\cdot) = \Psi_\circ(\|\cdot\|_2)\). The formal statement is as follows.

\begin{proposition}
\label{prop:l2_unique}
Let \(n\ge 2\). If, for any nonzero vector \(x\in\mathbb{R}^n\), the gradient \(\nabla\|x\|\) is collinear with the original vector \(x\), i.e., there exists a scalar-valued function \(\lambda(x)\) such that
\begin{equation}\label{eq:collinear_assumption}
\nabla\|x\| = \lambda(x) \cdot x,
\end{equation}
then there exists a positive constant \(c>0\) such that, for all \(x\in\mathbb{R}^n\), we have
 \(\|x\| = c \cdot \|x\|_2,\) 
i.e., the norm must be a positive scalar multiple of the Euclidean (\(L_2\)) norm. Conversely, all positive scalar multiples of the Euclidean norm satisfy the property that the gradient is collinear with the original vector.
\end{proposition}
The proof of this proposition can be found in Appendix~\ref{appendix:proof_l2_unique}.
Since \(\Psi^*(\nu)=\Psi_\circ^*(\|\nu\|_*)\), the update direction is \(\nabla\Psi^*(g)=\Psi_\circ^{*\prime}(\|g\|_*)\,\nabla\|g\|_*\), which is collinear with \(g\) if and only if \(\nabla\|\cdot\|_*\) is collinear with its argument. By Proposition~\ref{prop:l2_unique}, and because the dual of a scalar multiple of \(\|\cdot\|_2\) is again a scalar multiple of \(\|\cdot\|_2\), this happens if and only if \(\Psi(\cdot)=\Psi_\circ(\|\cdot\|_2)\).

When the smoothness energy is \(\Psi(\cdot)=\frac{L}{2}\|\cdot\|_2^2\), generalized GD reduces to classical GD. 
For comparison, the classical GD method with a constant step size \(\alpha\) is defined by the following update rule:
\begin{equation}\label{eq:gd_update}
 \theta_{k+1}=\theta_k-\alpha \nabla G(\theta_k).
\end{equation}

\subsubsection{Convergence Analysis}
Next, we establish the convergence rate of generalized GD for the deterministic optimization problem.
\begin{theorem}\label{thm:general_gd_convergence}
Consider the deterministic optimization problem in Definition~\ref{def:general_gd}. Suppose \(G(\theta)\) is \(\mathcal{H}(\Psi)\)-smooth. Then the generalized GD algorithm satisfies the following properties: 

\begin{enumerate}

 \item \label{thm:general_gd_convergence1}The optimal learning rate is given by \(\alpha = 1\). Substituting this value into the update rule yields \begin{equation}
 G(\theta_{k+1}) \leq G(\theta_k) - \Psi^*(\nabla G(\theta_k)).
 \end{equation}

 \item \label{thm:general_gd_convergence2}The number of iterations required to satisfy the convergence criterion
 \begin{equation}
 \Psi^*(0)\leq \Psi^*\bigl(\nabla G(\theta_k)\bigr) \leq \Psi^*(\varepsilon)
 \end{equation}
 is bounded by \(T = \mathcal{O}\!\bigl(1 / \Psi^*(\varepsilon)\bigr)\).

 \item \label{thm:general_gd_convergence3}If \(G(\theta)\) is both \(\mathcal{H}(\psi)\)-convex and \(\mathcal{H}(\Psi)\)-smooth, the number of iterations required for the generalized GD algorithm to achieve the error bound
 \begin{equation}
 G(\theta_{k}) - G_*\leq \psi^*(\varepsilon)
 \end{equation}
 is bounded by \(T=\mathcal{O}(1/\Psi^*(\varepsilon))\).
 \item \label{thm:general_gd_convergence4}Assume that \(G(\theta)\) is both \(\mathcal{H}(\psi)\)-convex and \(\mathcal{H}(\Psi)\)-smooth, where \(\Psi\) and \(\psi\) satisfy \(\Psi^*(\mu)\geq t\psi^*(\mu)+b\), \(0<t< 1\) and \(b\in \mathbb{R}\). The number of iterations for the generalized GD algorithm to reach
 \begin{equation}
G(\theta_{k}) - G_*
 \leq \psi^*(\varepsilon)-\psi^*(0)-\frac{b}{t}
 \end{equation}
 is bounded by \(T = \mathcal{O}\left(\kappa \log \frac{1}{\psi^*(\varepsilon)-\psi^*(0)}\right)\), where the condition number is defined as \(\kappa = -{1}/{\log(1-t)}\).

\end{enumerate}
\end{theorem}
The proof of this theorem can be found in Appendix~\ref{appendix:proof_general_gd_convergence}.

The above theorem provides the following insight: when the gradient energy is used to control the distance to the optimum, a smaller \(\psi^*(0)\) (i.e., \(\psi(0)\) closer to zero) gives tighter control; in particular, if the gradient energy vanishes, \(\psi^*(0)\) upper-bounds the distance to the optimum, and when \(\psi^*(0)=0\) a vanishing gradient energy certifies optimality. In summary, the smaller the absolute values of \(\psi(0)\) and \(\Psi(0)\), the tighter the achievable guarantee, and the exponent of the energy function determines the convergence rate. 

We now use the above theorem to derive the convergence of classical GD under the Lipschitz smoothness assumption. 
When \(G(\theta)\) is Lipschitz smooth with constant \(L\), this condition is equivalent to the function being bounded above by a quadratic: 
 \(G(y) \leq G(x) + \nabla G(x)^\top (y - x) + \frac{L}{2} \|y - x\|_2^2\). 
Therefore, according to the definition of \(\mathcal{H}(\Psi)\)-smoothness, we have \(\Psi(\cdot)=\frac{L}{2}\|\cdot\|_2^2\) and \(\Psi^*(\cdot)=\frac{1}{2L}\|\cdot\|_2^2\). In this case, generalized GD with step size \(\alpha\) coincides with classical GD with learning rate \(\alpha/L\). 
Similarly, if \(G(\theta)\) is \(\sigma\)-strongly convex, then according to the definition of \(\mathcal{H}(\psi)\)-convexity, we have \(\psi(\cdot)=\frac{\sigma}{2}\|\cdot\|_2^2\) and \(\psi^*(\cdot)=\frac{1}{2\sigma}\|\cdot\|_2^2\). 
By Theorem~\ref{thm:general_gd_convergence}, we obtain the following corollary. 
\begin{corollary}
 \label{cor:gd_strongc_lips}
 Assume that \(G(\theta)\) is both \(\mathcal{H}(\psi)\)-convex and \(\mathcal{H}(\Psi)\)-smooth with \(\Psi(\cdot)=\frac{L}{2}\|\cdot\|_2^2\) and \(\psi(\cdot)=\frac{\sigma}{2}\|\cdot\|_2^2\).
 \begin{enumerate}
 \item For the equivalent classical GD update, the optimal learning rate is \(\alpha = \frac{1}{L}\), and any \(0 < \alpha < \frac{2}{L}\) guarantees strict descent at each step.

 \item The number of iterations required for the GD algorithm to reach the convergence condition
\begin{equation}
 G(\theta_{k+1}) - G_*\leq \frac{\varepsilon^2}{2\sigma}
\end{equation}
 is \(T = \mathcal{O}(\kappa \log \frac{2\sigma}{\varepsilon^2})\), where the condition number is defined as \(\kappa = -\frac{1}{\log(1-t)}\) and \(t=\frac{\sigma}{L}\).
 \end{enumerate}
\end{corollary}
This is consistent with the classical convergence rate under strong convexity and Lipschitz smoothness, which supports the validity and the generality of the proposed framework.
In addition to generalizing classical GD, generalized GD offers two advantages under generalized convexity and smoothness:
\begin{enumerate}
\item It provides convergence rates for non-convex objectives, where the error to the optimum is controlled by \(\psi(0)\). If \(\psi(0)\) is near \(0\) (nearly convex), the bound is effective, corresponding to approximately convex optimization. If \(\psi(0)\) is very negative (i.e., \(\psi^*(0)\) is large), the bound becomes too loose to ensure global optimality, motivating the composite framework below.
\item Under general \(\mathcal{H}(\Psi)\)-smoothness, classical GD requires an explicit step size: the standard choice \(1/L\) attains the optimal worst-case complexity bound. In contrast, generalized GD absorbs the step size into the conjugate direction \(\nabla\Psi^*(\cdot)\), yielding the parameter-free update 
\[
\theta_{k+1}=\theta_k-\nabla\Psi^*(\nabla G(\theta_k)).
\]
The parameter-free update corresponds to the standard step size \(1/L\) when
\(\Psi(\cdot)=\frac{L}{2}\|\cdot\|_2^2\); it removes the need to estimate the energy parameters and gives a more compact form.
\end{enumerate}

\subsection{Generalized SGD}
While generalized GD provides convergence guarantees for deterministic problems, the stochastic setting of DNN training introduces an additional phenomenon that a deterministic analysis cannot capture: \emph{limited model capacity}, whereby no algorithm can drive the empirical risk to zero once the data are more complex than the model can represent. We therefore introduce generalized SGD together with the model capacity risk; generalized SGD recovers classical SGD as a special case under Lipschitz smoothness.

\subsubsection{Definitions}

\begin{definition}[Generalized SGD]
\label{def:general_sgd}
Let \(s\) denote a sampled dataset whose elements take values in the space \(\mathcal{Z}\). We consider the stochastic optimization problem
\begin{equation}
 \min_\theta G(\theta,s) = \frac{1}{|s|}\sum_{z\in s}G(\theta,z) = \mathbb{E}_{Z \sim \hat{q}} [G(\theta, Z)],
\end{equation}
where \(G(\theta, z)\) is \(\mathcal{H}(\Psi)\)-smooth and \(\mathcal{H}(\psi)\)-convex with respect to \(\theta\) for all \(z \in \mathcal{Z}\), and \(\hat{q}\) is the empirical distribution induced by the sample set \(s\). We define its gradient energy as
\begin{equation}
 \mathbb{E}_{Z\sim\hat{q}} \Psi^*(\nabla_\theta G(\theta,Z)).
\end{equation}

The generalized SGD method with constant step size \(\alpha\) is given by the update rule
\begin{equation}
 \theta_{k+1} = \theta_{k} - \alpha (g(\theta_k, s_k))_{\Psi^*}^*,
\end{equation}
where \(g(\theta_k, s_k)=\nabla_\theta G(\theta_k, s_k)\) is the stochastic gradient at iteration \(k\), \((g(\theta_k, s_k))_{\Psi^*}^* := \nabla\Psi^*(g(\theta_k, s_k))\), \(\theta_k\) denotes the parameter value at iteration \(k\), and \(s_k \subseteq s\) is the mini-batch sampled at iteration \(k\) with cardinality \(|s_k|\).  
\end{definition}

We note that the gradient energy is defined through the per-sample gradient \(\nabla_\theta G(\theta, Z)\) rather than the mini-batch gradient \(\nabla_\theta G(\theta, s)\); this distinction is essential because the energy averages over the empirical distribution \(\hat{q}\) instead of a single batch. The deterministic optimization problem is recovered as the special case in which \(Z\) follows a Dirac distribution \(\delta_z\). Since most machine learning training problems based on sampled data admit such a stochastic formulation, the stochastic setting and its primary algorithm, generalized SGD, constitute the main focus of this work. 
Similar to generalized GD, generalized SGD does not move strictly along the gradient direction; rather, it moves along the conjugate direction \((g(\theta_k,s_k))_{\Psi^*}^*\). This update direction coincides with the mini-batch gradient direction if and only if \(\Psi(\cdot) = \Psi_\circ(\|\cdot\|_2)\), in which case generalized SGD reduces exactly to classical SGD. 

The capacity of the model also affects the final convergence result during DNN training. If the model capacity is sufficiently large, we can expect the model to fit all samples, driving the gradient energy toward zero. Conversely, when the capacity is insufficient, the model cannot fully fit all samples, and the gradient energy remains positive. The model capacity risk is defined as follows. 
\begin{definition}[Model Capacity Risk]
\label{def:model_risk}
Let \(s^m\) be a batch of size \(m\) randomly sampled from \(s\), and let \(Z\) be a single draw from the empirical distribution \(\hat q\) of \(s\) (i.e., a uniform random element of \(s\)). Let
\[
    G(z)_* = \min_\theta G(\theta,z), \qquad
    G(s^m)_* = \min_\theta G(\theta,s^m)
\]
denote the global minima of the per-sample loss \(G(\theta,z)\) and the batch loss \(G(\theta,s^m)\), respectively. The model capacity risk is defined as
\begin{equation}
    R_G(s,m) = \mathbb{E}\bigl[G(s^m)_*\bigr] - \mathbb{E}\bigl[G(Z)_*\bigr].
\end{equation}
\end{definition}

The model capacity risk characterizes the discrepancy, arising from insufficient model capacity, between the minimal loss achievable on a batch of size \(m\) and the expected minimal loss on a single sample. Consider the MSE loss as an example: a simple linear model can achieve zero fitting error for a single sample, but due to bounded model capacity, it generally cannot achieve zero fitting error for all samples. This limitation is independent of the optimization algorithm and is determined by the model's capacity; we abstract it as the model capacity risk.

\subsubsection{Convergence Analysis}
We now prove an important property of the model capacity risk, namely that \(R_G(s,m)\) is monotonically non-decreasing in \(m\):

\begin{proposition}
 \label{prop:model_risk_prop}
\(R_G(s,m)\) is monotonically non-decreasing in batch size \(m\):
\begin{equation}
 R_G(s,m+1)\ge R_G(s,m).
\end{equation}
Consequently, \(R_G(s,m)\) attains its minimum \(0\) at \(m=1\) and its maximum at \(m=|s|\).
\end{proposition}
The proof of this proposition is detailed in Appendix~\ref{appendix:proof_model_risk_prop}.

The following theorem delineates the convergence behavior of generalized SGD within the \(\mathcal{H}(\Psi)\)-smooth setting.

\begin{theorem}[Stochastic gradient-energy bound and risk-gap bound]
\label{thm:general_sgd_convergence}
Let \(n=|s|\) and \(m=|s_k|\), and let the step size be \(\alpha=1\). Define the stochastic gradient noise
\[
\xi_k:=\nabla G(\theta_k,s_k)-\nabla G(\theta_k,s).
\]
Assume that \(\Psi\) is \(\mathcal{H}(\Pi)\)-convex for some energy function \(\Pi\), and that for every \(k\ge 0\),
\[
\mathbb E[\xi_k\mid\theta_k]=0,
\qquad
\mathbb E[\Pi^*(\xi_k)\mid\theta_k]\le \Sigma_\Pi.
\]
Then the following statements hold.
\begin{enumerate}
\item[(i)] For any \(\varepsilon>0\) with \(\Psi^*(\varepsilon)>2\Sigma_\Pi\), let
\begin{equation}
\label{eq:T_epsilon}
T_\varepsilon:=\max\left\{1,\left\lceil
\frac{G(\theta_0,s)-G(s)_*}{\Psi^*(\varepsilon)-2\Sigma_\Pi}
\right\rceil\right\}.
\end{equation}
Then
\begin{equation}
\label{eq:energy_bound}
\min_{0\le k<T_\varepsilon}
\mathbb E[\Psi^*(\nabla G(\theta_k,s_k))]
\le \Psi^*(\varepsilon).
\end{equation}

\item[(ii)] If there exist \(t>0\) and \(b\in\mathbb R\) such that
\[
\Psi^*(w)\ge t\psi^*(w)+b,\qquad \forall w,
\]
then
\begin{equation}
\label{eq:risk_gap_bound}
\min_{0\le k<T_\varepsilon}
\bigl\{\mathbb E[G(\theta_k,s)]-G(s)_*\bigr\}
\le \frac{\Psi^*(\varepsilon)-b}{t}
+R_G(s,m)-R_G(s,n)
\le \frac{\Psi^*(\varepsilon)-b}{t}.
\end{equation}
\end{enumerate}
\end{theorem} 
The proof of this theorem is detailed in Appendix~\ref{appendix:proof_general_sgd_convergence}.

Importantly, the quantity \(\Sigma_\Pi\) measures the noise and correlation level of the stochastic gradient within a mini-batch, and its magnitude determines the plateau \(2\Sigma_\Pi\) in the achievable gradient energy. A smaller \(\Sigma_\Pi\) therefore yields a better attainable gradient-energy level, while a larger \(\Sigma_\Pi\) prevents the gradient energy from vanishing (see Section~\ref{sec:com_opt}).

\begin{corollary}[Geometric convergence]
\label{cor:general_sgd_geometric}
Under the conditions of Theorem~\ref{thm:general_sgd_convergence}, let the step size be \(\alpha=1\) and define \(
A_k:=\mathbb E[G(\theta_k,s)]-G(s)_*. \) If there exist \(t\in(0,1]\) and \(b\in\mathbb R\) such that \(
\Psi^*(w)\ge t\psi^*(w)+b\) for all \(w\), then for \(0<t<1\) with \(A_0>(\Sigma_\Pi-b)/t\) and \(\delta>(\Sigma_\Pi-b)/t\), whenever
\[
T\ge
\frac{\log\bigl((A_0-c/t)/(\delta-c/t)\bigr)}
{-\log(1-t)},
\qquad c:=\Sigma_\Pi-b,
\]
it holds that \(A_T\le\delta\). When \(t=1\), we have \(A_T\le \Sigma_\Pi-b\).
\end{corollary}
The proof of this corollary is detailed in Appendix~\ref{appendix:proof_general_sgd_geometric}.

\begin{rmk}[Sign and physical meaning of \(R_G(s,m)-R_G(s,n)\)]
Under the definition of the model capacity risk (Definition~\ref{def:model_risk}), the capacity-risk term in \eqref{eq:risk_gap_bound} is
\(R_G(s,m)-R_G(s,n)\), whose sign is negative because \(m\le n\).
Its physical meaning is as follows. The batch risk gap \(G(\theta_k,s_k)-G(s_k)_*\le(\Psi^*(u_k)-b)/t\)
gives the suboptimality of \(\theta_k\) relative to the \textbf{batch optimum}.
In contrast, the full-dataset risk gap \(\mathbb E[G(\theta_k,s)]-G(s)_*\) is measured relative to the full-dataset optimum. Since \(G(s_k)_*\le G(s)_*\) (optimizing over a subset is easier), the batch optimum systematically underestimates the full-dataset optimum, with a discrepancy of
\(R_G(s,n)-R_G(s,m)\ge 0\). Therefore, when converting from batch suboptimality to full-dataset suboptimality, this discrepancy must be subtracted, i.e.,
\[
    \bigl[\mathbb E G(\theta_k,s)-G(s)_*\bigr]
    =\bigl[\mathbb E G(\theta_k,s)-\mathbb E G(s_k)_*\bigr]
    -\bigl[R_G(s,n)-R_G(s,m)\bigr].
\]
This explains the necessity of the term \(-R_G(s,n)\) in \eqref{eq:risk_gap_bound}. 
When \(m=n\), \(R_G(s,n)-R_G(s,n)=0\) and this term vanishes. 
\end{rmk}

Finally, the following special case recovers the classical SGD bound and serves as a consistency check that our framework is a genuine generalization of existing analyses rather than an unrelated alternative.
\begin{corollary}[Lipschitz special case]
\label{cor:lipschitz}
Suppose \(\Psi^*(w)=\frac{L_{\Psi^*}}{2}\|w\|^2\) and
\(\Pi^*(\xi)=\frac{L_{\Pi^*}}{2}\|\xi\|^2\),
and \(\mathbb E[\|\xi_k\|^2\mid\theta_k]\le\sigma^2\). Then
\(\Sigma_\Pi=\frac{L_{\Pi^*}}{2}\sigma^2\).
In particular, when \(L_{\Pi^*}=L_{\Psi^*}=L\), Conclusion (i) reduces to the stochastic gradient-energy bound of classical SGD,
\[
    \min_{0\le k<T}\mathbb E[\Psi^*(\nabla G(\theta_k,s_k))]
    \le\frac{G(\theta_0,s)-G(s)_*}{T}+L\sigma^2.
\]
\end{corollary}

We now analyze the conditions implied by Theorem~\ref{thm:general_sgd_convergence} and Corollary~\ref{cor:general_sgd_geometric}. For these guarantees to hold, the tolerance \(\varepsilon\) must satisfy
\begin{equation}
\begin{aligned}
    & \Psi^*(\varepsilon)>2\Sigma_\Pi,\\
    &\Psi^*(\varepsilon)\ge b,\\
    &\frac{\Psi^*(\varepsilon)-b}{t}
+R_G(s,m)-R_G(s,n)\ge 0,\\
& \delta>(\Sigma_\Pi-b)/t.
\end{aligned}
\end{equation}
This indicates that the gradient energy does not necessarily converge to zero; rather, \(\varepsilon\) must satisfy the above conditions. These conditions describe how model structure and capacity influence the optimum attainable by optimization:
\begin{enumerate}
    \item A smaller \(\Sigma_\Pi\) facilitates convergence to a better gradient energy level.
    \item A larger model capacity---equivalently, a smaller model capacity risk \(R_G(s,m)\)---facilitates convergence toward the global optimum, since \(R_G(s,m)\) appears with a positive sign in the target gaps of Conclusion~(ii) of Theorem~\ref{thm:general_sgd_convergence}.
    \item Corollary~\ref{cor:general_sgd_geometric} further shows that, under its stated assumptions, convergence can only be guaranteed for \(\delta \ge (\Sigma_\Pi-b)/t\); in other words, the existing assumptions do not guarantee convergence to the global minimum.
\end{enumerate}

These limitations stem from the intrinsic capacity of the model and from the limitations of the algorithm itself, and they are consistent with practice: we can guarantee neither that a stationary point is found for every sample nor that the global minimum is attained.

\section{Composite Optimization Problem}
\label{sec:com_opt}
This section develops a theoretical approach to a special class of non-convex and non-smooth optimization problems represented by DNN training. 
Although the framework developed above can characterize DNNs, applying it directly does not explain the strong convergence capability of SGD on DNNs. For DNN objectives, the convexity energy \(\psi\) typically has \(\psi(0)\) far from zero (very negative); consequently, even when the gradient energy vanishes, the loss may still be far from the global optimum, which contradicts practical experience. 
We therefore use the above framework and introduce a composite optimization problem to explain the underlying mechanism. 
The training objective of DNNs typically possesses a composite structure: a nonlinear mapping (e.g., a neural network) maps the input through parameters \(\theta\) to produce a predicted output, which is then evaluated by a task-specific loss function. This two-tier ``model--loss'' structure naturally forms a composite function. However, traditional optimization theory is primarily designed for smooth or convex functions and cannot directly capture the complex geometric properties of such composite objectives. 
To more accurately model the optimization problem of DNNs, this section formally defines composite optimization problems based on the generalized convexity and generalized smoothness introduced earlier, and presents an optimization mechanism based on gradient energy and the extreme eigenvalues of the model's structure matrix.

 \subsection{Definitions}
\label{subsec:model:compose:def}

Since all norms are equivalent in finite-dimensional spaces, i.e., there exist positive constants \(a\le b\) with \(a\|\mu\|_2\leq \|\mu\|\leq b\|\mu\|_2\), the following holds. If \(G(\cdot)\) is \(\mathcal{H}(\Psi)\)-smooth, then it is also necessarily \(\mathcal{H}(\Psi_\circ(b\|\cdot\|_2))\)-smooth, where \(\Psi(\mu)=\Psi_\circ(\|\mu\|)\). Similarly, \(\mathcal{H}(\psi)\)-convexity implies \(\mathcal{H}(\psi_\circ(a\|\cdot\|_2))\)-convexity. Replacing an arbitrary norm by the \(L_2\)-norm therefore still captures the essence of generalized convexity and smoothness, namely, bounding the first-order Taylor remainder by an energy function. Consequently, unless otherwise specified, we adopt the \(L_2\)-norm by default for the norm component of the energy function when applying the generalized convexity and smoothness defined in this paper to practical problems.

We define the composite optimization problem as follows:
\begin{definition}[Composite Optimization Problem]
\label{def:d_s_com_opt}
The deterministic composite optimization problem is defined as follows:
 \begin{equation}
 \label{eq:obj_st}
 \min_{\theta \in \mathbb{R}^m} G(\theta),
 \end{equation}
 where \(G(\theta) = F(h(\theta))\), and for almost all samples \(z\), the following conditions hold:
 \begin{itemize}
 \item The outer function \(F:\mathbb{R}^d\to \bar{\mathbb{R}}\) maintains \(\mathcal{H}(\phi)\)-convexity and \(\mathcal{H}(\Phi)\)-smoothness;
 \item The composite function \(G(\theta)\) is \(\mathcal{H}(\Psi)\)-smooth with respect to \(\theta\);
 \end{itemize}

The stochastic composite optimization problem is defined as follows:
 \begin{equation}
 \label{eq:obj_st_stochastic}
 \min_{\theta \in \mathbb{R}^m} \mathbb{E}_{Z\sim q} \big[ G(\theta,Z) \big],
 \end{equation}
 where \(G(\theta,z) = F(h_\theta(z))\), and for almost all samples \(z\), the following conditions hold:
 \begin{itemize}
 \item The outer function \(F:\mathbb{R}^d\to \bar{\mathbb{R}}\) maintains \(\mathcal{H}(\phi)\)-convexity and \(\mathcal{H}(\Phi)\)-smoothness;
 \item The composite function \(G(\theta,z)\) is \(\mathcal{H}(\Psi)\)-smooth with respect to \(\theta\);
 \end{itemize}
For consistency with the deterministic case, we also use \(G(\theta,s)\) to denote the stochastic composite optimization problem, where \(s\sim q^n\).
\end{definition}
Since the deterministic composite optimization problem can be viewed as a special case of the stochastic composite optimization problem when \(q\) is a Dirac delta at \(z\) or \(s=\{z\}\), we focus on the stochastic case for brevity.

\begin{definition}[Structure Matrix and its extreme eigenvalues]
\label{def:induced_matrix_norm}
For a stochastic composite optimization problem, let \(h_\theta: \mathbb{R}^{d_x} \to \mathbb{R}^d\) be a model parameterized by \(\theta \in \mathbb{R}^m\), where \(m\ge d\). The structure matrix of the model with respect to \(\theta\) is denoted by \(A_z = \mathbf{J}_\theta h_\theta(z)^\top \mathbf{J}_\theta h_\theta(z) \in \mathbb{R}^{d \times d}\), and its extreme eigenvalues are \(\lambda_{\max}(A_z)\) and \(\lambda_{\min}(A_z)\). For a sample \(z=(x,y)\) we also write \(A_x := A_z\) when the label plays no role. Throughout, we assume that \(\mathbf{J}_\theta h_\theta(z)\) has full row rank for almost all \(z\), so that \(\lambda_{\min}(A_z)>0\); this is the regime in which the bounds below are informative. 
\end{definition}

\subsection{Convergence Analysis}

\begin{theorem}
\label{thm:com_general_bound}
Consider the stochastic composite optimization problem \(G(\theta,s)\) defined in Definition~\ref{def:d_s_com_opt}, and let \(F_* = \min_{\nu\in\mathbb{R}^d} F(\nu)\). The following statements hold:
\begin{enumerate}
 \item \label{thm:indiv_sample_bound} For any sample \(z\):
 \begin{equation}
 \Phi^*(\frac{\nabla_\theta G(\theta, z)}{\lambda_{\max}(A_z)^{1/2}}) \le G(\theta, z) - F_* \le \phi^*(\frac{\nabla_\theta G(\theta, z)}{\lambda_{\min}(A_z)^{1/2}}).
 \end{equation}
 \item \label{thm:dataset_bound} For the dataset \(s\):
 \begin{equation}
 \mathbb{E}\Phi^*( \frac{ \nabla_\theta G(\theta, Z)}{\lambda_{\max}(A_Z)^{1/2}}) \le G(\theta, s) - F_* \le \mathbb{E}\phi^*( \frac{\nabla_\theta G(\theta, Z)}{\lambda_{\min}(A_Z)^{1/2}}).
 \end{equation}
\end{enumerate}
\end{theorem}
The proof can be found in Appendix~\ref{appendix:proof_com_general_bound}.

The above theorem indicates that \(G(\theta, s)\) can be optimized by reducing the gradient energy while keeping the structure matrix well conditioned, i.e., by increasing \(\lambda_{\min}(A_z)\) and preventing \(\lambda_{\max}(A_z)\) from growing excessively relative to \(\lambda_{\min}(A_z)\). Since \(G(\theta,z)\) is \(\mathcal{H}(\Psi)\)-smooth, Theorem~\ref{thm:general_sgd_convergence} implies that generalized SGD reduces the gradient energy, consequently reducing \(G(\theta, s)\). The fact that generalized SGD achieves a reduction in gradient energy without directly reducing the extreme eigenvalues provides a theoretical basis for solving composite optimization problems.

According to Theorem~\ref{thm:general_sgd_convergence}, under the \(\mathcal{H}(\Psi)\)-smoothness setting, the final convergence result of generalized SGD is constrained by the batch size \(m\), the noise-and-correlation level \(\Sigma_\Pi\), and the model capacity risk \(R_G(s, m)\).
Below, we analyze the impact of these factors on the convergence speed and solution quality for the composite optimization problem.
\begin{enumerate}
 \item \textbf{Impact of batch size.} The influence of batch size on convergence is multifaceted, involving the noise-and-correlation level \(\Sigma_\Pi\), the model capacity risk \(R_G(s,m)\), and \(\mathbb{E}\Psi^*(\nabla_\theta G(\theta,s_k))\). By Theorem~\ref{thm:com_general_bound}, the suboptimality gap in the composite optimization problem is controlled by the gradient energy \(\mathbb{E}\Psi^*(\nabla_\theta G(\theta,Z))\). Since \(\mathbb{E}[\Psi^*(\nabla_\theta G(\theta_k, s_k))]\) is itself a mini-batch estimate of this quantity, smaller batch sizes allow finer control over it. The choice of batch size in SGD therefore represents a fundamental tradeoff between two objectives: (i)~small batches enhance the precision of gradient-energy control, enabling convergence to better solutions; and (ii)~moderately larger batches improve convergence speed. For generalized SGD, the preferred strategy is hence to first use a larger \(m\) to converge quickly to a reasonably good solution, and then to reduce \(m\) so as to decrease the gradient energy further. This is consistent with the batch-size scheduling strategy commonly employed in practice for training DNNs.
\item \textbf{Impact of model capacity risk.} By Theorem~\ref{thm:general_sgd_convergence}, the suboptimality gap in the composite optimization problem is controlled by the model capacity risk \(R_G(s,m)\). A smaller model capacity risk \(R_G(s,m)\) tightens the upper bound on \(\mathbb{E}G(\theta_k,s) - \mathbb{E}G(Z)_*\), so that at the same gradient-energy level the excess over the optimal per-sample loss is more tightly controlled. A larger \(R_G(s,m)\) means that the batch optimum \(G(s)_*\) lies further from the per-sample optima---the samples cannot be fit simultaneously---which reflects a fundamental capacity limitation and loosens the bound rather than conferring any advantage.
\end{enumerate}

\subsection{Applications in Explaining DNN Trainability}
\label{sec:app}

In this section, we analyze the optimization mechanisms underlying the trainability of DNNs. 

\subsubsection{Bounds on Empirical Risk}

\label{subsubsec:def}
In deep learning, training DNNs is formulated as an empirical risk minimization (ERM) problem.
In this section, we reformulate the ERM problem by incorporating the concepts of composite optimization problems.
This allows us to provide theoretical support from our framework for gradient descent--type algorithms in optimizing DNN training.

We define the ERM problem and the constraints considered in this paper as follows:
\begin{definition}[ERM problem]
 \label{def:erm_opt}
 Let \(\ell\) be the loss function. The ERM problem on the dataset \(s=\{(X^{(i)},Y^{(i)})\}_{i=1}^n\) is defined as:
 \begin{equation}\label{eq:erm_obj}
 \begin{aligned}
 \min_\theta \mathcal{L}(\theta, s)&=\mathbb{E}_{XY}\ell(f_\theta(x),Y)=\mathbb{E}_Z\mathcal{L}(\theta,Z),
 \end{aligned}
 \end{equation}
where the expectations are taken with respect to the empirical distribution of the dataset \(s\).

 The loss function \(\ell(\cdot,\cdot)\) and the model \(f_\theta(x)\) satisfy the following conditions:
 \begin{itemize}
 \item \textbf{Condition 1:}~\label{con:1} \(\ell(f_\theta(x), y)\) (i.e., \(\mathcal{L}(\theta,z)\)) is \(\mathcal{H}(\phi)\)-convex and \(\mathcal{H}(\Phi)\)-smooth with respect to the model's output \(f_\theta(x)\), where \(\phi(0)\) is close to \(0\), and both \(\phi\) and \(\Phi\) are \(L_2\)-norm energy functions.
 \item \textbf{Condition 2:}~\label{con:2} \(\ell(f_\theta(x), y)\) is \(\mathcal{H}(\Psi)\)-smooth with respect to \(\theta\), where \(\Psi\) is an \(L_2\)-norm energy function.

 \end{itemize}
Define \(\ell_* = \min_{\nu\in \mathbb{R}^d} \ell(\nu, z)\), i.e., the minimum of the loss over the model output (not over the parameters); correspondingly, we write \(\mathcal{L}(z)_* = \min_\theta \mathcal{L}(\theta,z)\) for the per-sample minimum over the parameters.
\end{definition}
We comment on the requirement in Condition 1 that \(\phi(0)\) be close to \(0\), which implies that the loss is convex or nearly convex with respect to the model prediction. Indeed, practical supervised-learning loss functions are typically constructed to be convex (or strictly convex) in the model's output prediction to facilitate tractable optimization.
Let \(\hat{y}=f_\theta(x)\) denote the model prediction. Several standard loss functions satisfying this convexity property include MSE, Mean Absolute Error (MAE), Binary CrossEntropy (Log Loss), Hinge Loss, and Softmax CrossEntropy Loss. 
Accordingly, for widely used deep-learning losses one can choose \(\phi\) with \(\phi(0)=0\) (or close to \(0\)), so the conditions within our ERM formulation are reasonable and align with practical settings.

In deep learning, ERM is a prototypical composite optimization problem involving two fundamental components: the loss function and the model output. In the following subsections, we use the tools introduced above to characterize both the loss function and the model architecture, thereby deriving a principled understanding of the optimization mechanisms underlying deep learning. According to Theorem~\ref{thm:com_general_bound}, the following corollary establishes bounds on the ERM objective relative to the loss minimum \(\ell_*\), for the ERM problem defined in Definition~\ref{def:erm_opt}.

\begin{corollary}

 \label{cor:erm_bound}
For the ERM problem, the following results hold:
\begin{enumerate}
 \item \label{cor:erm_bound1} For any sample \(z\):
 \begin{equation*}
 \Phi^*\!\left(\frac{\nabla_\theta \mathcal{L}(\theta, z)}{\lambda_{\max}(A_z)^{1/2}}\right) \le \mathcal{L}(\theta, z) - \ell_* \le \phi^*\!\left(\frac{\nabla_\theta \mathcal{L}(\theta, z)}{\lambda_{\min}(A_z)^{1/2}}\right).
 \end{equation*}
 \item \label{cor:erm_bound2} For the dataset \(s\):
 \begin{equation*}
 \mathbb{E}\Phi^*\!\left(\frac{\nabla_\theta \mathcal{L}(\theta, Z)}{\lambda_{\max}(A_Z)^{1/2}}\right) \le \mathcal{L}(\theta, s) - \ell_* \le \mathbb{E}\phi^*\!\left(\frac{\nabla_\theta \mathcal{L}(\theta, Z)}{\lambda_{\min}(A_Z)^{1/2}}\right).
 \end{equation*}
\end{enumerate}
\end{corollary}

The above corollary indicates that the empirical risk of DNN training is governed by two controllable quantities: the gradient energy and the extreme eigenvalues of the structure matrix. 
    
The most widely used losses in deep learning are the MSE loss for regression and the Softmax CrossEntropy loss for classification, both belonging to the Fenchel--Young loss family. The former is generated by \(\Omega(\mu)=\frac{1}{2}\|\mu\|_2^2\), so that \(\mu_{\Omega^*}^*=\mu\); the latter corresponds to \(\Omega(p)=-H(p)\) with \(p\in\Delta\), so that \(\mu_{\Omega^*}^*=\mathrm{Softmax}(\mu)\). To ground our general bounds in practical settings, we specialize them to these two losses. Since the dataset-level risk is obtained by taking the expectation of the sample-level risk, the following corollary states only the sample-level gradient-energy bounds for brevity. 
\begin{corollary}
\label{cor:kl_mse_bound}
Assume that the model \(f_\theta(x)\) can fit any single sample, i.e., \(\mathcal{L}(z)_*=0\) for all \(z\in \mathcal{Z}\). Then the following results hold:
\begin{enumerate}
 \item For the MSE loss \(\mathcal{L}(\theta,z) = \frac{1}{2}\|y-f_\theta(x)\|_2^2\), we have
 \begin{equation*}
\begin{aligned}
 \frac{\|\nabla_\theta \mathcal{L}(\theta, z)\|_2^2}{2\lambda_{\max}(A_z)}&\le \mathcal{L}(\theta, z) \le \frac{\|\nabla_\theta \mathcal{L}(\theta, z)\|_2^2}{2\lambda_{\min}(A_z)}.
\end{aligned}
 \end{equation*}

 \item For the Softmax CrossEntropy loss \(\mathcal{L}(\theta,z)=\mathrm{CrossEntropy}(y,\mathrm{Softmax}(f_\theta(x)))\), where \(y\) is a one-hot label, we have
 \begin{equation*}
\begin{aligned}
 \frac{\|\nabla_\theta \mathcal{L}(\theta, z)\|_2^2}{8 \ln 2 \,\lambda_{\max}(A_z)}&\le \mathcal{L}(\theta, z) \le \frac{\|\nabla_\theta \mathcal{L}(\theta, z)\|_2^2}{4\min_i p_i(x)\lambda_{\min}(A_z)},
\end{aligned}
 \end{equation*}
 where \(p(x)=\mathrm{Softmax}(f_\theta(x))\) is the Softmax output, \(\min_i p_i(x)\) denotes its smallest component, and \(\min_i p_i(s)=\min_{x\in s}\min_i p_i(x)\).
\item \textbf{Fenchel Loss Bound:} Suppose \(\Omega\) is a function defined on a closed convex domain, continuously differentiable and twice differentiable on its domain, with a positive definite Hessian matrix. Let \(\mu_{\min}\) and \(\mu_{\max}\) denote the smallest and largest eigenvalues of the Hessian matrix \(\nabla^2 \Omega(\mu)\) for all \(\mu\) in the domain of \(\Omega\). Then we have
 \begin{equation*}
 \begin{aligned}
 \frac{\mu_{\min}\|\nabla_\theta \mathcal{L}(\theta, z)\|_2^2}{2\lambda_{\max}(A_z)}&\le \mathcal{L}(\theta, z) \le \frac{\mu_{\max}\|\nabla_\theta \mathcal{L}(\theta, z)\|_2^2}{2\lambda_{\min}(A_z)}.
 \end{aligned}
 \end{equation*}
\end{enumerate}
\end{corollary}
The proof of Corollary~\ref{cor:kl_mse_bound} is provided in Appendix~\ref{appendix:proof_kl_mse_bound}. 

The above corollary relates the empirical risk in DNN training to the gradient energy and the extreme eigenvalues of the structure matrix. Compared with Corollary~\ref{cor:erm_bound}, it is more concise because it assumes that the model \(f_\theta(x)\) can fit any single sample, i.e., \(\mathcal{L}(z)_*=0\) for all \(z\in\mathcal{Z}\). This assumption is reasonable in practice, since with a single training sample the loss can generally be driven close to zero, as is the case for MSE and Softmax CrossEntropy. Moreover, Theorem~\ref{thm:general_sgd_convergence} shows that generalized SGD reduces the gradient energy without directly affecting the extreme eigenvalues of the structure matrix, thereby providing a theoretical explanation for SGD-based DNN training.
Note, however, that when SGD reduces the gradient energy, the entries of the structure matrix may be suppressed, which would indirectly lower the extreme eigenvalues of the structure matrix and is therefore unfavorable for optimization. We conjecture that this is a reason why general non-convex optimization objectives are difficult to solve well: although the gradient energy decreases, the extreme eigenvalues of the structure matrix may decrease simultaneously, creating a trade-off. Consequently, in addition to SGD, we also need to explain, from the architectural perspective of DNNs, the mechanism that controls the extreme eigenvalues of the structure matrix.

\subsubsection{Control of the Extreme Eigenvalues of the Structure Matrix}
\label{subsubsec:independence}

In this subsection, we provide a mechanistic analysis of several typical properties of DNNs from the perspective of gradient energy and the structure matrix, explaining why these techniques help improve DNN trainability. We focus on over-parameterization and its associated phenomena: random parameter initialization and skip connections. 
    
Following~\citep{QiUnderstanding26}, when each column of the Jacobian matrix \(\mathbf{J}_\theta f(x)\) is independently and uniformly sampled from the ball \(\mathcal{B}^{m}_\epsilon = \{v\in\mathbb R^{m}:\|v\|_2 \le \epsilon\}\), the following result holds: 
\begin{proposition}
 \label{prop:number_para}
Let \(d\) denote the output dimension of \(f_\theta(x)\) and \(m\) the number of parameters. Define the auxiliary quantities
\begin{equation}
\eta := \frac{\sqrt{6\log d}}{\sqrt{m-1}}, \qquad
\zeta := 1 - \frac{2\log d}{m}.
\end{equation}
Under this assumption, with probability at least \(1 - O(1/d)\), the following bound holds:
\begin{equation}
\begin{aligned}
\frac{\lambda_{\max}(A_x)}{\lambda_{\min}(A_x)} &\le \frac{\zeta^2 + d\eta}{\zeta^2 - d\eta}.
\end{aligned}
\end{equation}
In the over-parameterized regime \(m \gg d^2 \log d\), we have \(\zeta > 0\) and \(\zeta^2 > d\eta\), so that the bound above is informative. Moreover, the ratio \(d\eta/\zeta^2\) is decreasing in \(m\) and increasing in \(d\).
\end{proposition}
    
Proposition~\ref{prop:number_para} shows that increasing \(m\) relative to \(d\) brings the ratio of the extreme eigenvalues closer to one. Provided that the eigenvalues themselves do not shrink, this keeps the risk bounds of Corollary~\ref{cor:erm_bound} tight, so that a reduction in gradient energy translates into a reduction in risk.
Building on the above theoretical result, we can theoretically characterize how the number of parameters affects the optimization properties of DNNs in two regimes under random initialization:
    \begin{enumerate}
            \item In the random-initialization regime (e.g., He or Xavier initialization), the Jacobian matrix \(\mathbf{J}_\theta f(x)\) can be approximated as independently and uniformly sampled from the ball \(\mathcal{B}^{m}_\epsilon = \{v\in\mathbb R^{m}:\|v\|_2 \le \epsilon\}\). Hence, the more parameters the model has, the smaller the ratio between the largest and smallest eigenvalues of the structure matrix; by Theorem~\ref{thm:com_general_bound}, the gap between the upper and lower risk bounds then shrinks, making the model easier to optimize.
            \item When the model is sufficiently over-parameterized, its parameters change little from their initial values after training, analogous to the NTK assumption that parameters remain close to initialization during training. In this regime, the Jacobian matrix \(\mathbf{J}_\theta f(x)\) can still be approximated as independently and uniformly sampled from the ball \(\mathcal{B}^{m}_\epsilon = \{v\in\mathbb R^{m}:\|v\|_2 \le \epsilon\}\), so that the number of parameters governs the trainability of the model.
    \end{enumerate}

We next consider depth. Current random initialization schemes use small variances, so that parameters are usually confined to a neighborhood of zero. Let \(f_\theta(x)\) denote an \(m\)-block sequential feed-forward network with \(f_\theta(x) = f_{\theta^{(m)}} \circ \cdots \circ f_{\theta^{(1)}}(x)\), and let \(J_k = \nabla_z f_{\theta^{(k)}}(z)\) denote the Jacobian of the \(k\)-th block at its input activation.
Under standard random initialization the block Jacobians are contractive, so that \(\delta_k := \sigma_{\max}(J_k)<1\), while the parameter-gradient factors \(\partial f_{\theta^{(k)}}/\partial\theta^{(k)}\) have small spectral norm. For a plain network, the corresponding block is
\[
\nabla_{\theta^{(k)}} f_\theta(x) =
\left( J_m \cdots J_{k+1} \right)
\frac{\partial f_{\theta^{(k)}}}{\partial \theta^{(k)}}(z_{k-1}).
\]
As the number of layers increases, the product of contractive block Jacobians decays exponentially, so the extreme eigenvalues of the structure matrix shrink and it becomes harder to control the risk through the gradient energy. This explains why plain deep networks become difficult to optimize as depth grows. 
Let \(g_\theta(x)\) denote the residual counterpart with skip-connected blocks \(g_{\theta^{(k)}}(z) = f_{\theta^{(k)}}(z) + z\). The following result explains the beneficial effect of skip connections from the perspective of the structure matrix. 
\begin{proposition}
    \label{prop:skip_con}
Let \(\delta_k = \sigma_{\max}(J_k)\) and \(\delta_U = \max_{1\le k\le m} \delta_k < 1\). Under standard initialization, the parameter-gradient factors \(\frac{\partial f_{\theta^{(k)}}}{\partial \theta^{(k)}}\) have spectral norms bounded above by \(\nu\), uniformly in \(k\). 
Let \(A'_x\) be the structure matrix of the residual network and \(A_x\) that of the plain network. Then the following bounds hold: 
\begin{equation}
\begin{aligned}
   \log \lambda_{\max}(A'_x) &\le \log C_\nu + 2m\log(1+\delta_U),\\
    \log \lambda_{\max}(A_x) &\le \log\left(\frac{\nu^2}{1-\delta_U^2}\right),
\end{aligned}
\end{equation}
where \(C_\nu = \nu^2 / ((1+\delta_U)^2 - 1)\). In particular, when \(\delta_U \ll 1\), the bound for the plain network reduces to a constant (approximately \(\log \nu^2\)), whereas the residual network's bound increases linearly with depth \(m\). 
\end{proposition}
Standard random parameter initialization typically produces parameters with zero mean and small variance, so that the block Jacobians are contractive, i.e., \(\sigma_{\max}(J_k)\ll 1\). 
Proposition~\ref{prop:skip_con} shows that, in this regime, skip connections allow \(\log \lambda_{\max}(A_x)\) to grow with depth; combined with the empirical observation that they likewise prevent the smallest eigenvalues from vanishing, this keeps the risk bounds of Corollary~\ref{cor:erm_bound} informative as depth increases. This provides a mechanistic account of how skip connections facilitate optimization through the spectral properties of the Jacobian matrix, complementing, rather than replacing, the classical explanation based on the prevention of gradient vanishing. In particular, it suggests that skip connections help a randomly initialized model reach an optimizable state quickly.

\section{Empirical Validation}
\label{sec:experiments}

This section presents an empirical validation of the core theoretical conclusions established in this paper through a series of targeted experiments.

\subsection{Experimental Configuration}

All experiments are implemented in Python 3.9 with PyTorch 2.2.2 and executed on a single GeForce RTX 2080 Ti GPU.

\subsection{Verification of GD/SGD Convergence}
\subsubsection{Experimental Designs}
Because the constant-step-size analysis of generalized SGD sets \(\alpha=1\), we first verify the prediction that the optimal learning rate of generalized GD is \(\alpha_* = 1\), under \(\mathcal{H}(\Phi)\)-smoothness. When \(\Phi(\cdot)=a\|\cdot\|_2^2\), generalized GD reduces to classical GD, which has been extensively studied; we therefore evaluate the framework on a family of non-Lipschitz objectives whose corresponding \(\Phi\) is tractable:
\begin{equation}
G(\theta) = \|\theta\|_2^r, \quad r \in [1.1, 2.0],
\end{equation}
where \(r\) is the polynomial exponent. We first establish the following technical lemma, which identifies the energy function \(\Phi\) used in the experiments:
\begin{lemma}\label{lem:nonlipsmooth}
Let \(G(\theta) = \|\theta\|_2^r\) with \(1 < r < 2\). Then \(G\) satisfies \(\mathcal{H}(\Phi)\)-smoothness, where \(\Phi(z) = 2^{2-r}\|z\|_2^r\) provides an explicit upper bound for the first-order Taylor remainder \(S_G(\eta,\theta)\).
\end{lemma}
The proof is provided in Appendix~\ref{appendix:proof_nonlipsmooth}. 

In our experiments, we initialized 100-dimensional parameters from \(\mathcal{N}(0,1)\) and optimized them with generalized GD using learning rates \(\alpha \in \{0.1, 0.3, 0.5, 0.7, 1.0, 1.5, 1.8, 2.0, 3.0\}\). We monitored \(\log_2 \|\nabla G(\theta_k)\|_2^2\), a monotone proxy for the gradient energy, to quantify convergence behavior.
\subsubsection{Experimental Results}

\begin{figure}[htbp]
\centering
\includegraphics[width=0.75\linewidth]{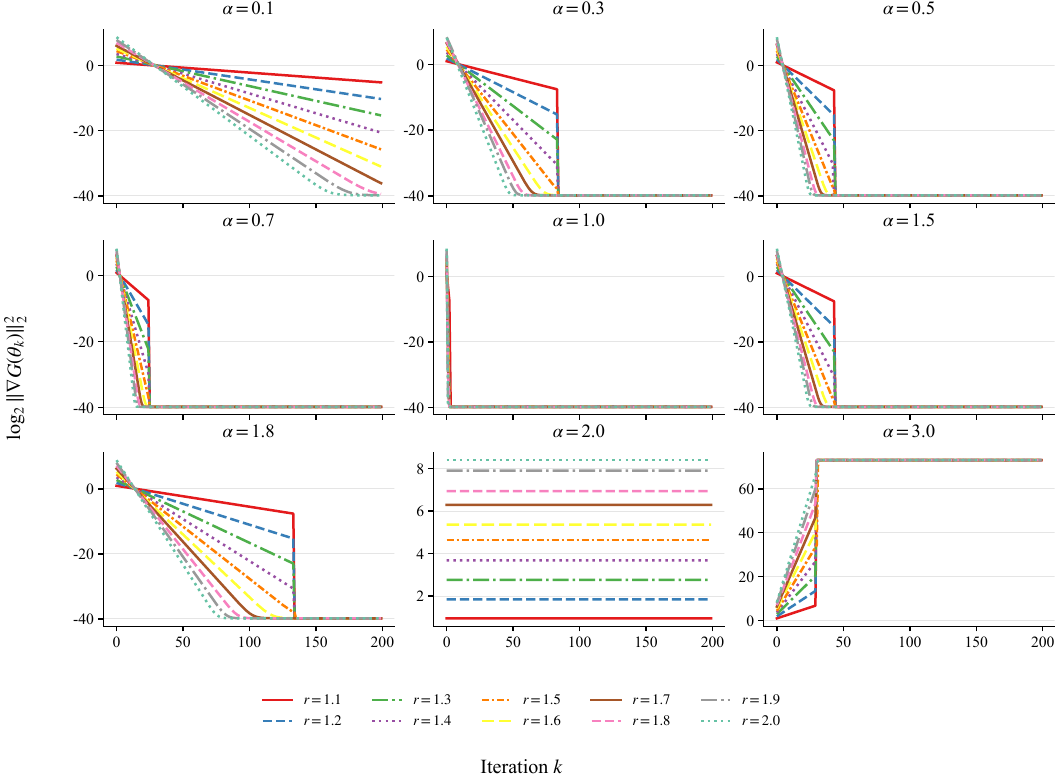}
\caption{Gradient energy \(\log_2\|\nabla G(\theta_k)\|_2^2\) versus iteration for \(G(\theta)=\|\theta\|_2^r\) with \(r\in[1.1,2.0]\), grouped by step size \(\alpha\); the fastest decay is attained at \(\alpha=1\), and divergence occurs for \(\alpha>2\).}
\label{fig:all_ratio_vs_all_r}
\end{figure}

\begin{figure}[htbp]
\centering
\includegraphics[width=0.75\linewidth]{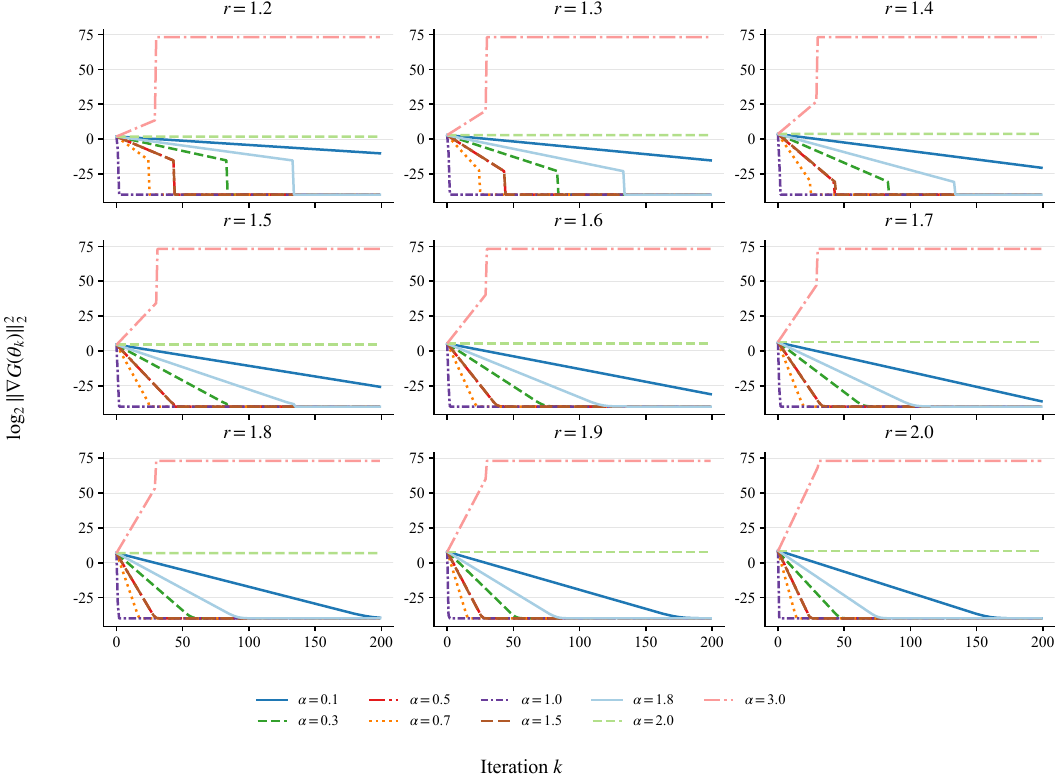}
\caption{Gradient energy \(\log_2\|\nabla G(\theta_k)\|_2^2\) versus iteration for a fixed exponent \(r\) across step sizes \(\alpha\); \(\alpha=1\) is optimal for every \(r\) tested.}
\label{fig:all_ratio_vs_r}
\end{figure}
\vspace{0.5em}
\noindent \textbf{Experimental Results and Discussion}
\vspace{0.5em}

The convergence curves in Figs.~\ref{fig:all_ratio_vs_all_r} and~\ref{fig:all_ratio_vs_r} show consistent empirical patterns across the tested non-Lipschitz-smooth functions. The best convergence performance was consistently attained at \(\alpha=1\); convergence degraded monotonically as \(\alpha\) deviated from \(1\), optimization stagnated at \(\alpha=2\), and divergence occurred for \(\alpha>2\). These trends held uniformly over all examined exponents \(r \in [1.1,2.0]\), providing direct empirical evidence for the optimal learning-rate condition stated in part~\ref{thm:general_gd_convergence1} of Theorem~\ref{thm:general_gd_convergence}.

\subsection{Verification of Optimization Mechanism}

This subsection empirically validates the optimization mechanism in Theorem~\ref{thm:com_general_bound} and its corollaries, Corollaries~\ref{cor:erm_bound} and~\ref{cor:kl_mse_bound}: we verify that SGD reduces gradient energy and that the empirical-risk bounds are governed jointly by gradient energy and the extreme eigenvalues of the structure matrix. 
    
\subsubsection{Experimental Designs}

To verify whether DNN training dynamics align with theoretical predictions, hyperparameters were chosen only to ensure convergence, without fine-tuning to improve convergence. We recorded three quantities during training, using simplified notation in the figures:
\begin{itemize}
 \item \(\log \mathcal{L}(\theta,z)\): logarithmic risk, measuring the suboptimality of the current model on sample \(z\);
 \item \(\text{bound}_{\text{up}} = -\log \lambda_{\min}(A_x) + \log \|\nabla_\theta \mathcal{L}(\theta,z)\|_2^2\): core components of the risk upper bound;
 \item \(\text{bound}_{\text{low}} = -\log \lambda_{\max}(A_x) + \log \|\nabla_\theta \mathcal{L}(\theta,z)\|_2^2\): core components of the risk lower bound.
\end{itemize}

By Corollary~\ref{cor:erm_bound}, \(\text{bound}_{\text{up}}\) and \(\text{bound}_{\text{low}}\) are not strict bounds on \(\log \mathcal{L}(\theta,z)\) but core components determining its bounds. Since \(\Phi(\cdot)\) and \(\phi(\cdot)\) are generally intractable but increasing in the norm of their arguments, we instead monitor the correlations of \(\text{bound}_{\text{up}}\) and \(\text{bound}_{\text{low}}\) with \(\log \mathcal{L}(\theta,z)\); a high correlation indicates that the risk is governed by its bounds, in line with the framework's prediction. For MSE and Softmax CrossEntropy, Corollary~\ref{cor:kl_mse_bound} shows that adding specific constants makes \(\text{bound}_{\text{up}}\) and \(\text{bound}_{\text{low}}\) strict upper and lower bounds, so we additionally check whether these bounds hold.

To verify that gradient energy and induced Jacobian norms jointly govern empirical-risk evolution, we use the local Pearson correlation coefficient. For window length \(w\), the local Pearson correlation between \(X\) and \(Y\) is
\begin{equation}
 r_w = \frac{\sum_{i=1}^w (X_i - \bar{X})(Y_i - \bar{Y})}{\sqrt{\sum_{i=1}^w (X_i - \bar{X})^2} \sqrt{\sum_{i=1}^w (Y_i - \bar{Y})^2}},
\end{equation}
where \(\bar{X} = \frac{1}{w} \sum_{i=1}^w X_i\) and \(\bar{Y} = \frac{1}{w} \sum_{i=1}^w Y_i\). It ranges over \([-1,1]\), with \(1\), \(-1\), and \(0\) indicating perfect positive correlation, perfect negative correlation, and no linear correlation, respectively. We use a sliding window of length \(50\).

\textbf{Datasets.} To test modality independence, experiments cover visual recognition and text classification: MNIST~\citep{LeCun1998GradientbasedLA}, Fashion-MNIST~\citep{DBLP:journals/corr/abs-1708-07747}, CIFAR-10, and CIFAR-100~\citep{krizhevsky2009learning} for images; TREC~\citep{LiR02} and SST-2~\citep{SocherPWCMNP13} for text.
\begin{figure}[htbp]
\centering
\centerline{\includegraphics[width=0.9\linewidth]{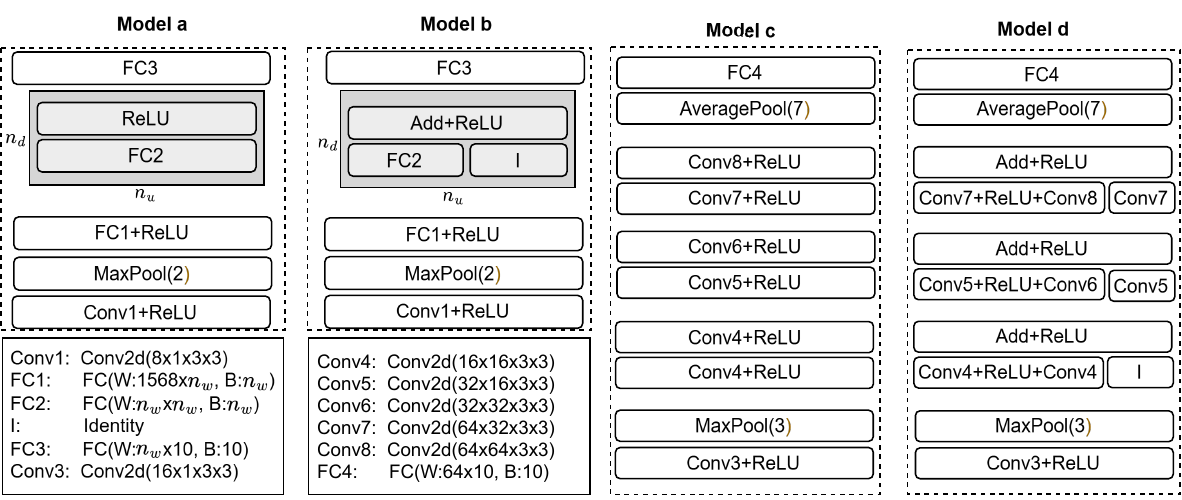}}
\caption{Model architectures and configuration parameters. Models b and d extend a and c by adding skip connections; \(n_d\) and \(n_w\) denote depth- and width-repeat counts.}
\label{fig:custom_model_structure}
\end{figure}
\textbf{Model Architectures.} To test architecture independence, we use diverse deep networks. For images, we employ custom models and classical architectures. The custom models (Fig.~\ref{fig:custom_model_structure}) consist of repeated modules: \(n_d\) modules along depth and \(n_w\) along width, with the width repeated in steps of 16. Optimization experiments use \(n_d=1,n_w=20\); the experiments on the structure matrix vary both. They are evaluated on MNIST.

For real-world applicability, we use LeNet~\citep{5265772}, ResNet-18~\citep{He2015DeepRL}, and ViT~\citep{dosovitskiy2021an} on CIFAR-10/100~\citep{Krizhevsky2009LearningML}, representing standard CNNs, skip-connection architectures, and Transformers, respectively (designs follow \citet{yoshioka2024visiontransformers}). For text, we use TextCNN~\citep{Kim14}, a residual MLP with layer normalization~\citep{ba2016layernormalization}, and a lightweight Transformer~\citep{Vaswani2017AttentionIA}, following standard open-source configurations~\citep{chinese-text-classification}.

We conduct the following experiments:
\begin{itemize}
 \item \textbf{Architecture Independence.} We examine training dynamics of custom models and LeNet, ResNet-18, and ViT to verify that the non-convex optimization mechanism is architecture-independent.
 \item \textbf{Non-Triviality of Bounds.} We perform a false-label sanity check: we train with true labels and then evaluate the risk and its bounds using randomly assigned incorrect labels. The parameters and the Jacobian remain identical, but the gradient energy becomes excessively large, making the upper bound too loose. Comparing true- and false-label correlations verifies that the bounds are informative only when SGD sufficiently reduces gradient energy.
 \item \textbf{Loss-Function Independence.} Beyond Softmax CrossEntropy, we test PyTorch MSE, KLDiv, PoissonNLL, and a custom loss \(\|f_{\theta}(x)-y\|_k^k\) for \(k=2,\dots,6\).
 \item \textbf{Data-Modality Independence.} We use TREC~\citep{LiR02} and SST-2~\citep{SocherPWCMNP13} with TextCNN, residual MLP with layer normalization, and lightweight Transformer, following standard configurations~\citep{chinese-text-classification}.
 \item \textbf{Model-Scale Independence.} On CIFAR-10, we use ResNet-18/34/50/101/152 with BatchNorm, trained with Adam (lr \(0.001\)) for \(80\) epochs.
\end{itemize}

\subsubsection{Experimental Results}

Figures~\ref{fig:custom_convergence}--\ref{fig:nlp_convergence} summarize the training dynamics for the settings reported below.
\begin{figure}[htbp]
\centering
\includegraphics[width=0.9\linewidth]{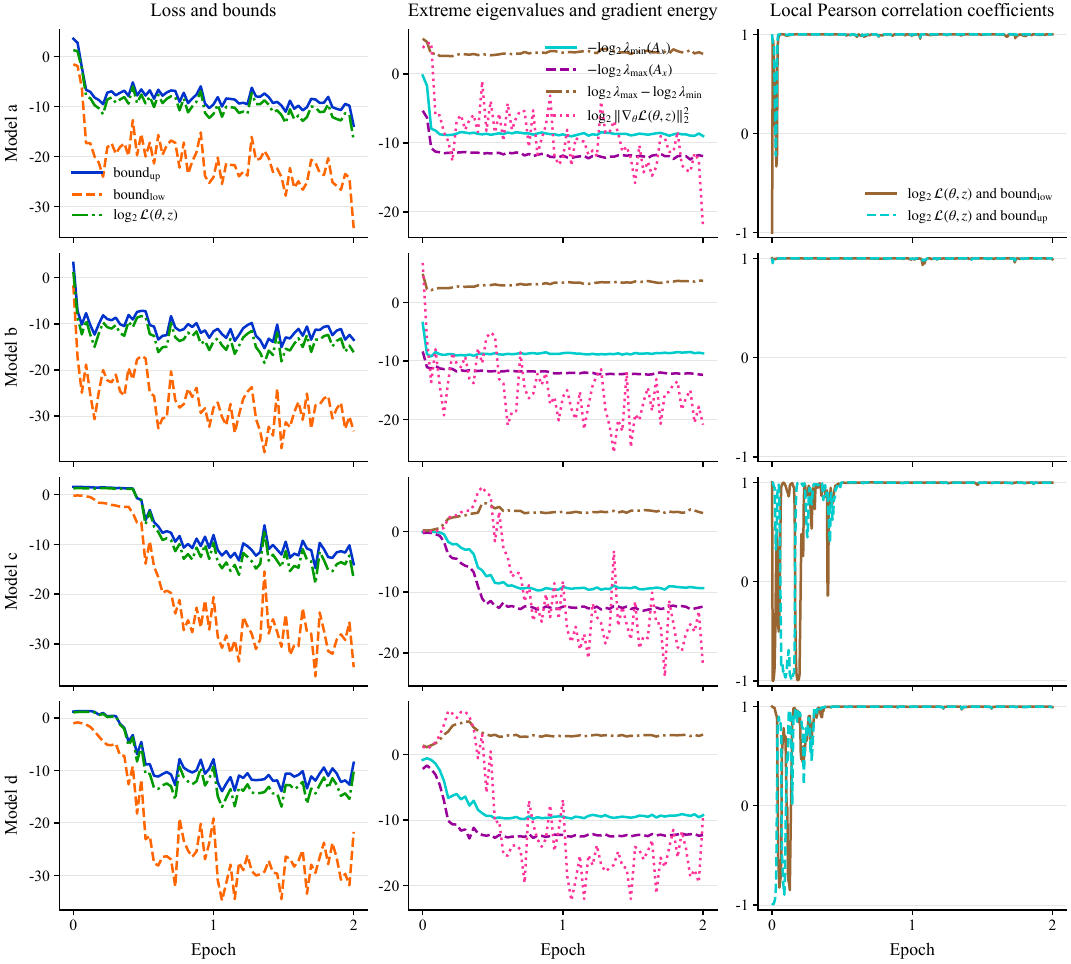}
\caption{Training dynamics of the custom models on MNIST. Left: risk and its upper/lower bounds. Middle: risk and the components constituting the bounds. Right: local Pearson correlation between the risk and its bounds.}
\label{fig:custom_convergence}
\end{figure}

\begin{figure}[htbp]
\centering
\includegraphics[width=0.9\linewidth]{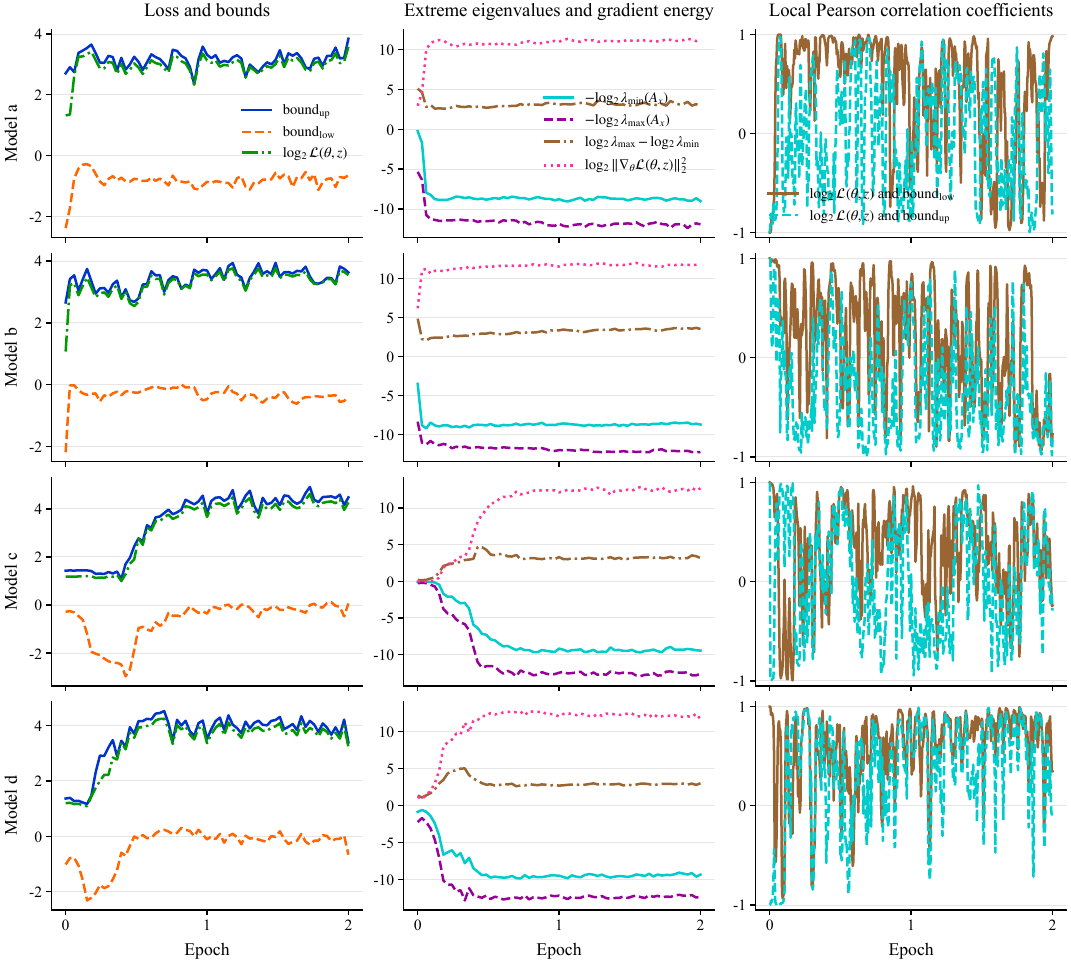}
\caption{Training dynamics of the custom models under the false-label setting. Left: risk and its upper/lower bounds. Middle: risk and the components constituting the bounds. Right: local Pearson correlation between the risk and its bounds.}
\label{fig:false_custom_convergence}
\end{figure}
 
\begin{figure}[htbp]
\centering
\includegraphics[width=0.9\linewidth]{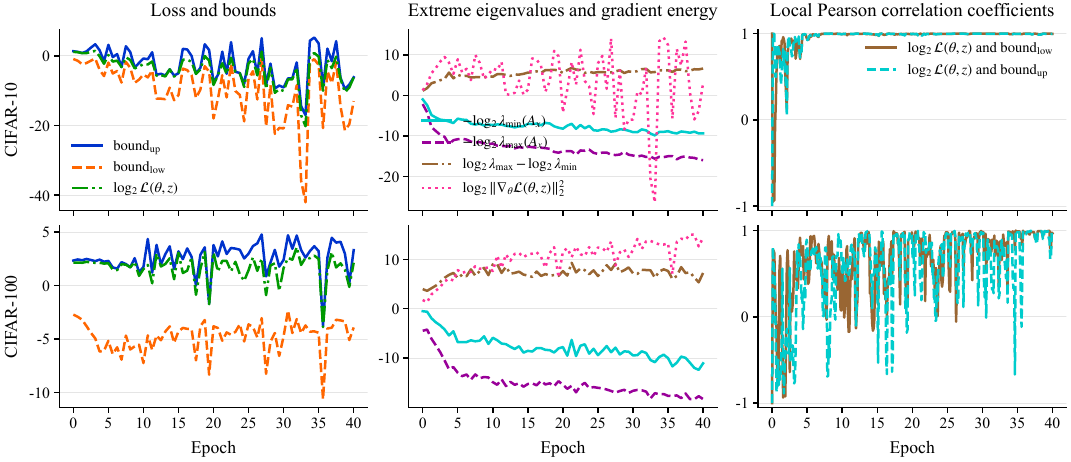}
\caption{Training dynamics of LeNet on CIFAR-10 and CIFAR-100. Left: risk and its upper/lower bounds. Middle: risk and the components constituting the bounds. Right: local Pearson correlation between the risk and its bounds.}
\label{fig:lenet_info_ch5}
\end{figure}

\begin{figure}[htbp]
\centering
\includegraphics[width=0.9\linewidth]{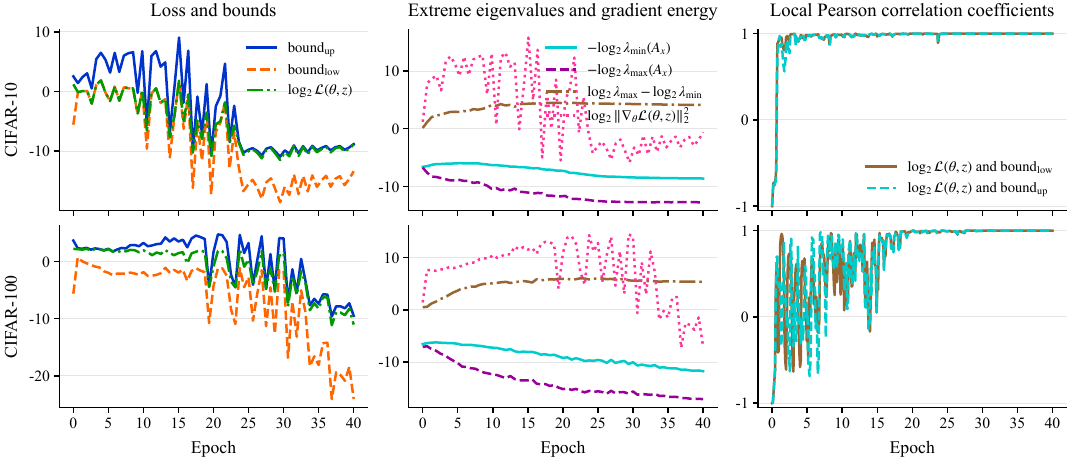}
\caption{Training dynamics of ResNet-18 on CIFAR-10 and CIFAR-100. Left: risk and its upper/lower bounds. Middle: risk and the components constituting the bounds. Right: local Pearson correlation between the risk and its bounds.}
\label{fig:resnet18_info_ch5}
\end{figure}

\begin{figure}[htbp]
\centering
\includegraphics[width=0.9\linewidth]{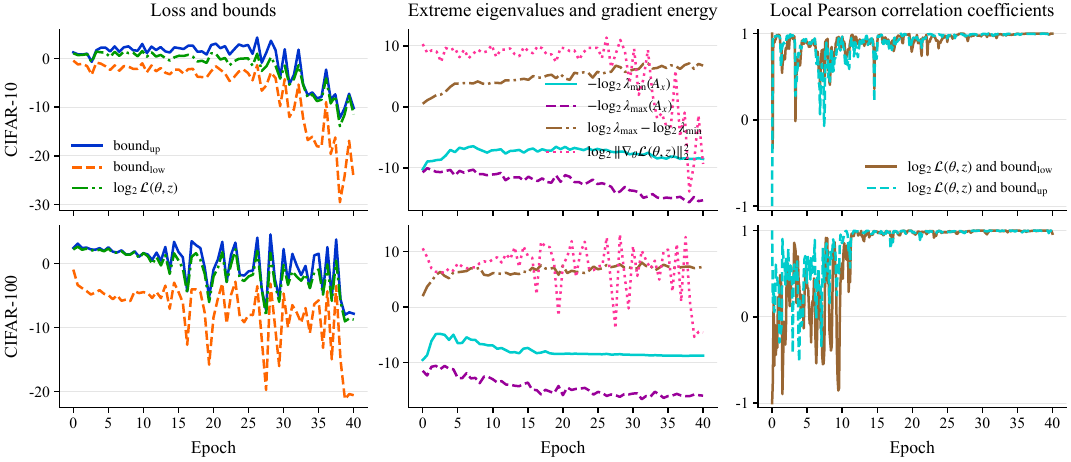}
\caption{Training dynamics of Vision Transformer (ViT) on CIFAR-10 and CIFAR-100. Left: risk and its upper/lower bounds. Middle: risk and the components constituting the bounds. Right: local Pearson correlation between the risk and its bounds.}
\label{fig:vit_info_ch5}
\end{figure}

\begin{figure}[htbp]
\centering
\includegraphics[width=0.9\linewidth]{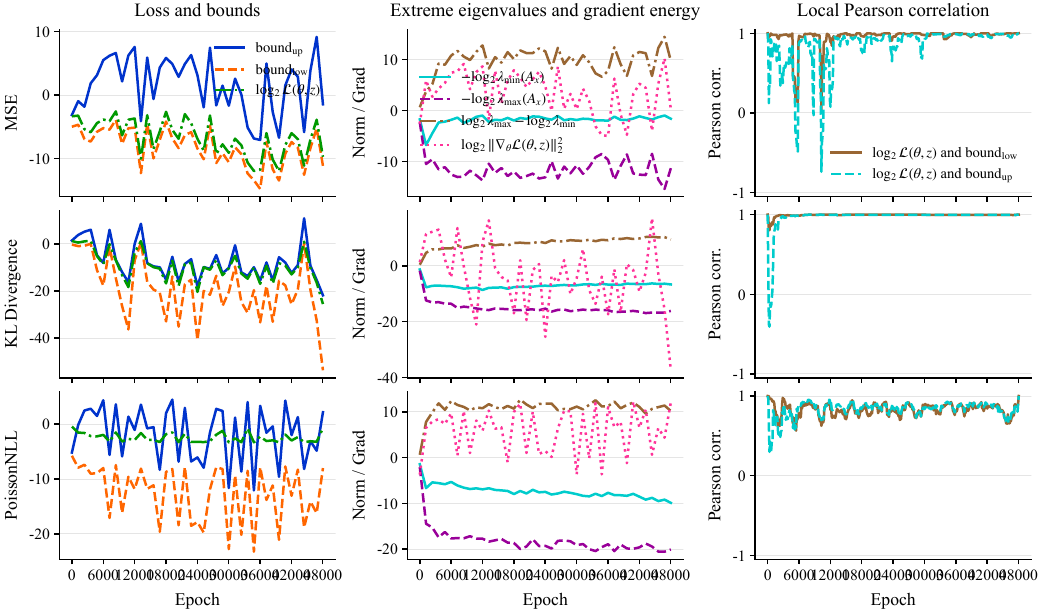}
\caption{Training dynamics of ResNet-18 on CIFAR-10 using different PyTorch built-in loss functions (MSE, KLDiv, and PoissonNLL). Left: risk and its upper/lower bounds. Middle: risk and the components constituting the bounds. Right: local Pearson correlation between the risk and its bounds.}
\label{fig:native_loss_convergence}
\end{figure}

\begin{figure}[htbp]
\centering
\includegraphics[width=0.9\linewidth]{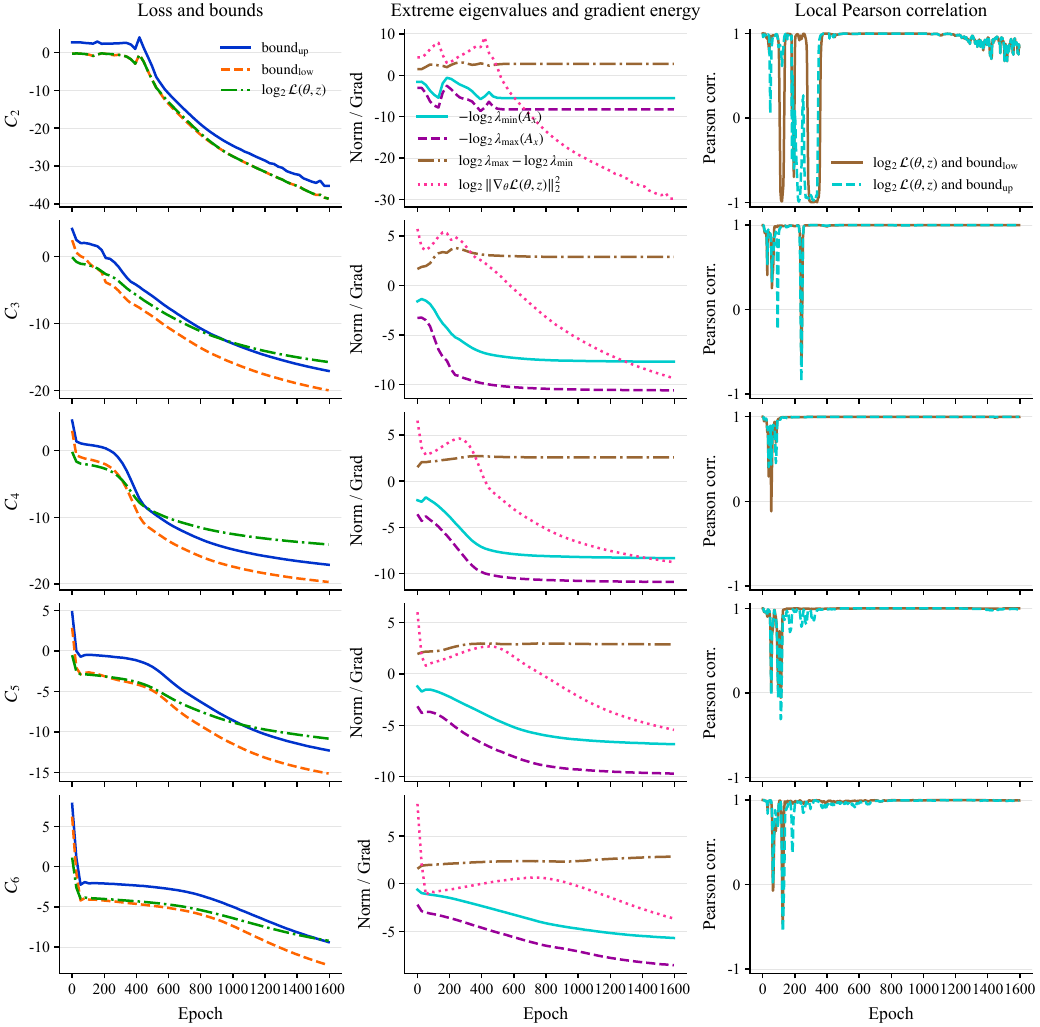}
\caption{Training dynamics of model~a on MNIST using the custom loss family \(\|f_{\theta}(x) - y\|_k^k\) (\(k \in \{2, \dots, 6\}\)). Only 10 randomly selected samples were used to ensure fast convergence. Left: risk and its upper/lower bounds. Middle: risk and the components constituting the bounds. Right: local Pearson correlation between the risk and its bounds.}
\label{fig:custom_loss_convergence}
\end{figure}

\begin{figure}[htbp]
\centering
\includegraphics[width=0.9\linewidth]{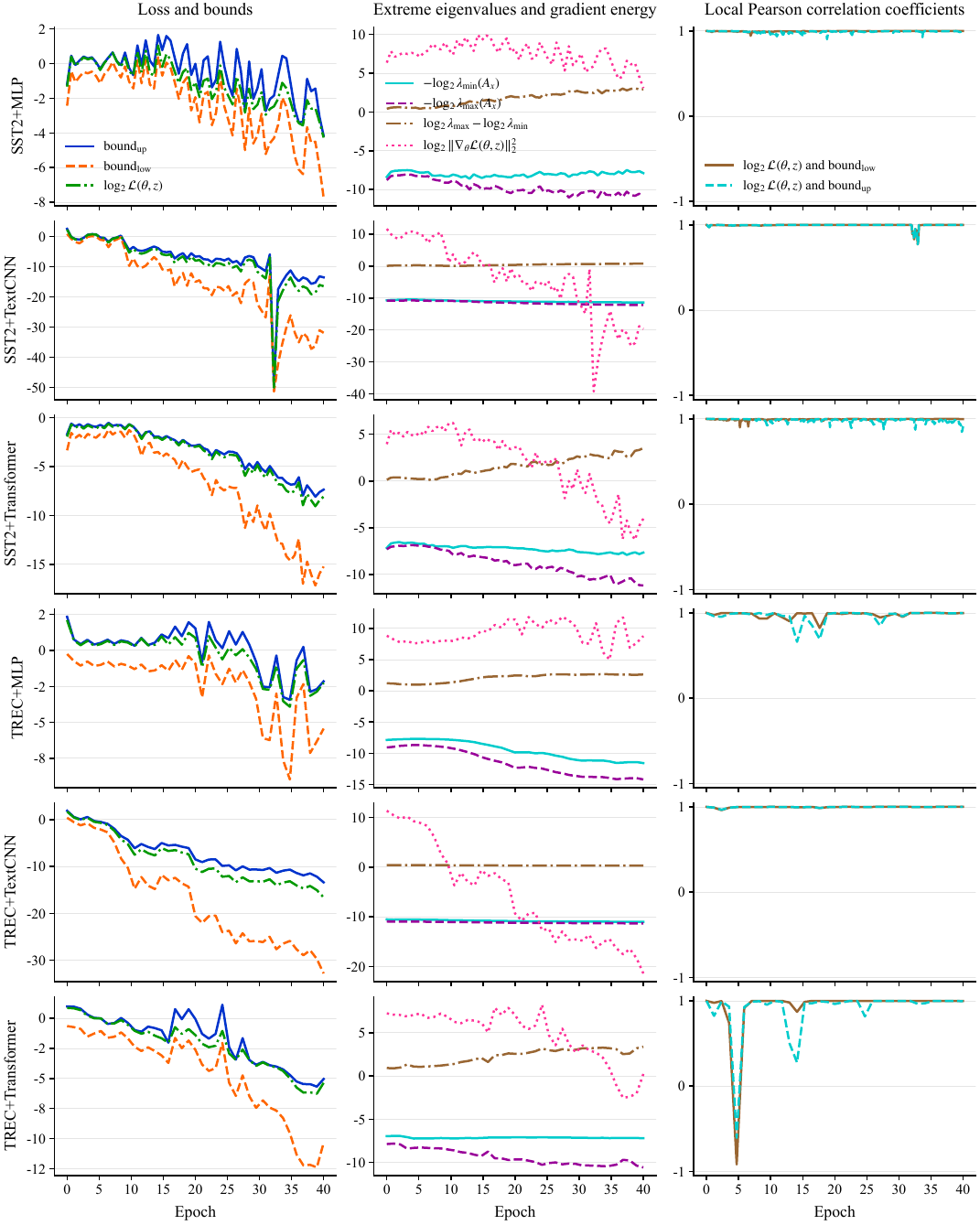}
\caption{Training dynamics on the text-classification datasets (TREC, SST-2), demonstrating the applicability of the conclusions across data modalities. The loss function is Softmax CrossEntropy and the optimizer is Adam with a learning rate of 0.001. Left: risk and its upper/lower bounds. Middle: risk and the components constituting the bounds. Right: local Pearson correlation between the risk and its bounds.}
\label{fig:nlp_convergence}
\end{figure}

The empirical results consistently support the theoretical framework, leading to the following conclusions:
\begin{enumerate}
 \item \textbf{Risk optimization is governed by its upper and lower bounds.} Under true labels (Figs.~\ref{fig:custom_convergence}, \ref{fig:lenet_info_ch5}, \ref{fig:resnet18_info_ch5}, and~\ref{fig:vit_info_ch5}), the risk evolved increasingly in line with its theoretical bounds; the gap narrowed and both converged toward zero, with local Pearson correlations approaching \(1\). This indicates that the bounds characterize risk dynamics tightly.
 \item \textbf{The bounds are non-trivial.} Under false-label evaluation (Fig.~\ref{fig:false_custom_convergence}), the algebraic relation still held, but the gradient energy became excessively large and the Jacobian spectrum was severely mismatched, so the bounds no longer constrained the risk. Thus, the bounds in Theorem~\ref{thm:com_general_bound} are universally valid yet reveal a training mechanism specific to the true-label setting.  
 \item \textbf{SGD reduces gradient energy.} Across loss functions and data modalities (Figs.~\ref{fig:native_loss_convergence}, \ref{fig:custom_loss_convergence}, and~\ref{fig:nlp_convergence}), the evolution of the extreme eigenvalues of the structure matrix and of the gradient energy is consistent with Theorem~\ref{thm:general_sgd_convergence}: mini-batch SGD progressively reduces the gradient energy, enabling effective non-convex optimization.
\end{enumerate}

\subsection{Verification of the Impact of Extreme-Eigenvalue Control}

We next examine whether width, depth, and skip connections behave as expected: under the approximate independence assumption of Proposition~\ref{prop:number_para}, increasing the number of parameters narrows the distribution of the eigenvalues of the structure matrix, whereas skip connections suppress the depth-induced decay of the extreme eigenvalues.

\subsubsection{Experimental Designs}

Using models~a--d (Fig.~\ref{fig:custom_model_structure}), we conduct two controlled experiments on MNIST and Fashion-MNIST~\citep{DBLP:journals/corr/abs-1708-07747} to study how architecture affects the spectral properties of the structure matrix \(A_x\):
\begin{enumerate}
 \item Fixing \(n_w=64\), we increase \(n_d\) from \(1\) to \(100\) to examine the effect of depth;
 \item Fixing \(n_d=1\), we increase \(n_w\) from \(16\) to \(1600\) to examine the effect of width.
\end{enumerate}
At initialization, we record \(\log_2\lambda_{\max}(A_x)\), \(\log_2\lambda_{\min}(A_x)\), and their gap as functions of the architectural parameters.

\subsubsection{Experimental Results}
\begin{figure}[htbp]
	\centering
	\includegraphics[width=0.9\linewidth]{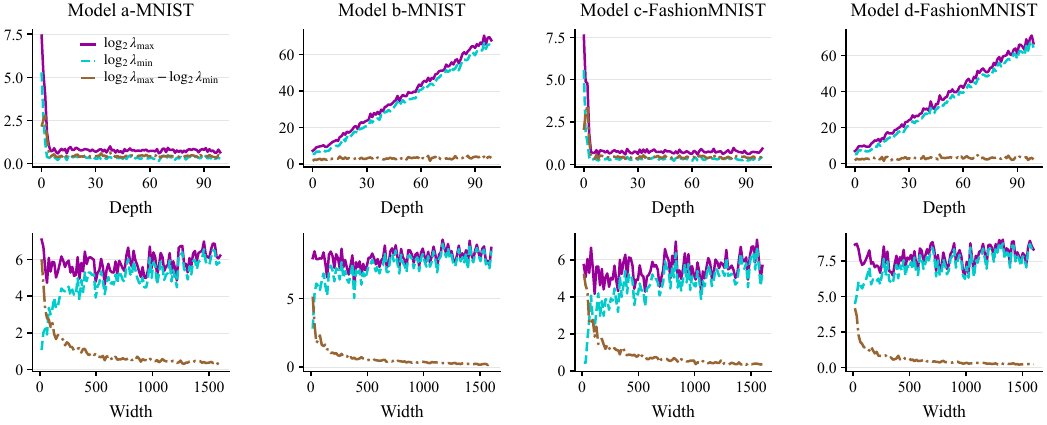}
	\caption{Variation of the extreme eigenvalues \(\lambda_{\max}(A_x)\) and \(\lambda_{\min}(A_x)\) with model width and depth under random initialization. For the width sweep, \(n_w\) is increased from \(16\) to \(1600\) with \(n_d=1\); for the depth sweep, \(n_d\) is increased from \(1\) to \(100\) with \(n_w=64\).}
	\label{fig:depth_width_init}
\end{figure}

\paragraph{Synergistic Effect of Random Parameter Initialization}
Figure~\ref{fig:depth_width_init} shows that the spectral gap \(\log_2 \lambda_{\max}(A_x) - \log_2\lambda_{\min}(A_x)\) decreases with model width or depth. Since increasing either expands the number of trainable parameters, this is consistent with the prediction of Proposition~\ref{prop:number_para} that more parameters reduce the ratio between the extreme eigenvalues of the structure matrix.

\paragraph{Network Width and Extreme Eigenvalues}
Figure~\ref{fig:depth_width_init} shows that as width increases, the extreme eigenvalues of models~a--d oscillate upward, gradually stabilize, and approach each other. If the parameters change little during training (i.e., remain near initialization), these eigenvalues stay approximately constant for wider models, so the empirical-risk bounds depend only on the gradient energy, which mini-batch SGD can effectively optimize. This aligns with NTK's ``lazy training'' but does not require the infinite-width assumption, a main limitation of NTK approaches. 

\paragraph{Network Depth and Skip Connections}
Figure~\ref{fig:depth_width_init} shows that with increasing depth, models~a and~c (without skip connections) exhibit rapid decay of the extreme eigenvalues, whereas models~b and~d (with skip connections) exhibit linear growth of their logarithms. This is consistent with the analysis of Proposition~\ref{prop:skip_con}: the overall Jacobian approximates a product of local module Jacobians. Under standard initialization (zero mean, variance \(\sigma^2<1\)), this product causes exponential decay with depth, while skip connections add an identity term and thus yield exponential growth. Since the empirical-risk upper bound in Corollary~\ref{cor:erm_bound} is governed by the inverse of the smallest extreme eigenvalue, keeping the eigenvalues away from zero is favorable. Figure~\ref{fig:depth_width_init} indicates that skip connections prevent depth-induced decay and keep the risk bounds informative.

\paragraph{Scope and limitations}
Three caveats delimit the claims above. First, the experiments compare curve-level quantities---risk, gradient energy, and the eigenvalues of the structure matrix---rather than the absolute values of the bounds, because the energy functions \(\Phi\) and \(\phi\) are intractable for these models; the reported evidence is therefore correlational and supports the bounds through the predicted functional relationships. Second, the bounds require \(\lambda_{\min}(A_z)>0\) and are uninformative when the Jacobian is rank deficient. Third, the analysis concerns optimization only and makes no claim about generalization.
 
\section{Conclusion}
\label{sec:conclusion}

This work develops a conjugate-duality framework for non-convex and non-smooth optimization. By replacing the quadratic energy \(\|\cdot\|_2^2\) underlying classical strong convexity and Lipschitz smoothness with a broad family of Legendre energy functions, we introduce \(\mathcal{H}(\psi)\)-convexity and \(\mathcal{H}(\Psi)\)-smoothness. These generalized notions unify convex and non-convex, as well as smooth and non-smooth, objectives within a single formalism, and reveal a natural duality between convexity and smoothness.

Using this framework, we address a central puzzle in deep learning: why SGD reliably succeeds on highly non-convex, non-smooth objectives for which classical optimization theory offers no guarantees. Specifically, for the composite formulation that captures DNN training, optimization admits a transparent two-term characterization: SGD reduces the gradient energy, while the extreme eigenvalues of the structure matrix are governed by the architecture and the initialization. Experiments indicate that the resulting upper and lower bounds track the empirical risk across a range of models, datasets, and loss functions.

Several directions remain open for future work. It would be valuable to connect the gradient-energy perspective with generalization and implicit-regularization theory, and to develop optimization strategies that directly control the extreme eigenvalues of the structure matrix. We hope that the conjugate-duality viewpoint developed here provides a reusable foundation for such analyses.

\section*{Acknowledgments}
This research did not receive any specific grant from funding agencies in the public, commercial, or not-for-profit sectors.

\section*{Declaration of competing interest}
The author declares that he has no known competing financial interests or personal relationships that could have appeared to influence the work reported in this paper. 

\section*{Data availability}
The datasets used in this study are publicly available standard benchmarks (MNIST, Fashion-MNIST, CIFAR-10, CIFAR-100, TREC, and SST-2). 

\bibliographystyle{cas-model2-names}
\bibliography{refs}

\printcredits

\appendix

\allowdisplaybreaks[2]


\section{Supporting Lemmas}
\label{sec:lemmas}
This section collects the auxiliary results used in the proofs of this paper.

\subsection{Convex Analysis}

\begin{lemma}[Properties of Convex Conjugate Duality]
\label{lem:fenchel_duality}
For all \(\mu,\nu\in \mathbb{R}^k\), the following hold:
\begin{enumerate}
 \item \textbf{Fenchel--Young inequality}\label{lem:fy_inequality}:
\begin{equation}
 \Psi(\mu) + \Psi^*(\nu) \ge \langle \mu,\nu\rangle, \quad \forall \mu \in \mathrm{dom}(\Psi),\, \nu \in \mathrm{dom}(\Psi^*).
 \label{eq:fenchel_young_inequality}
 \end{equation}
 Equality holds if and only if \(\nu \in \partial \Psi(\mu)\). If, furthermore, \(\Psi\) is strictly convex and differentiable, then equality holds if and only if \(\nu = \mu_\Psi^*\).
 \item \label{lem:mul_factor_outer}Suppose \(\Phi(x)\) is a convex function and \(a > 0\). Letting \(\Psi(x)=a\Phi(x)\), we have \(\Psi^*(f)=a\Phi^*(f/a)\).
\item \label{lem:mul_factor}Suppose \(\Phi(x)\) is a convex function and \(a > 0\). Letting \(\Psi(x)=\Phi(ax)\), we obtain \(\Psi^*(f)=\Phi^*(f/a)\).
 \item \textbf{Conjugate of radial functions}\label{lem:radial_fun}. Let \(\|\cdot\|\) be a norm on \(\mathbb{R}^d\) and \(\|\cdot\|_*\) its dual norm. If \(\Psi(x) = \phi(\|x\|)\), where \(\phi : \mathbb{R} \to (-\infty, +\infty]\) is an \textbf{even}, lower semicontinuous, proper convex function, then \(\Psi^*(y) = \phi^*(\|y\|_*)\). 
\item \label{lem:order_reverse}If \(f(x) \le g(x)\) pointwise, then \(f^*(y) \ge g^*(y)\) pointwise.
\end{enumerate}
\end{lemma}
For detailed expositions of these properties, see~\citet{1970Convex} and~\citet{Todd2003ConvexAA}.

\subsection{Legendre Functions}
\begin{definition}[Legendre Function]
A proper, lower semicontinuous convex function \(\Psi:\mathbb{R}^m\to \bar{\mathbb R}\) is called a Legendre function if it satisfies the following conditions:
 \begin{enumerate}
 \item \textbf{Essentially smooth:} The interior of the domain is nonempty, i.e., \(\mathrm{int}(\mathrm{dom}(\Psi))\neq \emptyset\), and \(\Psi\) is differentiable on \(\mathrm{int}(\mathrm{dom}(\Psi))\) such that \(\|\nabla\Psi(w^{\nu})\|\to \infty\) whenever \(w^{\nu}\to w\in \text{bdry }\mathrm{dom}(\Psi)\);
 \item \textbf{Essentially strictly convex:} \(\Psi\) is strictly convex on \(\mathrm{dom}(\partial \Psi):=\{w\in \mathbb{R}^m:\partial\Psi(w)\neq \emptyset\}\);
 \item \textbf{Legendre:} \(\Psi\) is both essentially smooth and essentially strictly convex~\citep{1970Convex}.
 \end{enumerate}
\end{definition}
The essential strict convexity of a Legendre function is transformed into essential smoothness under duality. The following proposition states that the conjugate of a Legendre function is also a Legendre function~\citep{Strmberg2009ANO,1970Convex}.
\begin{proposition}\label{prop:legendre_dual}
If \(F\) is a Legendre function, then its convex conjugate \(F^*\) is also a Legendre function.
\end{proposition}

\begin{lemma}
\label{lem:convex_smooth_dual}
Let \(f_*\) be the global minimum of the differentiable function \(f(x)\).
If \(f\) is strongly convex with parameter \(c\), it follows that
\begin{equation}\label{eq:strong}
\begin{aligned}
 f(y) &\geq f(x) + \nabla f(x)^\top(y-x) + \frac{c}{2}\|y-x\|^2, \quad \forall x, y,\\
 f(x) &- f_* \leq \dfrac1{2c}\|\nabla f(x)\|_2^2.
\end{aligned}
\end{equation}
If \(f\) is Lipschitz smooth over \(\mathbb{R}^n\) with Lipschitz constant \(L\), then it follows that
\begin{equation}\label{eq:lipschitz}
\begin{aligned}
 f(y) &\leq f(x) + \nabla f(x)^\top(y-x) + \frac{L}{2}\|y-x\|^2, \quad \forall x, y,\\
 f(x)&-f_*\geq \frac{1}{2L} \|\nabla f(x)\|_2^2.
\end{aligned}
\end{equation}
For the proofs of the above results, see the studies~\citep{kakade2009duality,zhou2018fenchel}.
They further reveal the duality properties between strong convexity and Lipschitz smoothness as follows.
For a function \(f\) and its Fenchel conjugate \(f^*\), the following assertions hold:
\begin{itemize}
 \item If \(f\) is closed and strongly convex with parameter \(c\), then \(f^*\) is Lipschitz smooth with parameter \(1/c\).
 \item If \(f\) is convex and Lipschitz smooth with parameter \(L\), then \(f^*\) is strongly convex with parameter \(1/L\).
\end{itemize}
\end{lemma}

\subsection{Vector Norms}
\textbf{Dual Norm.} Let \(\|\cdot\|\) be a norm on \(\mathbb{R}^d\). The \textbf{dual norm} \(\|\cdot\|_*\) is defined by

\[
\|y\|_* = \sup_{\|x\| \le 1} \langle x, y \rangle = \sup_{x \neq \mathbf{0}} \frac{\langle x, y \rangle}{\|x\|}.
\]
 \(\|\cdot\|_*\) is positive definite, absolutely homogeneous, and satisfies the triangle inequality; hence it is a norm on \(\mathbb{R}^d\).
    These definitions and properties can be found in standard convex analysis textbooks~\citep{1970Convex}.

\subsection{Information-Theoretic Inequalities}

\begin{lemma}[KL Divergence Upper Bound]
\label{lem:kl_upper_bound}
 If \(p\) and \(q\) are probability densities/masses both supported on a bounded interval \(I\), then we have~\citep{1603768}
 \begin{equation}
 D_{\mathrm{KL}}(p,q) \leq \frac{1}{\inf_{x\in I} q(x)} \|p-q\|_2^2.
 \end{equation}
\end{lemma}

\begin{lemma}[Pinsker's Inequality~\citep{2017Elements}]
\label{lem:pinsker}
If \(p\) and \(q\) are probability densities/masses both supported on a bounded interval \(I\), then we have
\begin{equation}
 D_{\mathrm{KL}}(p\|q) \ge \frac{1}{2\ln 2} \|p-q\|_1^2.
\end{equation}
\end{lemma}

\section{Related Work}
\label{sec:model:related}

This section reviews convexity and smoothness assumptions and the gradient algorithms built on them, and motivates our unified conjugate framework by identifying gaps in explaining deep network training. 
 
\subsection{Evolution of Strong Convexity and Smoothness Assumptions}
\label{subsec:model:related:convex}

Classical convex optimization relies on strong convexity and global Lipschitz smoothness \citep{nesterov2013introductory}, both violated by deep networks, prompting generalized formulations.

\subsubsection{Extensions of Strong Convexity}
Higher-order strong convexity~\citep{Lin2003SomeEP} allows curvature to vary with direction and scale. For a closed convex set \(K\), \(F:K\to\mathbb{R}\) is higher-order strongly convex if \(\exists \mu\ge0, p\ge1\) such that for all \(u,v\in K, t\in[0,1]\),
\begin{equation}
 \label{eq:higher_order_convex}
 F(u + t(v - u)) \leq (1 - t)F(u) + tF(v) - \mu \varphi(t) \|v - u\|^p,
\end{equation}
where \(\varphi(t)=t(1-t)\); \(p=2\) recovers standard strong convexity. Subsequent variants \citep{noor2020higher,Mohsen2019StronglyCF,Alabdali2019CharacterizationsOU} model a spectrum from weak to super convexity via \(p\) and \(\varphi(t)\).

\subsubsection{Extensions of Lipschitz Smoothness}
H\"older smoothness~\citep{Grimmer2021GeneralHS} is defined by \(\|g-g'\|\le L\|x-x'\|^\eta\) for all \(x, x' \in Q\) and all subgradients \(g \in \partial f(x)\), \(g' \in \partial f(x')\), where \(L \geq 0\), \(\eta \in [0,1]\). 
\((L_0,L_1)\)-smoothness~\citep{Zhang2019WhyGC} matches larger gradients with larger curvature,
\begin{equation}
\label{eq:l0l1_smooth}
\|\nabla^2 f(x)\| \leq L_0 + L_1 \|\nabla f(x)\|,
\end{equation}
fitting losses like CrossEntropy, and has been widely applied~\citep{crawshaw2022robustness,Reisizadeh2023VariancereducedCF,gorbunov2024methodsconvexl0l1smoothoptimization}. 
\(\ell\)-smoothness~\citep{Li2023ConvexAN} generalizes this to
\begin{equation}
 \label{eq:ell_smooth}
 \|\nabla^2 f(x)\| \leq \ell(\|\nabla f(x)\|),
\end{equation}
with non-decreasing continuous \(\ell\), allowing arbitrary monotonic Hessian-gradient relationships. Recently, \citet{qi2025extendedconvexitysmoothnessapplications} exploited strong convexity--Lipschitz smoothness duality to propose an extended framework and establish DNN trainability for convex differentiable losses.

However, existing generalizations still require convexity and differentiability for global SGD convergence guarantees, which deep network objectives violate, and treat convexity and smoothness independently.  

\subsection{Gradient-Based Optimization Algorithms}
\label{subsec:model:related:gdopt}

Gradient-based methods are well studied for Lipschitz-smooth objectives in convex and non-convex settings \citep{LiQTRJ23}.
In convex optimization, the goal is an \(\varepsilon\)-suboptimal point with \(f(x)-f_*\le\varepsilon\). For smooth convex functions, GD attains gradient complexity \(\mathcal{O}(1/\varepsilon)\); under \(\mu\)-strong convexity, GD improves to \(\mathcal{O}(\kappa\log(1/\varepsilon))\), which is optimal. In non-convex optimization, where global minimization is NP-hard, the focus is on \(\varepsilon\)-stationary points with \(\|\nabla f(x)\|\le\varepsilon\). Deterministic smooth non-convex GD achieves the optimal \(\mathcal{O}(1/\varepsilon^2)\) \citep{Carmon2017LowerBF}; with unbiased bounded-variance gradients, SGD attains sample complexity \(\mathcal{O}(1/\varepsilon^4)\) \citep{Ghadimi2013StochasticFA}, also matching lower bounds \citep{Arjevani2019LowerBF,abbaszadehpeivasti2023conditions}.

Algorithms have also evolved with generalized smoothness. For \((L_0,L_1)\)-smooth functions, clipped and normalized GD preserve the optimal \(\mathcal{O}(1/\varepsilon^2)\) deterministic complexity, and clipped SGD achieves \(\mathcal{O}(1/\varepsilon^4)\) stochastic complexity \citep{Zhang2019WhyGC}; momentum-accelerated variants follow the same rates \citep{Zhang2020ImprovedAO}, and SPIDER-based variance reduction improves stochastic complexity to \(\mathcal{O}(1/\varepsilon^3)\) \citep{Reisizadeh2023VariancereducedCF}. For \(\ell\)-smooth objectives, GD and SGD retain the classical Lipschitz-smooth complexities \citep{Li2023ConvexAN}.

Despite these advances, existing theories only guarantee convergence to stationary points and cannot explain the empirically observed implicit preference for globally optimal or well-generalizing solutions in deep learning. A framework that fully characterizes DNN optimization dynamics therefore remains an open problem.

\subsection{Non-Convex Optimization Mechanisms in Deep Learning}
\label{subsec:alg:intro:gd}

SGD's effectiveness stems from two synergistic mechanisms: stochastic noise, which enables saddle-point escape, and over-parameterization, which reshapes loss-landscape geometry. 

\subsubsection{Stochastic Noise}
SGD noise drives dynamics along negative curvature to escape saddle points \citep{Ge2015EscapingFS}; large-step SGD with momentum or adaptive variants bypasses poor local minima \citep{li2023restarted}, with analogous effects for compressed SGD in federated learning \citep{ChenLC24}. In contrast, deterministic GD lacks such perturbations and tends to stagnate at suboptimal stationary points \citep{Jentzen2018StrongEA}. However, large-batch training (5000--6000) can still fit benchmarks \citep{KeskarMNST17,abs-1712-07628}, and near-full-batch training in over-parameterized networks reaches zero error via interpolation \citep{LeeHA23}, suggesting over-parameterization may be more critical than noise alone \citep{do2024revisitinglarslargebatch}.

\subsubsection{Role of Over-Parameterization}

Over-parameterization improves trainability and generalization by reshaping the landscape. Two infinite-width theories explain this. NTK theory \citep{Jacot2018NeuralTK} shows lazy training with a frozen kernel, reducing non-convex optimization to convex kernel learning, but the frozen-kernel assumption contradicts finite-width feature learning and empirical NTK evolution \citep{Vyas2023EmpiricalLO}. Mean Field Theory (MFT) \citep{Sirignano2018MeanFA} models parameters as a particle system and SGD as gradient flow; convexity of the loss functional ensures global convergence and extends to deep networks \citep{nguyen2023rigorous}, yet it remains limited to infinite width, and its extension to architectures with residual connections and attention remains limited. Recent work on mild over-parameterization shows ReLU networks contain high-dimensional global minimum sets, with full-rank Jacobian regions expanding with parameter redundancy \citep{KarhadkarMTM24,Kohler26,duzgun2024doesoverparameterizationaffectfeatures}. NTK and MFT still lack a unified quantitative characterization of the SGD--over-parameterization synergy, especially for finite-width feature learning \citep{Oneto2023DoWR}.

Recently, \citet{QiGL25} interpreted classification as conditional distribution estimation and linked the squared parameter gradient energy to the MSE between predicted and empirical conditional distributions via a structure matrix, suggesting that DNN training implicitly minimizes this norm while preserving the non-degeneracy of the structure matrix. Two limitations remain:
\begin{enumerate}
 \item No monotonic decrease guarantee for the Fenchel--Young loss itself, leaving a gap between the surrogate and the true objective.
 \item The analysis is restricted to supervised classification and does not extend to regression, generative modeling, or unsupervised learning.

\end{enumerate}

\section{Deferred proofs}

\label{appendix:proof}

\subsection{Proof of Proposition~\ref{prop:energy_fun}}
\label{appendix:proof_min_energy}
\noindent \textbf{Proof of Conclusion 1.}

According to Proposition~\ref{prop:legendre_dual},
if \(\Psi\) is a Legendre function, then \(\Psi^*\) is also a Legendre function.
By the definition of convex conjugate:

\[
\Psi^*(\mathbf{0}) = \sup_{x} \{ -\Psi(x) \} = -\inf_{x} \Psi(x) = -\Psi(\mathbf{0}).
\]

Taking \( x = \mathbf{0} \) as a test point in the supremum definition, we have:
\[
\Psi^*(y) = \sup_{x} \{ \langle y, x \rangle - \Psi(x) \} \ge \langle y, \mathbf{0} \rangle - \Psi(\mathbf{0}) = -\Psi(\mathbf{0}).
\]
Combining the above two steps, we obtain:
\[
\Psi^*(y) \ge -\Psi(\mathbf{0}) = \Psi^*(\mathbf{0}),
\]

Therefore, \( \Psi^*(\mathbf{0}) \) is the global minimum of \( \Psi^* \).
By the definition of energy functions, \(\Psi^*(\cdot)\) is also an energy function.

Since \(\Psi\) is a convex function and \(\mathbf{0}\) is a global minimum, we have \(\nabla\Psi(\mathbf{0}) = 0\). For any \(z \in \mathbb{R}^m\), by the first-order optimality condition of convex functions:
\[
\Psi(\mathbf{0}) \ge \Psi(z) + \langle \nabla\Psi(z),\, 0 - z \rangle
= \Psi(z) - \langle \nabla\Psi(z),\, z \rangle,
\]
Rearranging, we obtain
\begin{equation}\label{eq:energy_dual_pos}
\langle z_\Psi^*,\, z \rangle \ge \Psi(z) - \Psi(\mathbf{0}) \ge 0.
\end{equation}
The last inequality follows because \(\mathbf{0}\) is the global minimum.
According to Lemma~\ref{lem:radial_fun}, \(\Psi^*(\nu)=\Psi_\circ^*(\|\nu\|_*)\).
Conclusion 1 is proved.

\noindent \textbf{Proof of Conclusion 2.}
By the definition of the Legendre--Fenchel transform,
\[
\Psi^*(\nu) = \sup_{\mu\in\mathbb{R}^m} \left\{ \langle \mu, \nu \rangle - \Psi(\mu) \right\}.
\]
Applying the Cauchy--Schwarz inequality \(\langle \mu,\nu\rangle \le \|\mu\|_2\|\nu\|_2\), with equality if and only if \(\mu\) and \(\nu\) are linearly dependent, we obtain
\[
\Psi^*(\nu) \le \sup_{\mu} \left\{ \|\mu\|_2\|\nu\|_2 - \frac{a_\Psi^{r_\Psi}}{r_\Psi}\|\mu\|_2^{r_\Psi} \right\} - c_\Psi.
\]
The upper bound is tight: the supremum over the sphere \(\|\mu\|_2=t\) is attained at \(\mu = t\nu/\|\nu\|_2\). For \(\nu=0\), the supremum is attained at \(\mu=0\) and gives \(-c_\Psi\); the following calculation also covers this case by taking the limit. Assume \(\nu\neq 0\). Put \(t = \|\mu\|_2 \ge 0\). Then the right-hand side becomes
\[
\sup_{t\ge 0} \left\{ t\|\nu\|_2 - \frac{a_\Psi^{r_\Psi}}{r_\Psi} t^{r_\Psi} \right\} - c_\Psi.
\]
Define \(g(t) = t\|\nu\|_2 - \frac{a_\Psi^{r_\Psi}}{r_\Psi} t^{r_\Psi}\). Since \(r_\Psi>1\), \(g\) is strictly concave on \([0,\infty)\) and attains its unique maximum at the critical point:
\[
g'(t) = \|\nu\|_2 - a_\Psi^{r_\Psi} t^{r_\Psi-1} = 0
\quad\Longrightarrow\quad
t_* = \left( \frac{\|\nu\|_2}{a_\Psi^{r_\Psi}} \right)^{\frac{1}{r_\Psi-1}}
= \frac{\|\nu\|_2^{\frac{1}{r_\Psi-1}}}{a_\Psi^{\frac{r_\Psi}{r_\Psi-1}}}.
\]
Substituting \(t_*\) into \(g\) yields
\[
g(t_*) = \frac{r_\Psi-1}{r_\Psi} \left( \frac{\|\nu\|_2}{a_\Psi} \right)^{\frac{r_\Psi}{r_\Psi-1}}.
\]
Therefore,
\[
\Psi^*(\nu) = \frac{r_\Psi-1}{r_\Psi} \left( \frac{\|\nu\|_2}{a_\Psi} \right)^{\frac{r_\Psi}{r_\Psi-1}} - c_\Psi.
\]
Rewriting this in the standard form of a norm power function,
\[
\Psi^*(\nu) = \frac{\|a_{\Psi^*}\nu\|_2^{r_{\Psi^*}}}{r_{\Psi^*}} + c_{\Psi^*},
\]
we identify
\[
r_{\Psi^*} = \frac{r_\Psi}{r_\Psi-1}, \qquad
a_{\Psi^*} = \frac{1}{a_\Psi}, \qquad
c_{\Psi^*} = -c_\Psi.
\]
Thus, Conclusion 2 is proved.

\noindent \textbf{Proof of Conclusion 3.}

By the chain rule, \(F\) is differentiable and
\begin{equation}\label{eq:prop1_tmp1}
F'(a) = \langle \nabla\Psi(a\mu),\, \mu \rangle.
\end{equation}
Let \(z = a\mu\). If \(a \ge 0\), by Conclusion 1,
\begin{equation}\label{eq:prop1_tmp2}
\langle \nabla\Psi(a\mu),\, a\mu \rangle \ge 0.
\end{equation}
When \(a > 0\), dividing by \(a\) (a positive number) yields
\[
\langle \nabla\Psi(a\mu),\, \mu \rangle \ge 0.
\]
Combining equation~\eqref{eq:prop1_tmp1} with the above inequality, we have \(F'(a) \ge 0\) for \(a > 0\). Also, \(F'(0) = \langle \nabla\Psi(0), \mu \rangle = 0\). Therefore, \(F'(a) \ge 0\) holds for all \(a \ge 0\), so \(F\) is monotonically non-decreasing on \([0, \infty)\).

If \(\mu \neq \mathbf{0}\), then for any \(a > 0\), \(a\mu \neq \mathbf{0}\). Since \(\mathbf{0}\) is the unique global minimum, \(\Psi(a\mu) > \Psi(\mathbf{0})\). If equality holds in inequality~\eqref{eq:prop1_tmp2}, i.e., \(\langle \nabla\Psi(a\mu), a\mu \rangle = 0\), then
\[
\Psi(\mathbf{0}) = \Psi(a\mu) - \langle \nabla\Psi(a\mu), a\mu \rangle = \Psi(a\mu),
\]
which contradicts \(\Psi(a\mu) > \Psi(\mathbf{0})\). Therefore, the inequality~\eqref{eq:prop1_tmp2} is strict for \(z \neq \mathbf{0}\):
\[
\langle \nabla\Psi(z), z \rangle > 0 \quad (z \neq \mathbf{0}).
\]
Thus, for \(a > 0\),
\[
a \langle \nabla\Psi(a\mu), \mu \rangle > 0 \quad \Rightarrow \quad F'(a) > 0.
\]
Therefore, \(F'\) is strictly positive on \((0, \infty)\), and \(F\) is strictly monotonically increasing on \([0, \infty)\).

If \(\mu = \mathbf{0}\), then \(F(a) = \Psi(\mathbf{0})\) is a constant, which is monotonically non-decreasing but not strictly increasing.
When \(\mu\neq \mathbf{0}\), since \(\Psi\) is a strictly convex function, its Hessian matrix is positive semidefinite. For any vector \(\mu\):
\[\mu^\top \nabla^2\Psi(a\mu) \mu \ge 0\]
Thus \(F''(a) = \mu^\top \nabla^2\Psi(a\mu) \mu \ge 0\), so \(F(a)\) is convex.
When \(\mu=\mathbf{0}\), \(F(a)=\Psi(\mathbf{0})\) is a constant.

\subsection{Proof of Proposition~\ref{prop:gen_prop}}
\label{appendix:proof_gen_prop}

\noindent \textbf{Proof of Conclusion~\ref{prop:gen_prop_duality}}.
We prove the two implications separately.
\textbf{Part 1.}
Assume \(G\) is \(\mathcal{H}(\Psi)\)-smooth. Then for all \(\mu, \nu\),
\begin{equation}
\label{eq:dual_tmp1}
 G(\mu) \le G(\nu) + \langle \nabla G(\nu), \mu - \nu \rangle + \Psi(\mu - \nu).
\end{equation}
Fix \(\nu\) and define the auxiliary function \(H(\mu) := G(\nu) + \langle \nabla G(\nu), \mu - \nu \rangle + \Psi(\mu - \nu)\); from \eqref{eq:dual_tmp1} we have \(G(\mu) \le H(\mu)\) for all \(\mu\).
By Lemma~\ref{lem:order_reverse}, since the Fenchel conjugate is order-reversing, it follows that
\begin{equation}
\label{eq:dual_tmp2}
 G^*(\theta) \ge H^*(\theta), \quad \forall \theta.
\end{equation}
Now compute \(H^*\). Let \(\eta = \nabla G(\nu)\); then by Legendre involution, \(\nu = \nabla G^*(\eta)\). We have
\[
\begin{aligned}
H^*(\theta)
&= \sup_{\mu} \left\{ \langle \theta, \mu \rangle - G(\nu) - \langle \eta, \mu - \nu \rangle - \Psi(\mu - \nu) \right\} \\
&= \sup_{z} \left\{ \langle \theta, z + \nu \rangle - G(\nu) - \langle \eta, z \rangle - \Psi(z) \right\}, \quad (z := \mu - \nu) \\
&= \langle \theta, \nu \rangle - G(\nu) + \sup_{z} \left\{ \langle \theta - \eta, z \rangle - \Psi(z) \right\} \\
&= \langle \theta, \nu \rangle - G(\nu) + \Psi^*(\theta - \eta).
\end{aligned}
\]
Using \(G^*(\eta) = \langle \eta, \nu \rangle - G(\nu)\) and \(\nu = \nabla G^*(\eta)\), we obtain
\(\langle \theta, \nu \rangle - G(\nu) = G^*(\eta) + \langle \nabla G^*(\eta), \theta - \eta \rangle\), hence
\[
H^*(\theta) = G^*(\eta) + \langle \nabla G^*(\eta), \theta - \eta \rangle + \Psi^*(\theta - \eta).
\]
Combining this with \eqref{eq:dual_tmp2} and rearranging terms yields \(G^*(\theta) - G^*(\eta) - \langle \nabla G^*(\eta), \theta - \eta \rangle \ge \Psi^*(\theta - \eta)\), which is exactly the \(\mathcal{H}(\Psi^*)\)-convexity condition.

\textbf{Part 2.}
Assume \(G^*\) is \(\mathcal{H}(\Psi^*)\)-convex. Then for all \(\theta, \eta\),
\begin{equation}
 \label{eq:dual_tmp3}
 G^*(\theta) \ge G^*(\eta) + \langle \nabla G^*(\eta), \theta - \eta \rangle + \Psi^*(\theta - \eta).
\end{equation}
Fix \(\eta\) and define
\[
K(\theta) := G^*(\eta) + \langle \nabla G^*(\eta), \theta - \eta \rangle + \Psi^*(\theta - \eta).
\]
Then \(G^* \ge K\). Taking conjugates (and reversing the inequality) gives
\[
G^{**}(\mu) \le K^*(\mu), \quad \forall \mu.
\]
Since \(G\) is closed convex, \(G^{**} = G\), hence
\begin{equation}
 \label{eq:dual_tmp4}
 G(\mu) \le K^*(\mu).
\end{equation}
Now compute \(K^*\). Let \(\eta = \nabla G(\nu)\) (equivalently, \(\nu = \nabla G^*(\eta)\)). Then
\[
\begin{aligned}
K^*(\mu)
&= \sup_{\theta} \left\{ \langle \mu, \theta \rangle - G^*(\eta) - \langle \nabla G^*(\eta), \theta - \eta \rangle - \Psi^*(\theta - \eta) \right\} \\
&= \sup_{z} \left\{ \langle \mu, z + \eta \rangle - G^*(\eta) - \langle \nu, z \rangle - \Psi^*(z) \right\}, \quad (z := \theta - \eta) \\
&= \langle \mu, \eta \rangle - G^*(\eta) + \sup_{z} \left\{ \langle \mu - \nu, z \rangle - \Psi^*(z) \right\} \\
&= \langle \mu, \eta \rangle - G^*(\eta) + \Psi^{**}(\mu - \nu).
\end{aligned}
\]
Since \(\Psi\) is closed convex, \(\Psi^{**} = \Psi\). Also, \(\langle \mu, \eta \rangle - G^*(\eta) = G(\nu) + \langle \mu - \nu, \eta \rangle\) because \(G^*(\eta) = \langle \eta, \nu \rangle - G(\nu)\) and \(\eta = \nabla G(\nu)\). Therefore
\[
K^*(\mu) = G(\nu) + \langle \nabla G(\nu), \mu - \nu \rangle + \Psi(\mu - \nu).
\]
Substituting into inequality~\eqref{eq:dual_tmp4} yields
\[
G(\mu) \le G(\nu) + \langle \nabla G(\nu), \mu - \nu \rangle + \Psi(\mu - \nu),
\]
which is precisely the \(\mathcal{H}(\Psi)\)-smoothness condition for \(G\).

\noindent \textbf{Proof of Conclusion~\ref{prop:gen_prop_scale}}.
We prove each part separately.

\textbf{Smoothness scaling.}
Since \(\nabla F(x) = ab\,\nabla G(bx)\), we have
\[
\begin{aligned}
S_F(x, y)
&= F(x) - F(y) - \langle \nabla F(y), x-y \rangle \\
&= aG(bx) - aG(by) - \langle ab\nabla G(by), x-y \rangle \\
&= a\left[ G(bx) - G(by) - \langle \nabla G(by), b(x-y) \rangle \right] \\
&= a \cdot S_G(bx, by).
\end{aligned}
\]
If \(G\) is \(\mathcal{H}(\Psi)\)-smooth, then \(S_G(bx, by) \le \Psi(b(x-y))\). Hence
\[
S_F(x, y) \le a\Psi(b(x-y)),
\]
which proves \(F\) is \(\mathcal{H}(a\Psi(b\,\cdot))\)-smooth. For \(0 < b \le 1\), since \(\Psi\) is an energy function (radial and non-decreasing), we have \(\Psi(bz) \le \Psi(z)\), so \(S_F(x,y) \le a\Psi(x-y)\), i.e., \(F\) is \(\mathcal{H}(a\Psi)\)-smooth.

\textbf{Convexity scaling.}
If \(G\) is \(\mathcal{H}(\psi)\)-convex, then \(S_G(bx, by) \ge \psi(b(x-y))\). Multiplying by \(a>0\) yields
\[
S_F(x, y) \ge a\psi(b(x-y)),
\]
so \(F\) is \(\mathcal{H}(a\psi(b\,\cdot))\)-convex. Unlike the smoothness case, no simplification to \(\mathcal{H}(a\psi)\)-convexity holds for \(0<b\le 1\): since \(\psi\) is non-decreasing, \(\psi(bz)\le\psi(z)\) yields a weaker---not stronger---lower bound.

\noindent \textbf{Proof of Conclusion~\ref{prop:gen_prop_dataset}}.
For convenience, define \(F(\theta) = G(\theta,s)\). Then,

\[
\nabla F(\eta) = \mathbb{E}_Z[\nabla_\eta G(\eta, Z)] = \nabla_\eta G(\eta, s),
\]
and
\begin{equation*}
\begin{aligned}
S_F(\theta, \eta)
&=G(\theta,s)- G(\eta,s) - \langle \nabla_\eta G(\eta,s), \theta - \eta \rangle \\
&= \mathbb{E}_Z \left[ G(\theta,Z) - G(\eta,Z) - \langle \nabla_\eta G(\eta,Z), \theta - \eta \rangle \right] \\
&= \mathbb{E}_Z S_{G(Z)}(\theta, \eta).
\end{aligned}
\end{equation*}

Since \(G(\theta, z)\) has the same generalized convexity and smoothness properties with respect to \(\theta\) for all \(z \in \mathcal{Z}\), it follows that \(G(\theta,s)\) inherits these same properties.

This completes the proof.

\noindent \textbf{Proof of Conclusion~\ref{prop:gen_prop_transform_keep}}.
Since \(F(\mu) = d_G(\mu,s)\), by the definition of the Fenchel--Young loss, we have:
\begin{equation}
d_{F}(\mu, \nu_{F}^*) = F(\mu) + F^*(\nu_F^*) - \langle \mu, \nu_F^* \rangle.
\end{equation}

Since \(G\) is strictly convex, expanding this expression step by step:
\begin{equation}
\begin{aligned}
d_{F}(\mu, \nu_F^*) &= F(\mu) + F^*(\nu_F^*) - \langle \mu, \nu_F^* \rangle \\
&= F(\mu) + F^*(\nu_F^*) - \langle \nu, \nu_F^* \rangle + \langle \nu, \nu_F^* \rangle - \langle \mu, \nu_F^* \rangle \\
&= F(\mu) - F(\nu) - \langle \nu_F^*, \mu - \nu \rangle.
\end{aligned}
\end{equation}

Substituting \(F(\mu) = d_G(\mu,s)\), we get:
\begin{equation}
\begin{aligned}
d_{F}(\mu, \nu_{F}^*) &= d_G(\mu,s) - d_G(\nu,s) - \langle \nabla G(\nu) - s, \mu - \nu \rangle \\
&= G(\mu) - \langle s, \mu - \nu \rangle - G(\nu) - \langle \nabla G(\nu) - s, \mu - \nu \rangle \\
&= G(\mu) - G(\nu) - \langle \nabla G(\nu), \mu - \nu \rangle \\
&= d_G(\mu, \nabla G(\nu)).
\end{aligned}
\end{equation}

Since \(G\) is closed and strictly convex, it follows that \(d_G(s,\mu) = d_{G^*}(\mu,s)\).
Given that \(F(\mu) = d_G(s,\mu)\), we obtain \(F(\mu) = d_{G^*}(\mu,s)\) and \(d_F(\mu,\nu_F^*) = d_{G^*}(\mu,\nu_{G^*}^*)\).

Therefore, \(d_G(\mu, s)\) has the same generalized convexity and smoothness properties as \(G(\mu)\), and \(d_G(s, \mu)\) has the same generalized convexity and smoothness properties as \(G^*(\mu)\).

\subsection{Proof of Proposition~\ref{prop:convex_smooth_case}}
\label{appendix:proof_convex_smooth_case}
\noindent \textbf{Proof of Conclusion 1.}
 Since \(G\) is a twice-differentiable function defined on a closed set, with a positive definite Hessian matrix, by Taylor's theorem, for any \(\mu, \nu\) in the domain of \(G\), there exists \(\xi\) such that
\begin{equation}
 S_G(\mu, \nu) = \frac{1}{2} (\mu - \nu)^\top \nabla^2 G(\xi) (\mu - \nu),
\end{equation}
where \(\xi = t\mu + (1-t)\nu\).

Since \(\nabla^2 G(\xi)\) is symmetric and positive definite, we can apply the Courant--Fischer min-max theorem to obtain the eigenvalue bounds:
\begin{equation}
\begin{aligned}
     \lambda_{\min}\|\mu - \nu\|_2^2 &\le \lambda_{\min}(\nabla^2 G(\xi)) \|\mu - \nu\|_2^2 \le (\mu - \nu)^\top \nabla^2 G(\xi) (\mu - \nu) \\
     &\le \lambda_{\max}(\nabla^2 G(\xi)) \|\mu - \nu\|_2^2 \le \lambda_{\max}\|\mu - \nu\|_2^2.
\end{aligned}
\end{equation}
By the definition of \(\lambda_{\min}\) and \(\lambda_{\max}\) as the global minimum and maximum eigenvalues of the Hessian over the domain, we have:
\begin{equation}
 \frac{\lambda_{\min}}{2} \|\mu - \nu\|_2^2 \le S_G(\mu, \nu) \le \frac{\lambda_{\max}}{2} \|\mu - \nu\|_2^2.
\end{equation}

This inequality shows that \(S_G(\mu, \nu)\) is bounded both below and above by quadratic functions. Specifically:
- The upper bound implies that \(G\) is \(\mathcal{H}(\lambda_{\max}\|\cdot\|_2^2/2)\)-smooth.
- The lower bound implies that \(G\) is \(\mathcal{H}(\lambda_{\min}\|\cdot\|_2^2/2)\)-convex.

\noindent \textbf{Proof of Conclusion 2.}
By Conclusion 1, it follows directly that the MSE loss \(\frac{1}{2}\|f_\theta(x)-y\|_2^2\) is both \(\mathcal{H}(\|\cdot\|_2^2/2)\)-convex and \(\mathcal{H}(\|\cdot\|_2^2/2)\)-smooth with respect to \(f_\theta(x)\).

\noindent \textbf{Proof of Conclusion 3.}

Let \(v\) be a vector of dimension \(k\), where \(k \ge 1\).
We have:
\begin{equation}\label{eq_con_1}
\begin{aligned}
\|v\|_1^2 = \sum_{i=1}^k \sum_{j=1}^k |v_i| |v_j|
 \ge \sum_{i=1}^k |v_i|^2
 = \|v\|_2^2.
\end{aligned}
\end{equation}

Thus, applying Lemma~\ref{lem:pinsker} and substituting the above norm inequality, we obtain:
\begin{equation*}
 D_{\text{KL}}(y \| p) \;\ge\; \frac{1}{2 \ln 2} \|y - p\|_1^2 \;\ge\; \frac{1}{2 \ln 2} \|y - p\|_2^2,
\end{equation*}
where the second inequality follows from \(\|v\|_1^2\ge\|v\|_2^2\) in \eqref{eq_con_1}.
Using Lemma~\ref{lem:kl_upper_bound}, we further derive:
\begin{equation*}
\begin{aligned}
D_{\text{KL}}(y \| p) \le \frac{1}{\min_i p_i} \|y - p\|_2^2.
\end{aligned}
\end{equation*}

Combining the upper and lower bounds of \(D_{\text{KL}}(y \| p)\), we have:
\begin{equation*}
\frac{1}{2 \ln 2} \|y - p\|_2^2 \le D_{\text{KL}}(y \| p) \le \frac{1}{\min_i p_i} \|y - p\|_2^2.
\end{equation*}
Let \(p, q \in \Delta^n\), where \(\Delta^n\) denotes the \(n\)-dimensional probability simplex. When \(G(p)=\sum_{i=1}^n p_i\log p_i\) (negative Shannon Entropy), \(S_G(q,p)=D_{\text{KL}}(q\|p)\).
Therefore,
\begin{equation*}
\frac{1}{2 \ln 2} \|q - p\|_2^2 \le S_G(q,p) \le \frac{1}{\min_i p_i} \|q - p\|_2^2.
\end{equation*}

Thus, \(G\) is \(\mathcal{H}(\psi)\)-convex and \(\mathcal{H}(\Psi)\)-smooth with \(\psi(\cdot)=\frac{1}{2 \ln 2} \|\cdot\|_2^2\) and \(\Psi(\cdot)=\frac{1}{\min_i p_i} \|\cdot\|_2^2\). 
By Conclusion~\ref{prop:energy_fun2} of Proposition~\ref{prop:energy_fun} and Conclusions~\ref{prop:gen_prop_duality} and~\ref{prop:gen_prop_transform_keep} of Proposition~\ref{prop:gen_prop}, the Fenchel--Young loss \(d_G(q,\cdot)\) is \(\mathcal{H}(\psi^*)\)-smooth and \(\mathcal{H}(\Psi^*)\)-convex, where \(\psi^*(\cdot)= 2 \ln 2 \|\cdot\|_2^2\) and \(\Psi^*(\cdot)=\min_i p_i \|\cdot\|_2^2\). 
Using \(d_G(q,\mu)=D_{\text{KL}}(q\|\mathrm{Softmax}(\mu))\) \citep{Blondel2019LearningWF} and Proposition~\ref{prop:gen_prop_scale}, \(\mathrm{CrossEntropy}(q,\mathrm{Softmax}(\mu))\) is also 
\(\mathcal{H}(\psi^*)\)-smooth and \(\mathcal{H}(\Psi^*)\)-convex.

\subsection{Proof of Theorem~\ref{thm:h_bound}}
\label{appendix:proof_h_bound}

\noindent \textbf{Proof of Conclusion~\ref{thm:h_bound:1}}.

According to the definition of \(S_G(\nu,\mu)\), we have
\begin{equation}\label{eq:def_base}
\begin{aligned}
 G_*&= \min_{\nu} G(\nu)\\
 &= \min_{\nu} \{G(\mu)+(\nabla G(\mu))^\top (\nu-\mu) +S_G(\nu,\mu)\}\\
 &=\min_{z} \{G(\mu)+(\nabla G(\mu))^\top z+S_G(\mu+z,\mu)\},
\end{aligned}
\end{equation}
where the last line uses the change of variable \(z=\nu-\mu\).

Since \(\mu_*\) is a global minimum, it must also be a stationary point; hence \(\nabla G(\mu_*) = 0\).
Since \(G\) is \(\mathcal{H}(\Psi)\)-smooth, we have
\begin{equation}
 \begin{aligned}
 G(\mu)&= G(\mu_*)+\nabla G(\mu_*)^\top (\mu-\mu_*) +S_G(\mu,\mu_*)\\
 &\le G_*+\Psi(\mu-\mu_*).
 \end{aligned}
\end{equation}
Thus, the upper bound of the first conclusion is proved.

Since \(G\) is \(\mathcal{H}(\Psi)\)-smooth, we have
\begin{equation}
 S_G(\mu_*,\mu)\le \Psi(z),
\end{equation} where \(z=\mu_*-\mu\).

By substituting the above inequality into equation~\eqref{eq:def_base}, it follows that
\begin{equation}\label{eq_nonconvex_1}
\begin{aligned}
 G_*&\le\min_z \{G(\mu)+ (\nabla G(\mu))^\top z+\Psi(z)\}\\
 &=G(\mu)+ \min_z d_\Psi(z,-\nabla G(\mu))-\Psi^*(-\nabla G(\mu))\\
&=G(\mu)-\Psi^*(-\nabla G(\mu))\\
&=G(\mu)-\Psi^*(\nabla G(\mu)).
\end{aligned}
\end{equation}
The lower bound of the first conclusion is thus proved.

\noindent \textbf{Proof of Conclusion~\ref{thm:h_bound:2}}.
Since \( \mu_* \) is a global minimum, which must also be a stationary point, it follows that \( \nabla G(\mu_*) = 0 \).
Since \(G\) is \(\mathcal{H}(\psi)\)-convex, we have
\begin{equation}
 \begin{aligned}
 G(\mu)&=G(\mu_*)+\nabla G(\mu_*)^\top(\mu-\mu_*) +S_G(\mu,\mu_*)\\
 &\ge G_*+\psi(\mu-\mu_*).
 \end{aligned}
\end{equation}
Thus, the lower bound of the second conclusion is proved.

Since \(G\) is \(\mathcal{H}(\psi)\)-convex, we have
\begin{equation}
\begin{aligned}
 G(\mu_*)&= G(\mu)+\nabla G(\mu)^\top(\mu_*-\mu) +S_G(\mu_*,\mu)\\
 &\ge G(\mu)+(\nabla G(\mu))^\top z+\psi(z) \\
&= G(\mu)+d_\psi(z,-\nabla G(\mu))-\psi^*(-\nabla G(\mu))\\
&\ge G(\mu)-\psi^*(\nabla G(\mu)),
\end{aligned}
\end{equation}
where \(z=\mu_*-\mu\).
Since \(d_\psi(z,-\nabla G(\mu))\ge 0\), we have
\begin{equation}\label{eq_nonconvex_2}
G(\mu)\le G_*+\psi^*(\nabla G(\mu)).
\end{equation}

\noindent \textbf{Proof of Conclusion~\ref{thm:h_bound:3}}.
Given that \(G(\mu)\) is \(\mathcal{H}(\psi)\)-convex and \(\mathcal{H}(\Psi)\)-smooth, by combining equation~\eqref{eq_nonconvex_1} and equation~\eqref{eq_nonconvex_2}, we obtain
\[
\begin{aligned}
 \Psi^*(\nabla G(\mu)) &\le G(\mu) - G_* \le \psi^*(\nabla G(\mu)).
\end{aligned}
\]
Therefore, the proof of the theorem is complete.

\subsection{Proof of Proposition~\ref{prop:l2_unique}}
\label{appendix:proof_l2_unique}

By positive homogeneity of degree one, Euler's homogeneous function theorem gives
\(\nabla\|x\|^\top x = \|x\|\) for all \(x\neq 0\); indeed, differentiating \(g(t)=\|tx\|=t\|x\|\) at \(t=1\) yields the identity.
Substituting the collinearity assumption into Euler's identity,
\[
(\lambda(x)x)^\top x = \|x\|
\;\Longrightarrow\;
\lambda(x)\|x\|_2^2 = \|x\|
\;\Longrightarrow\;
\lambda(x) = \frac{\|x\|}{\|x\|_2^2},
\qquad\text{hence}\qquad
\nabla\|x\| = \frac{\|x\|}{\|x\|_2^2}\,x.
\]
Dividing by \(\|x\|>0\) and using \(\nabla\log\|x\| = \nabla\|x\|/\|x\|\) together with \(\nabla\log\|x\|_2 = x/\|x\|_2^2\), we obtain
\(\nabla\bigl(\log\|x\| - \log\|x\|_2\bigr) = 0\).
Because \(\mathbb{R}^n\setminus\{0\}\) is connected for \(n\ge 2\), the function \(\log\|x\| - \log\|x\|_2\) is constant. Thus there exists \(c>0\) such that
\[
\|x\| = c\,\|x\|_2.
\]
The equality also holds at \(x=0\).

Conversely, if \(\|x\| = c\|x\|_2\) with \(c>0\), then
\[
\nabla\|x\| = c\,\nabla\|x\|_2 = c\,\frac{x}{\|x\|_2},
\]
which is clearly collinear with \(x\). This completes the proof.

\subsection{Proof of Theorem~\ref{thm:general_gd_convergence}}
\label{appendix:proof_general_gd_convergence}

\noindent \textbf{Proof of Conclusion 1.}

By Theorem~\ref{thm:h_bound}, the function \(G\) satisfies
\begin{equation}
\begin{aligned}
G(\eta)
&= G(\theta) + (\nabla G(\theta))^\top (\eta-\theta) + S_G(\eta,\theta) \\
&\le G(\theta) + (\nabla G(\theta))^\top (\eta-\theta) + \Psi(\theta-\eta).
\end{aligned}
\end{equation}
We analyze the range of learning rates that ensure descent along the direction
\(-(\nabla G(\theta))_{\Psi^*}^*\).

Let the update be given by \(\eta = \theta - \alpha (\nabla G(\theta))_{\Psi^*}^*\) with \(\alpha \ge 0\).
Substituting the update rule, we obtain

\begin{equation*}
\begin{aligned}
G(\theta-\alpha(\nabla G(\theta))_{\Psi^*}^*)
&\le G(\theta) - \alpha (\nabla G(\theta))^\top (\nabla G(\theta))_{\Psi^*}^*
+ \Psi(\alpha(\nabla G(\theta))_{\Psi^*}^*)
\end{aligned}
\end{equation*}

Let \[F(\alpha) = G(\theta) - \alpha (\nabla G(\theta))^\top (\nabla G(\theta))_{\Psi^*}^*
+ \Psi(\alpha(\nabla G(\theta))_{\Psi^*}^*)\]

By Proposition~\ref{prop:energy_fun}, \(F(\alpha)\) is strictly convex and
\[
\begin{aligned}
&F'(\alpha)=\Big(\nabla \Psi\big(\alpha \nabla \Psi^*(\nabla G(\theta))\big)\Big)^\top
\Big(\nabla \Psi^*(\nabla G(\theta))\Big) - \nabla G(\theta)^\top (\nabla G(\theta))_{\Psi^*}^*,\\
&\nabla G(\theta)^\top (\nabla G(\theta))_{\Psi^*}^* \ge 0.
\end{aligned}
\]
Observe that when \(\alpha=1\), \(F'(\alpha)=0\), so \(F(\alpha)\) attains its minimum. Substituting \(\alpha=1\) and using the Fenchel--Young identity
\begin{equation*}
\begin{aligned}
(\nabla G(\theta))^\top (\nabla G(\theta))_{\Psi^*}^*
&= \Psi^*(\nabla G(\theta))+\Psi((\nabla G(\theta))_{\Psi^*}^*),
\end{aligned}
\end{equation*}
we obtain
\begin{equation}
G(\theta_{k+1}) \le G(\theta_k) - \Psi^*(\nabla G(\theta_k)).
\end{equation}

\noindent \textbf{Proof of Conclusion 2.}
We now use the descent inequality established above to bound the iteration complexity. 
From Conclusion 1,
\begin{equation}
G(\theta_{k+1}) \le G(\theta_k) - \Psi^*(\nabla G(\theta_k)).
\end{equation}
Summing over \(T\) iterations leads to
\begin{equation}
\begin{aligned}
G(\theta_{T}) \le G(\theta_0) - \sum_{k=0}^{T-1}\Psi^*(\nabla G(\theta_k)).
\end{aligned}
\end{equation}
Since \(G_* = \inf_\theta G(\theta)\), rearranging gives
\begin{equation}
\begin{aligned}
G(\theta_0) - G_*
&\ge G(\theta_0) - G(\theta_T) \\
&\ge \sum_{k=0}^{T-1}\Psi^*(\nabla G(\theta_k)) \\
&\ge T \min_{0 \le k \le T-1} \Psi^*(\nabla G(\theta_k)).
\end{aligned}
\end{equation}
Since \(\Psi^*(\mu)\) is strictly increasing in \(\|\mu\|_2\), the condition \(\|\nabla G(\theta_k)\|_2 \le \|\varepsilon\|_2\) is equivalent to \(\Psi^*(\nabla G(\theta_k)) \le \Psi^*(\varepsilon)\). Therefore, to guarantee \(\min_k \Psi^*(\nabla G(\theta_k)) \le \Psi^*(\varepsilon)\), it suffices to require
\begin{equation*}
\frac{G(\theta_0)-G_*}{T} \le \Psi^*(\varepsilon).
\end{equation*}
Thus the iteration complexity is bounded by \(T = \mathcal{O}(1/\Psi^*(\varepsilon))\).

\noindent \textbf{Proof of Conclusion 3.}

Since \(G(\theta)\) is \(\mathcal{H}(\psi)\)-convex, Theorem~\ref{thm:h_bound} implies
\begin{equation}
G(\theta) - G_* \le \psi^*(\nabla G(\theta)),
\end{equation}
where \(G_* = G(\theta_*)\) is the global minimum.
Since both \(\psi^*\) and \(\Psi^*\) are increasing in the norm of their arguments, the condition \(G(\theta)-G_* \le \psi^*(\varepsilon)\) is implied by \(\psi^*(\nabla G(\theta_k)) \le \psi^*(\varepsilon)\). 
Moreover, the latter is implied by \(\|\nabla G(\theta_k)\|_2 \le \|\varepsilon\|_2\), which is equivalent to \(\Psi^*(\nabla G(\theta_k)) \le \Psi^*(\varepsilon)\). 
From Conclusion 2, the number of iterations is \(T = \mathcal{O}(1/\Psi^*(\varepsilon))\).

\noindent \textbf{Proof of Conclusion 4.}
Since \(G(\theta)\) is \(\mathcal{H}(\psi)\)-convex, Theorem~\ref{thm:h_bound} implies
\[
G(\theta_k) - G_* \le \psi^*(\nabla G(\theta_k)).
\]
Since \(G(\theta)\) is \(\mathcal{H}(\Psi)\)-smooth, the descent guarantee gives \(G(\theta_k)-G(\theta_{k+1}) \ge \Psi^*(\nabla G(\theta_k))\).
Combining this with \(\Psi^*(\mu)\ge t\psi^*(\mu)+b\) and the convexity bound \(\psi^*(\nabla G(\theta_k)) \ge G(\theta_k) - G_*\), we obtain
\begin{equation}
 G(\theta_k)-G(\theta_{k+1})
 \ge t\big(G(\theta_k) - G_*\big)+b.
\end{equation}
Rearranging,
\begin{equation}
 \label{eq:convergence4_tmp1}
 G(\theta_{k+1})\le (1-t)G(\theta_k)+tG_*-b.
\end{equation}
Define the auxiliary variable
\[
\begin{aligned}
 H(\theta_{k})&= G(\theta_k)-G_* +\frac{b}{t}
\end{aligned}
\]
Substituting \(H(\theta_k)\) into \eqref{eq:convergence4_tmp1} yields
\[
H(\theta_{k+1})\le \left(1-t\right)H(\theta_k).
\]

Since \(0 < t < 1\), iterating the recursion over \(T\) steps gives
\begin{equation}
 \begin{aligned}
 H(\theta_{T})
 &\le H(\theta_{T-1}) \left( 1 - t \right) \\
 &\;\;\vdots \\
 &\le H(\theta_{0})\left( 1 - t \right)^{T}.
 \end{aligned}
\end{equation}

To ensure \(\min_k H(\theta_k)\le \psi^*(\varepsilon)-\psi^*(0)\), it suffices to require
\begin{equation}
H(\theta_{0}) \left( 1 - t\right)^{T} \le \psi^*(\varepsilon)-\psi^*(0).
\end{equation}
Taking the natural logarithm on both sides, we obtain
\[
\log \frac{H(\theta_{0})}{\psi^*(\varepsilon)-\psi^*(0)}
\le T \cdot \log \left( \frac{1}{1 - t} \right).
\]
Rearranging yields the iteration complexity
\[
T = \mathcal{O}\left(\kappa \log \frac{1}{\psi^*(\varepsilon)-\psi^*(0)}\right),
\]
where the condition number is defined as
\[
\kappa = -\frac{1}{\log ( 1 - t)}.
\]

By the definition of \(H(\theta)\), the condition \(\min_k H(\theta_{k}) \le \psi^*(\varepsilon)-\psi^*(0)\) is equivalent to the existence of some \(\theta_k\) such that
\begin{equation*}
 G(\theta_{k})- G_* \le \psi^*(\varepsilon)-\psi^*(0) -\frac{b}{t}.
\end{equation*}
This confirms the logarithmic convergence rate under the stated assumptions.

\subsection{Proof of Proposition~\ref{prop:model_risk_prop}}
\label{appendix:proof_model_risk_prop}
Let \(Z_1,Z_2,\dots\) be i.i.d. random variables drawn from the empirical distribution \(\hat q\) of \(s\) (equivalently, sampled uniformly from \(s\) with replacement).
Denote:
\[
s^m = (Z_1,\dots,Z_m),
\]
Given \(s^{m+1}=(Z_1,\dots,Z_{m+1})\), for any \(k=1,\dots,m+1\), denote:
\[
s^{m}_{(k)} = s^{m+1}\setminus\{Z_k\},
\]
i.e., the subset of size \(m\) obtained by removing the \(k\)-th sample.

For any \(\theta\):
\[
\begin{aligned}
 G(\theta,s^{m+1})
&= \frac1{m+1}\sum_{i=1}^{m+1} G(\theta,Z_i)\\
&= \frac1{m+1}\sum_{k=1}^{m+1}
 \frac1m\sum_{i\neq k}G(\theta,Z_i)\\
 &=\frac1{m+1}\sum_{k=1}^{m+1} G(\theta,s_{(k)}^{m}),
\end{aligned}
\]
where \(G(\theta,s_{(k)}^{m})\) is the empirical loss on \(s^m_{(k)}\).

For each \(k\), denote:
\[
\hat\theta_k = \arg \min_\theta G(\theta,s_{(k)}^{m}).
\]
Since the minimum of an average is at least the average of the minima, we have
\[
\min_\theta G(\theta,s^{m+1})
\;\ge\;
\frac1{m+1}\sum_{k=1}^{m+1} G(\hat\theta_k,s_{(k)}^{m}).
\]
Taking expectations on both sides with respect to \(s^{m+1}\):
\[
\mathbb{E}\min_\theta G(\theta,s^{m+1})
\;\ge\;
\frac1{m+1}\sum_{k=1}^{m+1}
\mathbb{E}\Bigl[\min_\theta G(\theta,s_{(k)}^{m})\Bigr].
\]

By symmetry, all \(\mathbb{E}\min_\theta G(\theta,s_{(k)}^{m})\) are equal to \(\mathbb{E}\min_\theta G(\theta,s^{m})\). Therefore,
\[
\mathbb{E}\min_\theta G(\theta,s^{m+1})
\;\ge\;
\mathbb{E}\min_\theta G(\theta,s^{m}).
\]
Subtracting the constant \(\mathbb{E}[G(Z)_*]\) from both sides gives
\[
    R_G(s,m+1) \ge R_G(s,m).
\]
Finally, when \(m=1\), we have \(s^1 = Z\) and hence \(R_G(s,1) = 0\). Since \(R_G(s,m)\) is non-decreasing, its minimum is \(0\) at \(m=1\) and its maximum is attained at \(m=|s|\). This completes the proof.

\subsection{Proof of Theorem~\ref{thm:general_sgd_convergence}}
\label{appendix:proof_general_sgd_convergence}

\begin{proof}[Proof of Theorem~\ref{thm:general_sgd_convergence}]
For notational convenience, define \(u_k=\nabla G(\theta_{k},s_k)\) and \(v_k=\nabla G(\theta_k,s)\), so that \(u_k-\xi_k=v_k\).

\textbf{Proof of Conclusion (i).} The proof proceeds in two steps. \emph{Step 1: bounding the full-dataset gradient energy \(\mathbb E[\Psi^*(v_k)]\).} By Proposition~\ref{prop:gen_prop_dataset}, \(G(\cdot,s)\) is \(\mathcal{H}(\Psi)\)-smooth, so the descent guarantee of Theorem~\ref{thm:general_gd_convergence} gives \(G(\theta_{k+1},s)\le G(\theta_k,s)-\Psi^*(u_k)
+\langle\xi_k,(u_k)_{\Psi^*}^*\rangle\). By Proposition~\ref{prop:gen_prop_duality}, \(\Psi^*\) is \(\mathcal{H}(\Pi^*)\)-smooth, and since \(\Pi\) is an energy function, \(\Pi^*(-\xi_k)=\Pi^*(\xi_k)\):
\[
    \Psi^*(v_k)\le\Psi^*(u_k)
    -\langle\xi_k,(u_k)_{\Psi^*}^*\rangle+\Pi^*(\xi_k).
\]
Hence \(G(\theta_{k+1},s)\le G(\theta_k,s)-\Psi^*(v_k)+\Pi^*(\xi_k)\).
Taking the full expectation and using \(\mathbb E[\Pi^*(\xi_k)]\le\Sigma_\Pi\) gives
\[
    \mathbb E[G(\theta_{k+1},s)]
    \le\mathbb E[G(\theta_k,s)]-\mathbb E[\Psi^*(v_k)]+\Sigma_\Pi.
\]
Summing over \(k=0,\ldots,T-1\):
\begin{equation}
\label{eq:sum_vk}
    \sum_{k=0}^{T-1}\mathbb E[\Psi^*(v_k)]\le G(\theta_0,s)-G(s)_*+T\Sigma_\Pi.
\end{equation}

\emph{Step 2: passing to the stochastic gradient energy \(\mathbb E[\Psi^*(u_k)]\).}
Since \(\Psi^*\) is \(\mathcal{H}(\Pi^*)\)-smooth,
\[
    \Psi^*(u_k)\le\Psi^*(v_k)
    +\langle\xi_k,\nabla\Psi^*(v_k)\rangle+\Pi^*(\xi_k).
\]
Taking the conditional expectation over \(\theta_k\) and using \(\mathbb E[\xi_k\mid\theta_k]=0\):
\[
    \mathbb E[\Psi^*(u_k)]\le\mathbb E[\Psi^*(v_k)]+\Sigma_\Pi.
\]
Combining with \eqref{eq:sum_vk} yields
\(\sum_{k=0}^{T-1}\mathbb E[\Psi^*(u_k)]\le G(\theta_0,s)-G(s)_*+2T\Sigma_\Pi\),
and hence
\[
    \min_{0\le k<T}\mathbb E[\Psi^*(u_k)]
    \le\frac{G(\theta_0,s)-G(s)_*}{T}+2\Sigma_\Pi.
\]
Setting \(T=T_\varepsilon\) gives \eqref{eq:energy_bound}.

\medskip
\textbf{Proof of Conclusion (ii).} By Theorem~\ref{thm:h_bound} and \(\Psi^*\ge t\psi^*+b\):
\[
    G(\theta_k,s_k)-G(s_k)_*
    \le\psi^*(u_k)\le\frac{\Psi^*(u_k)-b}{t}.
\]
Taking the conditional expectation over \(\theta_k\):
\[
    \mathbb E[G(\theta_k,s)]-\mathbb E[G(s_k)_*]
    \le\frac{\mathbb E[\Psi^*(u_k)]-b}{t}.
\]
By Definition~\ref{def:model_risk}, \(\mathbb E G(s_k)_*=R_G(s,m)+\mathbb E[G(Z)_*]\), which upon substitution gives
\[
    \mathbb E[G(\theta_k,s)]-\mathbb E[G(Z)_*]
    \le\frac{\mathbb E[\Psi^*(u_k)]-b}{t}+R_G(s,m).
\]
Again by Definition~\ref{def:model_risk}, \(G(s)_*=R_G(s,n)+\mathbb E[G(Z)_*]\), i.e., \(\mathbb E[G(Z)_*]=G(s)_*-R_G(s,n)\); substituting yields
\begin{equation}
\label{eq:risk_gap_one_step}
    \mathbb E[G(\theta_k,s)]-G(s)_*
    \le\frac{\mathbb E[\Psi^*(u_k)]-b}{t}
    +R_G(s,m)-R_G(s,n).
\end{equation}
Taking the minimum over \(0\le k<T_\varepsilon\) and applying Conclusion (i):
\[
    \min_{0\le k<T_\varepsilon}
    \bigl\{\mathbb E[G(\theta_k,s)]-G(s)_*\bigr\}
    \le\frac{\Psi^*(\varepsilon)-b}{t}
    +R_G(s,m)-R_G(s,n).
\]
This is \eqref{eq:risk_gap_bound}. Since \(m\le n\) and \(R_G\) is non-decreasing in \(m\), we have \(R_G(s,m)-R_G(s,n)\le 0\), so the right-hand side does not exceed \(\frac{\Psi^*(\varepsilon)-b}{t}\). This completes the proof.
\end{proof}

\subsection{Proof of Corollary~\ref{cor:general_sgd_geometric}}
\label{appendix:proof_general_sgd_geometric}

\begin{proof}[Proof of Corollary~\ref{cor:general_sgd_geometric}]
For notational convenience, write
\[
u_k:=\nabla G(\theta_k,s_k),\qquad
v_k:=\nabla G(\theta_k,s),\qquad
\xi_k:=u_k-v_k.
\]
By the assumptions of Theorem~\ref{thm:general_sgd_convergence}, \(\mathbb E[\xi_k\mid\theta_k]=0\) and \(\mathbb E[\Pi^*(\xi_k)\mid\theta_k]\le\Sigma_\Pi\).

\textbf{Step 1: full-dataset smoothness yields descent.}
By the \(\mathcal H(\Psi)\)-smoothness of \(G(\cdot,s)\),
\[
G(\theta_{k+1},s)
\le G(\theta_k,s)
+\langle v_k,\theta_{k+1}-\theta_k\rangle
+\Psi(\theta_{k+1}-\theta_k).
\]
Generalized SGD takes \(\alpha=1\), so
\(\theta_{k+1}-\theta_k=-(u_k)_{\Psi^*}^*\). Using the Fenchel identity
\[
\Psi((u_k)_{\Psi^*}^*)
=\langle u_k,(u_k)_{\Psi^*}^*\rangle-\Psi^*(u_k),
\]
we obtain
\[
G(\theta_{k+1},s)
\le G(\theta_k,s)
+\langle\xi_k,(u_k)_{\Psi^*}^*\rangle
-\Psi^*(u_k).
\]

\textbf{Step 2: eliminating the cross term.}
By the \(\mathcal H(\Pi^*)\)-smoothness of \(\Psi^*\),
\[
\Psi^*(v_k)
\le \Psi^*(u_k)
-\langle\xi_k,(u_k)_{\Psi^*}^*\rangle
+\Pi^*(\xi_k),
\]
i.e.,
\[
\langle\xi_k,(u_k)_{\Psi^*}^*\rangle
\le \Psi^*(u_k)-\Psi^*(v_k)+\Pi^*(\xi_k).
\]
Substituting back into the inequality of Step 1 gives
\[
G(\theta_{k+1},s)
\le G(\theta_k,s)-\Psi^*(v_k)+\Pi^*(\xi_k).
\]
Taking the conditional expectation over \(\theta_k\) and using
\(\mathbb E[\Pi^*(\xi_k)\mid\theta_k]\le\Sigma_\Pi\), we obtain
\[
\mathbb E[G(\theta_{k+1},s)\mid\theta_k]
\le G(\theta_k,s)-\Psi^*(v_k)+\Sigma_\Pi.
\]

\textbf{Step 3: contraction from convexity.}
By the \(\mathcal H(\psi)\)-convexity of \(G(\cdot,s)\),
\[
G(\theta_k,s)-G(s)_*\le \psi^*(v_k).
\]
Together with \(\Psi^*(w)\ge t\psi^*(w)+b\), this gives
\[
\Psi^*(v_k)\ge t\psi^*(v_k)+b
\ge t\bigl[G(\theta_k,s)-G(s)_*\bigr]+b.
\]
Substituting into the inequality of Step 2:
\[
\mathbb E[G(\theta_{k+1},s)\mid\theta_k]-G(s)_*
\le (1-t)\bigl[G(\theta_k,s)-G(s)_*\bigr]
+\Sigma_\Pi-b.
\]
Taking the full expectation and letting \(A_k=\mathbb E[G(\theta_k,s)]-G(s)_*\) yields
\begin{equation}\label{eq:geometric_recursion_cor}
A_{k+1}\le (1-t)A_k+\Sigma_\Pi-b.
\end{equation}

\textbf{Step 4: closed-form solution.}
Let \(c:=\Sigma_\Pi-b\). Iterating \eqref{eq:geometric_recursion_cor} gives
\[
A_T\le (1-t)^T A_0
+\frac{c}{t}\bigl[1-(1-t)^T\bigr].
\]
When \(0<t<1\), \(A_0>c/t\), and \(\delta>c/t\), the condition
\((1-t)^T(A_0-c/t)\le \delta-c/t\) is solved by
\[
T\ge
\frac{\log\bigl((A_0-c/t)/(\delta-c/t)\bigr)}
{-\log(1-t)},
\]
in which case \(A_T\le\delta\). When \(t=1\), the recursion directly gives \(A_{k+1}\le c\), hence \(A_T\le c=\Sigma_\Pi-b\). This completes the proof.
\end{proof}

\subsection{Proof of Theorem~\ref{thm:com_general_bound}}
\label{appendix:proof_com_general_bound}

By applying the chain rule for differentiation, we have
\begin{equation*}
\begin{aligned}
 \nabla_\theta G(\theta, z)&=\frac{\partial h_\theta(z)}{\partial \theta} \nabla_h G(\theta, z).
\end{aligned}
\end{equation*}
According to the Courant--Fischer min-max theorem, we have
\[
\begin{aligned}
 \lambda_{\min}(A_z) \| \nabla_h G(\theta, z)\|^2_2\le \|\nabla_\theta G(\theta, z)\|^2_2\le \lambda_{\max}(A_z) \|\nabla_h G(\theta, z)\|^2_2.
\end{aligned}
\]
It follows that
\[
\begin{aligned}
 \frac{ \|\nabla_\theta G(\theta, z)\|_2}{\lambda_{\max}(A_z)^{1/2}} \le \| \nabla_h G(\theta, z)\|_2\le \frac{\|\nabla_\theta G(\theta, z)\|_2}{\lambda_{\min}(A_z)^{1/2}}.
\end{aligned}
\]
According to the definition of energy functions and Conclusion~\ref{prop:energy_fun3} of Proposition~\ref{prop:energy_fun}, we have
\begin{equation}\label{eq:uu_ll_1}
 \begin{aligned}
 &\phi^*(\nabla_h F(h_\theta(z)))\le \phi^*(\frac{\nabla_\theta G(\theta, z)}{\lambda_{\min}(A_z)^{1/2}}),\\
 &\Phi^*( \frac{ \nabla_\theta G(\theta, z)}{\lambda_{\max}(A_z)^{1/2}})\le \Phi^*(\nabla_h F(h_\theta(z))).
 \end{aligned}
\end{equation}

By Conclusion~\ref{thm:h_bound:3} of Theorem~\ref{thm:h_bound}, we have
\begin{equation}
 \Phi^*(\nabla_h F(h_\theta(z))) \le G(\theta, z) - F_* \le \phi^*(\nabla_h F(h_\theta(z))).
\end{equation}
Therefore, we have
\begin{equation*}
\begin{aligned}
 & G(\theta, z) -F_* \le \phi^*(\nabla_h F(h_\theta(z)))\le \phi^*(\frac{\nabla_\theta G(\theta, z)}{\lambda_{\min}(A_z)^{1/2}}),\\
 & G(\theta, z) - F_* \ge \Phi^*(\nabla_h F(h_\theta(z)))\ge \Phi^*( \frac{ \nabla_\theta G(\theta, z)}{\lambda_{\max}(A_z)^{1/2}}).
\end{aligned}
\end{equation*}
Combining the above two inequalities, we obtain
\begin{equation}
 \Phi^*(\frac{\nabla_\theta G(\theta, z)}{\lambda_{\max}(A_z)^{1/2}}) \le G(\theta, z) - F_* \le \phi^*(\frac{\nabla_\theta G(\theta, z)}{\lambda_{\min}(A_z)^{1/2}}).
\end{equation}

Taking expectations with respect to \(Z\) in the above inequality, we obtain
\begin{equation*}
\begin{aligned}
 \mathbb{E}\Phi^*( \frac{ \nabla_\theta G(\theta, Z)}{\lambda_{\max}(A_Z)^{1/2}}) \le G(\theta, s) - F_* \le \mathbb{E}\phi^*( \frac{\nabla_\theta G(\theta, Z)}{\lambda_{\min}(A_Z)^{1/2}}).
\end{aligned}
\end{equation*}

\subsection{Proof of Corollary~\ref{cor:kl_mse_bound}}
\label{appendix:proof_kl_mse_bound}

\textbf{Part 1 (MSE).}
By Proposition~\ref{prop:convex_smooth_case_mse}, the MSE loss \(\frac{1}{2}\|f_\theta(x)-y\|_2^2\) is both \(\mathcal{H}(\phi)\)-convex and \(\mathcal{H}(\Phi)\)-smooth with respect to \(f_\theta(x)\), where \(\phi(\cdot)=\Phi(\cdot)=\|\cdot\|_2^2/2\).
By Conclusion~\ref{prop:energy_fun2} of Proposition~\ref{prop:energy_fun}, \(\phi^*(\cdot)=\Phi^*(\cdot)=\|\cdot\|_2^2/2\).
According to Corollary~\ref{cor:erm_bound}, we obtain the stated bounds:
 \begin{equation*}
\begin{aligned}
 \frac{\|\nabla_\theta \mathcal{L}(\theta, z)\|_2^2}{2\lambda_{\max}(A_z)}&\le \mathcal{L}(\theta, z) - \mathcal{L}(z)_* \le \frac{\|\nabla_\theta \mathcal{L}(\theta, z)\|_2^2}{2\lambda_{\min}(A_z)}.
\end{aligned}
 \end{equation*}

 Since \(\mathcal{L}(z)_*=0\) for all \(z\in \mathcal{Z}\), we have
 \begin{equation*}
\begin{aligned}
 \frac{\|\nabla_\theta \mathcal{L}(\theta, z)\|_2^2}{2\lambda_{\max}(A_z)}&\le \mathcal{L}(\theta, z) \le \frac{\|\nabla_\theta \mathcal{L}(\theta, z)\|_2^2}{2\lambda_{\min}(A_z)}.
\end{aligned}
 \end{equation*}

\noindent \textbf{Part 2 (Softmax CrossEntropy).}
By Proposition~\ref{prop:convex_smooth_case_ce}, the loss \(\mathrm{CrossEntropy}(y,\mathrm{Softmax}(f_\theta(x)))\) is both \(\mathcal{H}(\phi)\)-convex and \(\mathcal{H}(\Phi)\)-smooth with respect to \(f_\theta(x)\), where \(\phi(\cdot)=\min_i p_i(x) \|\cdot\|_2^2\), \(\Phi(\cdot)=2 \ln 2 \|\cdot\|_2^2\), and \(p(x)=\mathrm{Softmax}(f_\theta(x))\).
By Conclusion~\ref{prop:energy_fun2} of Proposition~\ref{prop:energy_fun}, \(\phi^*(\cdot)=\frac{\|\cdot\|_2^2}{4\min_i p_i(x)}\) and \(\Phi^*(\cdot)=\frac{\|\cdot\|_2^2}{8 \ln 2}\).
According to Corollary~\ref{cor:erm_bound} and Definition~\ref{def:induced_matrix_norm}, and since \(\mathcal{L}(z)_*=0\) for all \(z\in \mathcal{Z}\), we obtain the stated bounds:
 \begin{equation*}
\begin{aligned}
 \frac{\|\nabla_\theta \mathcal{L}(\theta, z)\|_2^2}{8 \ln 2 \,\lambda_{\max}(A_z)}&\le \mathcal{L}(\theta, z) \le \frac{\|\nabla_\theta \mathcal{L}(\theta, z)\|_2^2}{4\min_i p_i(x)\lambda_{\min}(A_z)},
\end{aligned}
 \end{equation*}
where \(p(x)=\mathrm{Softmax}(f_\theta(x))\) and \(\min_i p_i(s)=\min_{x\in s}\min_i p_i(x)\).

\noindent \textbf{Part 3 (Fenchel Loss Bound).}
According to Proposition~\ref{prop:convex_smooth_case_hessian}, \(\Omega\) is
\(\mathcal{H}(\Phi)\)-smooth and \(\mathcal{H}(\phi)\)-convex, where \(\Phi(\cdot)=\mu_{\max}\|\cdot\|_2^2/2\) and \(\phi(\cdot)=\mu_{\min}\|\cdot\|_2^2/2\).
By Conclusion~\ref{prop:gen_prop_duality} of Proposition~\ref{prop:gen_prop} and Conclusion~\ref{prop:energy_fun2} of Proposition~\ref{prop:energy_fun}, \(\Omega^*\) is \(\mathcal{H}(\phi^*)\)-smooth and \(\mathcal{H}(\Phi^*)\)-convex, where \(\Phi^*(\cdot)=\frac{\|\cdot\|_2^2}{2\mu_{\max}}\) and \(\phi^*(\cdot)=\frac{\|\cdot\|_2^2}{2\mu_{\min}}\).

By Proposition~\ref{prop:gen_prop_transform_keep}, \(d_\Omega(y,f_\theta(x))\) is \(\mathcal{H}(\phi^*)\)-smooth and \(\mathcal{H}(\Phi^*)\)-convex, where \(\Phi^*(\cdot)=\frac{\|\cdot\|_2^2}{2\mu_{\max}}\) and \(\phi^*(\cdot)=\frac{\|\cdot\|_2^2}{2\mu_{\min}}\).
According to Corollary~\ref{cor:erm_bound} and Definition~\ref{def:induced_matrix_norm}, and since \(\mathcal{L}(z)_*=0\) for all \(z\in \mathcal{Z}\), we obtain the stated bounds:
 \begin{equation*}
\begin{aligned}
 \frac{\mu_{\min}\|\nabla_\theta \mathcal{L}(\theta, z)\|_2^2}{2\lambda_{\max}(A_z)}&\le \mathcal{L}(\theta, z) \le \frac{\mu_{\max}\|\nabla_\theta \mathcal{L}(\theta, z)\|_2^2}{2\lambda_{\min}(A_z)}.
\end{aligned}
 \end{equation*}

\subsection{Proof of Proposition~\ref{prop:skip_con}}
\label{appendix:proof_skip_con}
For a residual block \(g_{\theta^{(k)}}(z) = f_{\theta^{(k)}}(z) + z\), the chain rule gives
\[
\nabla_{\theta^{(k)}} g_\theta(x) = 
\left( \prod_{\ell = k+1}^{m} (J_\ell + I) \right) \cdot \frac{\partial f_{\theta^{(k)}}}{\partial \theta^{(k)}}(z_{k-1}),
\]
where \(J_\ell = \nabla_z f_{\theta^{(\ell)}}(z_{\ell-1})\). 
    For a plain network, the corresponding block is
\[
\nabla_{\theta^{(k)}} f_\theta(x) = 
\left( \prod_{\ell = k+1}^{m} J_\ell \right) \cdot \frac{\partial f_{\theta^{(k)}}}{\partial \theta^{(k)}}(z_{k-1}).
\]

\noindent\textbf{Residual network.}
By the triangle inequality for the spectral norm, for each residual block,
\[
\sigma_{\max}(J_k + I) \le 1 + \delta_k \le 1 + \delta_U.
\]
Thus for each parameter block,
\[
\begin{aligned}
    \|\nabla_{\theta^{(k)}} g_\theta(x)\|_2 
    &\le \left\| \prod_{\ell=k+1}^{m} (J_\ell + I) \right\|_2 \cdot \left\| \frac{\partial f_{\theta^{(k)}}}{\partial \theta^{(k)}}(z_{k-1}) \right\|_2 \\
    &\le \nu \prod_{\ell=k+1}^m (1+\delta_\ell) \\
    &\le \nu (1+\delta_U)^{m-k}.
\end{aligned}
\]
Then, we have
\[
\begin{aligned}
    \lambda_{\max}(A_x) 
    &\le \sum_{k=1}^m \|\nabla_{\theta^{(k)}} g_\theta(x)\|_2^2 \\
    &\le \nu^2 \sum_{k=1}^m (1+\delta_U)^{2(m-k)} \\
    &= \nu^2 \sum_{j=0}^{m-1} (1+\delta_U)^{2j} \\
    &= \nu^2 \frac{(1+\delta_U)^{2m} - 1}{(1+\delta_U)^2 - 1} \\
    &\le C_\nu (1+\delta_U)^{2m},
\end{aligned}
\]
where \(C_\nu = \nu^2 / ((1+\delta_U)^2 - 1)\).  

\noindent\textbf{Plain network.}
For each plain parameter block,
\[
\|\nabla_{\theta^{(k)}} f_\theta(x)\|_2 \le \nu \prod_{\ell=k+1}^m \delta_\ell \le \nu \delta_U^{m-k}.
\]
Thus
\[
\lambda_{\max}(A_x) \le \sum_{k=1}^m \|\nabla_{\theta^{(k)}} f_\theta(x)\|_2^2
\le \nu^2 \sum_{j=0}^{m-1} \delta_U^{2j}
\le \frac{\nu^2}{1-\delta_U^2}.
\]
This completes the proof.

\subsection{Proof of Lemma~\ref{lem:nonlipsmooth}}
\label{appendix:proof_nonlipsmooth}

We first show that \(G\) is not Lipschitz smooth. For \(x \neq 0\), the Hessian of \(G(x) = \|x\|_2^r\) is
\(\nabla^2 G(x) = r \|x\|^{r-4} [ (r-2) x x^\top + \|x\|^2 I ]\). Hence, for the unit vector \(v = x/\|x\|\) aligned with \(x\), the corresponding quadratic form satisfies
\(
v^\top \nabla^2 G(x) v = r \|x\|^{r-4} [ (r-2) \|x\|^2 + \|x\|^2 ] = r(r-1) \|x\|^{r-2}.
\)
As \(x \to 0\) and \(1 < r < 2\), we have \(r-2 < 0\), so \(\|x\|^{r-2} \to +\infty\), implying unbounded eigenvalues of the Hessian. For a \(C^2\) function, gradient Lipschitz continuity is equivalent to a uniformly bounded Hessian. Thus, the gradient of \(G\) is not Lipschitz continuous, and \(G\) is not Lipschitz smooth.

We now prove that \(G\) satisfies \(\mathcal{H}(\Phi)\)-smoothness. By the fundamental theorem of calculus,
\(G(y)-G(x)=\int_0^1\langle\nabla G(x+t(y-x)),y-x\rangle\,dt\). Subtracting
\(\langle\nabla G(x),y-x\rangle=\int_0^1\langle\nabla G(x),y-x\rangle\,dt\) gives the remainder
\[
S_G(y,x) = \int_0^1 \langle \nabla G(x + t(y-x)) - \nabla G(x), y-x \rangle \, dt.
\]
Taking absolute values and applying the Cauchy--Schwarz inequality and the triangle inequality for integrals yields
\begin{align}
|S_G(y,x)| &\le \int_0^1 \left| \langle \nabla G(x + t(y-x)) - \nabla G(x), y-x \rangle \right| dt \nonumber \\
&\le \int_0^1 \|\nabla G(x + t(y-x)) - \nabla G(x)\|_2 \|y-x\|_2 dt.
\label{eq:intermediate}
\end{align}

For \(1 < r < 2\), the gradient satisfies a H\"older continuity condition~\cite[Theorem~6.3]{RodomanovN20}: for all \(x,y\in\mathbb{R}^d\),
\(\|\nabla G(x) - \nabla G(y)\|_2 \le 2^{2-r} r \|x - y\|_2^{r-1}\).
Replacing \(y\) by \(x + t(y-x)\) in this bound and inserting it into \eqref{eq:intermediate} gives
\[
|S_G(y,x)| \le \int_0^1 2^{2-r} r \, t^{r-1} \|y-x\|_2^r \, dt = 2^{2-r} r \|y-x\|_2^r \int_0^1 t^{r-1} dt = 2^{2-r} \|y-x\|_2^r,
\]
where the last step uses \(\int_0^1 t^{r-1} dt = 1/r\).
With \(\Phi(\cdot) = 2^{2-r}\|\cdot\|_2^r\), the function \(G\) satisfies \(\mathcal{H}(\Phi)\)-smoothness, which completes the proof.

\end{document}